%% file: main.tex
\documentclass[twoside]{article}

\usepackage[accepted]{aistats2026}

\usepackage[authoryear, round]{natbib}

\input{macros}

\begin{document}

\twocolumn[

\aistatstitle{%
Busemann Functions in the Wasserstein Space: Existence, Closed-Forms, and Applications to Slicing
}

\aistatsauthor{Clément Bonet \And Elsa Cazelles \And  Lucas Drumetz \And Nicolas Courty }

\aistatsaddress{CMAP \\ Ecole Polytechnique, IP Paris \And  CNRS, IRIT \\  Université de Toulouse \And Lab-STICC \\ IMT Atlantique \And IRISA \\ Université Bretagne Sud } ]

\begin{abstract}
The Busemann function has recently found much interest in a variety of geometric machine learning problems, as it naturally defines projections onto geodesic rays of Riemannian manifolds and generalizes the notion of hyperplanes. As several sources of data can be conveniently modeled as probability distributions, it is natural to study this function in the Wasserstein space, which carries a rich formal Riemannian structure induced by Optimal Transport metrics. In this work, we investigate the existence and computation of Busemann functions in Wasserstein space, which admits geodesic rays. We establish closed-form expressions in two important cases: one-dimensional distributions and Gaussian measures. These results enable explicit projection schemes for probability distributions on $\mathbb{R}$, which in turn allow us to define novel Sliced-Wasserstein distances over Gaussian mixtures and labeled datasets. We demonstrate the efficiency of those original schemes on synthetic datasets as well as transfer learning problems.
        
\end{abstract}

\section{INTRODUCTION}

\looseness=-1 The Busemann function, introduced by \citet{busemann2005geometry}, provides a natural generalization of affine functions on non-compact metric spaces admitting geodesics which can be extended to infinity. As such, its level sets generalize the notion of affine hyperplanes and so it provides a reliable means of projecting onto geodesics. Thus, it has recently received a lot of attention in geometric Machine Learning approaches, which aim at extending classical Euclidean algorithms towards different spaces such as manifolds \citep{bronstein2017geometric, papillon2025beyond} for data analysis purposes.

In particular, the Busemann function is well defined on geodesically complete spaces, on which geodesics can be extended to infinity in both directions. Although this rules out compact manifolds, such spaces include for instance hyperbolic manifolds, on which the Busemann function has been widely used to perform Principal Component Analysis \citep{chami2021horopca}, to characterize directions and perform classification with prototypes \citep{ghadimi2021hyperbolic, durrant2023hmsn, berg2024horospherical, berg2025multi}, to define decision boundaries for classification \citep{fan2023horospherical,doorenbos2024hyperbolic}, to define layers of neural networks \citep{wang2021laplacian, sonoda2022fully, nguyen2025neural}, or as a projection operator on geodesics in order to define a Sliced-Wasserstein distance \citep{bonet2023hyperbolic,bonet2025sliced}.

These successes in hyperbolic geometry suggest exploring the role of the Busemann function in other non-Euclidean settings. Many real-world data are best modeled as probability distributions. This is the case for instance for documents that are distributions of words \citep{kusner2015word}, single-cells \citep{bellazzi2021gene, haviv2025covariance}, point clouds \citep{haviv2025wasserstein, geuter2025ddeqs}, or even images \citep{seguy2015principal}. Moreover, Gaussian mixtures \citep{chen2018optimal, delon2020wasserstein} or datasets with discrete labels that can be represented as mixtures of discrete distributions \citep{alvarez2020geometric, bonet2025flowing}, can be seen as datasets of probability distributions. 
One powerful way to endow the space of probability distributions with a metric consists in using Optimal Transport (OT) and, in particular, the Wasserstein distance \citep{villani2009optimal}. This distance allows to define the Wasserstein space, which enjoys a very rich geometry thoroughly studied in the last decades, see \emph{e.g.} \citep{ambrosio2008gradient, villani2009optimal, santambrogio2015optimal}. Notably, it carries a formal Riemannian structure, and admits geodesics. Thus, the study of Busemann functions on the Wasserstein space is especially compelling for data analysis. A key challenge, however, is that the Wasserstein space is not geodesically complete, which prevents defining the Busemann function along every geodesic. Fortunately, for any base measure $\mu_0$, there is always at least one geodesic starting from $\mu_0$ that can be extended to infinity in one direction \citep{zhu2021busemann}. 

\textbf{Contributions.} \looseness=-1 In this work, we first provide sufficient conditions to characterize geodesic rays on the Wasserstein space, \emph{i.e} geodesics that can be extended in one direction. We then investigate how to compute the Busemann function along such geodesics. In full generality, we show that this computation reduces to solving an OT problem. In specific cases such as one-dimensional or Gaussian distributions, the Busemann function also admits closed-form expressions. Leveraging these closed-forms, we introduce new sliced distances between labeled datasets. Our results show a strong correlation with classical distances between datasets while being more computationally efficient. Finally, minimizing these distances allows to flow datasets, which we apply in a transfer learning setting.

\section{WASSERSTEIN SPACE} \label{section:geodesic_rays}

\looseness=-1 In this section, we introduce the Wasserstein space and its associated Riemannian structure. Then, we study characterizations of geodesic rays on this space, \emph{i.e.} geodesics that can be extended to infinity in one direction. In particular, we focus on absolutely continuous probability measures, one dimensional probability measures, and Gaussians.

\subsection{Wasserstein Distance}

Let $\cP_2(\R^d) = \{\mu\in \cP(\R^d),\ \int \|x\|_2^2\ \mathrm{d}\mu(x)<\infty\}$ be the space of probability measures in $\R^d$ with finite second moments. Optimal Transport (OT) provides a principled way to define a distance on this space through the 2-Wasserstein distance \citep{villani2009optimal}, defined for all $\mu,\nu\in\cP_2(\R^d)$ as
\begin{equation} \label{eq:w_dist}
    \W_2^2(\mu,\nu) = \inf_{\gamma\in\Pi(\mu,\nu)}\ \int\|x-y\|_2^2\ \mathrm{d}\gamma(x,y),
\end{equation}
\looseness=-1 where $\Pi(\mu,\nu) = \{\gamma\in \mathcal{P}(\mathbb{R}^d\times \mathbb{R}^d),\ \pi^1_\#\gamma=\mu,\ \pi^2_\#\gamma=\nu\}$ is the set of couplings between $\mu$ and $\nu$, $\pi^1:(x,y)\mapsto x$ and $\pi^2:(x,y)\mapsto y$ are the projections on the coordinates, and $\#$ is the push-forward operator defined such that $\T_\#\mu(A)=\mu\big( \T^{-1}(A)\big)$ for a measurable map $\T:\R^d\to\R^{d'}$, $\mu\in\cP_2(\R^d)$ and Borelian $A\in\mathcal{B}(\R^{d'})$. In particular, $\W_2$ is a well defined distance, and $(\cP_2(\R^d),\W_2)$ is called the Wasserstein space. 

\looseness=-1 The OT problem introduces two objects of interest: the distance $\W_2$ and the optimal coupling $\gamma^*\in\Pi_o(\mu,\nu)$ solving \eqref{eq:w_dist}. In the particular case where $\mu\in\cPa$ is absolutely continuous with respect to the Lebesgue measure, it is well known by Brenier's theorem \citep{brenier1991polar} that the optimal coupling $\gamma^*$ is unique and supported on the graph of a map $\T$, called the Monge map, \emph{i.e.} $\gamma^* = (\id, \T)_\#\mu$ with $\T$ satisfying $\T_\#\mu=\nu$ and $\id:x\mapsto x$ the identity function.

Additionally, there are a few specific cases for which the Wasserstein distance can be computed in closed-form. In dimension $d=1$, the Wasserstein distance is only the $L^2$ norm between the quantile functions, \emph{i.e.} for all $\mu,\nu\in\cP_2(\R)$,
\begin{equation} \label{eq:1d_wasserstein}
    \W_2^2(\mu,\nu) = \int_0^1 \big|F_{\mu}^{-1}(u) - F_{\nu}^{-1}(u)\big|^2\ \mathrm{d}u, %
\end{equation}
\looseness=-1 where $F_\mu^{-1}$ and $F_{\nu}^{-1}$ denote the quantile functions of $\mu$ and $\nu$.
The optimal coupling is then obtained as $\gamma^*=(F_{\mu}^{-1}, F_{\nu}^{-1})_\#\mathrm{Unif}([0,1])$, and if $\mu$ is absolutely continuous, the OT map is the increasing rearrangement $\T = F_\nu^{-1} \circ F_{\mu}$ \citep[Theorem 2.9]{santambrogio2015optimal}. 
In higher dimensions, we usually do not have a closed-form, except in particular cases such as the Wasserstein distance between two Gaussian distributions \citep{givens1984class, gelbrich1990formula}. %
Namely, for $\mu=\cN(m_\mu,\Sigma_\mu)$, $\nu=\cN(m_\nu,\Sigma_\nu)$ two Gaussian distributions with respective means $m_\mu,m_\nu\in\R^d$ and positive definite covariance matrices $\Sigma_\mu,\Sigma_\nu\in S_d^{++}(\R)$, we get, 
\begin{equation} \label{eq:closed_form_gaussians}
    \W_2^2(\mu,\nu) = \|m_\mu-m_\nu\|_2^2 + \cB^2(\Sigma_\mu,\Sigma_\nu),
\end{equation}
where $\cB$ defines a distance between positive semi-definite matrices, known in the literature of quantum information as the Bures distance \citep{bhatia2019bures}, and is of the form
\begin{equation}
    \cB^2(\Sigma_\mu,\Sigma_\nu) = \tr\left(\Sigma_\mu+\Sigma_\nu - 2(\Sigma_\mu^\frac12 \Sigma_\nu \Sigma_\mu^\frac12)^\frac12\right).
\end{equation}
Thus, we refer to the Wasserstein distance between Gaussians as the Bures-Wasserstein distance $\bw$, and the space of Gaussians endowed with $\bw$ is called the Bures-Wasserstein space $\bw(\R^d)$. Furthermore, the OT map between $\mu$ and $\nu$ is of the form $\T:x\mapsto m_\nu + A_\mu^\nu(x-m_\mu)$ \citep[Remark 2.31]{peyre2019computational} where 
\begin{equation}
    A_\mu^\nu = \Sigma_\mu^{-\frac12}(\Sigma_\mu^\frac12 \Sigma_\nu \Sigma_\mu^\frac12)^\frac12\Sigma_\mu^{-\frac12}.
\end{equation}

\paragraph{Riemannian Structure.}

\looseness=-1 It is well known that the Wasserstein space has a formal Riemannian structure \citep{otto2001geometry}. In particular, it is a geodesic space: for any measures $\mu_0,\mu_1\in\cP_2(\R^d)$, there is at least one constant-speed geodesic, \emph{i.e.} a continuous curve $t\in [0,1]\mapsto \mu_t\in\cP_2(\R^d)$ interpolating between $\mu_0$ and $\mu_1$, and satisfying 
\begin{equation} \label{eq:cst_speed_geod}
    \forall s,t\in [0,1],\ \W_2(\mu_t,\mu_s) = |t-s| \W_2(\mu_0,\mu_1).
\end{equation}
We call $\kappa_\mu = \W_2(\mu_0, \mu_1)$ the speed of the geodesic $(\mu_t)_{t\in[0,1]}$. Such a curve is always a displacement interpolation \citep[Proposition 2.9]{bertrand2012geometric}, \emph{i.e.} it is of the form \citep{mccann1997convexity}
\begin{equation}
    \forall t\in [0,1],\ \mu_t = \big((1-t)\pi^1 + t \pi^2\big)_\#\gamma^*,
\end{equation}
where $\gamma^*\in\Pi_o(\mu_0,\mu_1)$, and is fully characterized by $\mu_0$, $\mu_1$ and $\gamma^*$. %
In the case where an OT map $\T$ exists between $\mu_0$ and $\mu_1$, \emph{e.g.} if $\mu_0\in \cPa$, %
then the geodesic curve can be further written as
\begin{equation}
    \forall t\in [0,1],\ \mu_t=\big((1-t)\id + t \T\big)_\#\mu_0.
\end{equation}
If the geodesic can be extended for any $t\in\R$, \emph{i.e.} \eqref{eq:cst_speed_geod} is satisfied for any $s,t\in\R$, it is called a geodesic line. If \eqref{eq:cst_speed_geod} holds for any $s,t\in\R_+$, it is called a geodesic ray \citep{bridson2013metric}.

\subsection{Geodesic Rays on $(\cP_2(\R^d),\W_2)$}

\citet{kloeckner2010geometric} first studied the conditions on the measures $\mu_0$ and $\mu_1$ under which the geodesics connecting them can be extended. For instance, in \citep[Proposition 3.6]{kloeckner2010geometric}, it was shown that the geodesic curve $t\mapsto \mu_t$ is a geodesic line if and only if %
$\mu_1$ is a translation of $\mu_0$. Consequently, constructing geodesic lines is very restrictive. Geodesic rays are more flexible, and are also the appropriate object that allows %
defining Busemann functions: in \citep{zhu2021busemann}, it is proved that for any $\mu_0\in\cP_2(\R^d)$, there exists at least one geodesic ray originating from it. 

\looseness=-1 In this paper, we discuss how to characterize geodesic rays on $(\cP_2(\R^d),\W_2)$. First, if $\mu_0=\delta_{x_0}$ for some $x_0\in \R^d$, then by \citep[Lemma 2.1]{bertrand2016geometric}, we can extend any geodesic starting from $\mu_0$ and passing through $\mu_1\in\cP_2(\R^d)$ as a geodesic ray of the form $\mu_t = \big((1-t)x_0 + t\id)_\#\mu_1$ for any $t\ge 0$, since the optimal coupling in this case is the independent coupling $\gamma=\mu_0\otimes\mu_1$. However, for an arbitrary $\mu_0\in\cP_2(\R^d)$, not all geodesics can be extended as geodesic rays.

\looseness=-1 In the setting of Brenier's theorem, we show that geodesics are rays if and only if the Monge map between $\mu_0$ and $\mu_1$ is the gradient of a 1-convex Brenier potential function $u$, that is $x\mapsto u(x) - \frac{\|x\|_2^2}{2}$ is convex.
\begin{proposition} \label{prop:geodesic_rays}
    Let $\mu_0\in\cPa,\ \mu_1\in\cP_2(\R^d)$, and $\T$ the Monge map between $\mu_0$ and $\mu_1$. The curve $t\mapsto \mu_t = \big((1-t)\id + t \T\big)_\#\mu_0$ is a geodesic ray if and only if $\T$ is the gradient of a 1-convex function $u$.
\end{proposition}

This result is strongly related to \citep[Section 4]{natale2022geodesic} in which it is stated that a geodesic can be extended on a segment $[0,\alpha]$ for $\alpha\ge 1$ if and only if $x\mapsto \alpha u(x) - (\alpha-1)\frac{\|x\|_2^2}{2}$ is convex (or equivalently, $x\mapsto u(x) - (1-\frac{1}{\alpha})\frac{\|x\|_2^2}{2}$ is convex). Taking the limit $\alpha\to+\infty$, we recover the result of \Cref{prop:geodesic_rays}.

Note that $\mu_0\in\cPa$ in \Cref{prop:geodesic_rays} allows leveraging Brenier's theorem and guarantees that there exists an OT map. In the one dimensional case, we can further characterize geodesic rays starting from any $\mu_0\in\cP_2(\R)$ with quantile functions. Indeed, denoting $F_0^{-1}$ and $F_1^{-1}$ the quantile functions of $\mu_0, \mu_1\in\mathcal{P}_2(\mathbb{R})$, the geodesic between $\mu_0$ and $\mu_1$ at time $t\in [0,1]$ is defined by $\mu_t = \big((1-t)\pi^1 + t \pi^2 \big)_\#\gamma^*$ with $\gamma^*=(F_0^{-1},F_1^{-1})_\#\mathrm{Unif}([0,1])$ the optimal coupling between $\mu_0$ and $\mu_1$. Then, for $F_t^{-1}$ the quantile of the geodesic at time $t\in [0,1]$, it is well known (see \emph{e.g.} \citep[Equation 7.2.8]{ambrosio2008gradient}) that
\begin{equation} \label{eq:geodesics_1d}
    \forall t\in [0,1],\ F_t^{-1} = (1-t)F_0^{-1} + tF_1^{-1}.
\end{equation}
\looseness=-1 As observed by \citet{kloeckner2010geometric}, non-decreasing left-continuous functions are the inverse cumulative distribution function of a probability distribution. We can thus extend the geodesic as long as $F_t^{-1}$ is non-decreasing, which gives a condition on $F_1^{-1}-F_0^{-1}$.

\begin{proposition} \label{prop:1d_geodesic_rays}
    Let $\mu_0,\ \mu_1\in\mathcal{P}_2(\mathbb{R})$ and $F_0^{-1}$, $F_1^{-1}$ their quantile functions. 
    The geodesic between $\mu_0$ and $\mu_1$ is a ray if and only if $F_1^{-1}-F_0^{-1}$ is non-decreasing.
\end{proposition}

As an application of \Cref{prop:1d_geodesic_rays}, we get the following results on discrete 1D distributions with the same number of samples, and on 1D Gaussian distributions.
\begin{corollary} \label{corr:1d_discrete_geod}
    Let $x_1<\dots<x_n\in \R$, $y_1<\dots<y_n\in \R$, $\mu_0=\frac1n\sum_{i=1}^n\delta_{x_i}$ and $\mu_1=\frac1n \sum_{i=1}^n \delta_{y_i}$. Then, the geodesic between $\mu_0$ and $\mu_1$ is a ray if and only if for all $j>i$, $y_i-x_i\le y_j-x_j$.
\end{corollary}
Note that in the setting of \Cref{corr:1d_discrete_geod}, the geodesic is given by $\mu_t=\frac1n \sum_{i=1}^n \delta_{(1-t)x_i+ty_i}$ as points are sorted. Hence, the condition $y_i-x_i\le y_j-x_j$ for all $j>i$ ensures that there are no crossings between particles for all time $t\ge 0$ since $x_i+t(y_i-x_i) < x_j + t(y_j-x_j)$. The optimal assignment between $\mu_t$ and $\mu_s$ therefore remains the same (the identity) for any $s,t\ge 0$.
\begin{corollary} \label{corr:1d_gaussian_geod}
    Let $\mu_0=\cN(m_0,\sigma_0^2),\ \mu_1=\cN(m_1,\sigma_1^2)$ with $m_0,m_1\in \R,\ \sigma_0,\sigma_1\in\R_+$. Then, the geodesic between $\mu_0$ and $\mu_1$ is a ray if and only if $\sigma_1\ge \sigma_0$.
\end{corollary}
Note that if $\sigma_0^2=\sigma_1^2$ in Corollary \ref{corr:1d_gaussian_geod}, \emph{i.e.} if the measures are translated, we recover that the geodesic is indeed a line, as it can be extended to infinity in both directions. In the case of arbitrary 1D Gaussian distributions, we can actually obtain the largest interval over which the geodesic can be extended, see Appendix \ref{app:add_result}.

For Gaussians of any dimension, we have that the Bures-Wasserstein space $\bw(\R^d)$ is geodesically convex, \emph{i.e.} the geodesic between two Gaussian distributions stays in $\bw(\R^d)$ at each time $t\in [0,1]$. In particular, the geodesic between $\mu=\cN(m_\mu,\Sigma_\mu)$ and $\nu=\cN(m_\nu,\Sigma_\nu)$ is given, for all $t\in [0,1]$, by $\mu_t=\cN(m_t,\Sigma_t)$ \citep{altschuler2021averaging} with,
\begin{equation} \label{eq:geodesics_bw}
    \left\{
    \begin{aligned}
        m_t      &= (1-t)m_\mu + t m_\nu \\
        \Sigma_t &= \big((1-t)I_d + t A_\mu^\nu\big)\Sigma_\mu\big((1-t)I_d + t A_\mu^\nu\big).
    \end{aligned}
    \right.
\end{equation}

\looseness=-1 As an application of \Cref{prop:geodesic_rays}, we can extend \Cref{corr:1d_gaussian_geod} to Gaussian distributions for $d\geq1$. In particular, the condition $\sigma_0\le \sigma_1$ is extended to $\Sigma_0$ and $\Sigma_1$ through the partial (Loewner) ordering $\preceq$ on $S_d^{+}(\R)$, the space of positive semi-definite matrices, defined by $A\succeq B$ if and only if $A-B\in S_d^{+}(\R)$.
\begin{corollary} \label{corr:gaussian_geod}
    Let $\mu_0=\cN(m_0,\Sigma_0)$, $\mu_1=\cN(m_1,\Sigma_1)$ with $m_0,m_1\in\R^d$ and $\Sigma_0,\Sigma_1\in S_d^{++}(\R)$. Then, the geodesic between $\mu_0$ and $\mu_1$ is a geodesic ray if and only if $(\Sigma_0^{\frac12}\Sigma_1\Sigma_0^{\frac12})^{\frac12}\succeq \Sigma_0$.
\end{corollary}
Note that this condition is implied by $\Sigma_1^{\frac12}\succeq \Sigma_0^{\frac12}$ in general by Furata's inequality \citep[Theorem 1.3]{fujii2010furuta}, and equivalent whenever $\Sigma_0$ and $\Sigma_1$ commute.

\looseness=-1 The conditions on $\mu_0$ and $\mu_1$ to obtain geodesic rays can be seen as guaranteeing a certain regularity along the geodesic. For instance, when $\mu_0\in\cPa$, choosing $\mu_1=\nabla u_\#\mu_0$ for a 1-convex function $u$ (see \Cref{prop:geodesic_rays})
guarantees that the geodesic always stays in $\cPa$. %
While it does not hold for an arbitrary $\mu_1\in\cP_2(\R^d)$, it would be possible to find the closest geodesic ray \emph{e.g.} by minimizing $\W_2(\nabla u_\#\mu_0,\mu_1)$ with $u$ 1-convex \citep{paty2020regularity}. In the discrete case, these conditions can be interpreted in terms of particle dynamics: failure to satisfy the conditions corresponds to particle crossings. 

An alternative extrapolation of geodesic curves beyond $t=1$ was recently proposed by \citet{gallouet2025metric} by solving a suitable variational problem. Their extrapolation coincides with the constant-speed geodesic as long as the curve remains geodesic, and therefore coincides with a true geodesic for all $t\ge 0$ when considering rays. Nonetheless, it is not the case for geodesics that are not rays. In particular for $\mu_1$ a Dirac, their solution merges particles after crossing at the Dirac. In this work, we only consider true geodesic rays.

We also note that the conditions to have geodesic rays are very similar to those where $\mu_0$ is smaller than $\mu_1$ in the convex order when their first moments coincide, \emph{i.e.} satisfying for all $f:\R^d\to\R^d$ convex, $\int f\mathrm{d}\mu_0 \le \int f\mathrm{d}\mu_1$, see \emph{e.g.} \citep[Theorem 4 and 6]{muller2001stochastic} for the Gaussian case, and \citep{shu2020hopf} for the 1D case. These conditions are also connected to projections in the convex order, see \emph{e.g.} \citep[Proposition 3.3]{alfonsi2025wasserstein}. We defer a more comprehensive study of these relations to future work.

\section{BUSEMANN FUNCTION} \label{section:busemann}

\looseness=-1 In this section, we first introduce the Busemann function in geodesic metric spaces. We then discuss how to compute it in the Wasserstein space in general settings, and in specific cases where a closed-form exists.

\subsection{Background on the Busemann Function} \label{section:bg_busemann}

In any geodesic metric space $(\X,d)$ that admits geodesic rays, the Busemann function $B^\gamma$ associated to a geodesic ray $\gamma$ can be defined for any $x\in \X$, as in \citep[II.8.17]{bridson2013metric},
\begin{equation}
    B^\gamma(x) = \lim_{t\to\infty}\ d\big(\gamma(t), x\big) - t \cdot d\big(\gamma(0), \gamma(1)\big).
\end{equation}
\looseness=-1 This function has attracted particular interest in geometric Machine Learning as it provides a natural generalization of hyperplanes on metric spaces. Indeed, in the particular case of Euclidean spaces, geodesic rays are of the form $\gamma(t) = t\theta$ for $\theta\in S^{d-1}=\{\theta\in\R^d,\ \|\theta\|_2^2=1\}$, $t\in\R$, and the Busemann function is given, for any $x\in\R^d$, by $B^\gamma(x)=-\langle x,\theta\rangle$. Therefore, its level sets are (affine) hyperplanes. Moreover, the Busemann function provides a principled way to project a point $x\in \X$ on the geodesic ray $\gamma$. In fact, noticing that for any $s\in \R_+$, $B^\gamma\big(\gamma(s)\big)=-s$, the projection of $x\in \X$ on $\gamma$ is given by $P^\gamma(x)=\gamma\big(-B^\gamma(x)\big)$. In particular, all points on a level set of $B^\gamma$ are projected on the same point. Note however that when $-B^\gamma(x)<0$, there is, in general, no guarantee that $\gamma\big(-B^\gamma(x)\big)$ belongs to the geodesic, but $B^\gamma(x)\in\R$ is always well defined, and provides a projection on $\R$. In a Hilbertian space (\emph{i.e.} of null curvature), the Busemann projection is actually equivalent to the coordinate of the metric projection, \emph{i.e.} $-B^\gamma(x) = \argmin_{t}\ d\big(x, \gamma(t)\big)$.

\subsection{Busemann on the Wasserstein Space}

As the Wasserstein space is not geodesically complete, not all geodesic can be extended as a ray, and thus the Busemann function is not defined along every geodesic. Fortunately, any $\mu_0\in\cP_2(\R^d)$ is the starting point of at least one geodesic ray \citep[Theorem 1.1]{zhu2021busemann}, and in some particular cases, we can characterize them as described in the previous section. Let $(\mu_t)_{t\ge 0}$ be a geodesic ray and $\kappa_\mu=\W_2(\mu_0,\mu_1)$ its speed. Let us define the Busemann function $B^\mu$ associated to $(\mu_t)_{t\ge 0}$ by, for any $\nu\in\cP_2(\R^d)$,
\begin{equation}
    B^\mu(\nu)=\lim_{t\to\infty}\ \W_2(\mu_t,\nu) - \kappa_\mu t.
\end{equation}
\looseness=-1 Thanks to the Riemannian structure of $(\cP_2(\R^d),\W_2)$, we can always assume that geodesics have unit speed, \emph{i.e.} $\kappa_\mu=1$, see Appendix \ref{app:bg_busemann}. In the following formulas of $B^\mu$, this translates as a renormalization by $\kappa_\mu$.

First, we show that $B^\mu$ admits a more convenient form as an infimum over a suitable set of couplings $\Gamma(\mu_0,\mu_1,\nu) = \{\Tilde{\gamma}\in \Pi(\mu_0,\mu_1,\nu),\ \pi^{1,2}_\#\Tilde{\gamma}\in\Pi_o(\mu_0,\mu_1)\}$ between $\mu_0,\mu_1$ and $\nu$, and such that the coupling between the two first marginals is optimal for \eqref{eq:w_dist}.
\begin{proposition} \label{prop:busemann_general}
    Let $(\mu_t)_{t\ge 0}$ be a geodesic ray on $\cP_2(\R^d)$. Let $\nu\in\cP_2(\mathbb{R}^d)$, then
    \begin{multline} \label{eq:busemann_general}
        \raisetag{40pt}
        B^\mu(\nu)  = \\   
        \inf_{\Tilde{\gamma}\in\Gamma(\mu_0,\mu_1,\nu)}\ -\kappa_\mu^{-1} \int \langle x_1-x_0,y-x_0\rangle\ \mathrm{d}\Tilde{\gamma}(x_0,x_1,y).
    \end{multline}    
\end{proposition}

\looseness=-1 We can refine the result in the case where the geodesic ray is given by an OT map, \emph{i.e.}, when the OT map is the gradient of a 1-convex function by \Cref{prop:geodesic_rays}.

\begin{corollary} \label{corollary:busemann_ot_map}
    Let $\mu_0\in\cPa$, $\mu_1\in \cP_2(\R^d)$, and assume the OT map $\T$ between $\mu_0$ and $\mu_1$ is the gradient of a 1-convex function. Let $\nu\in \cP_2(\R^d)$, then, %
    \begin{equation}
        B^\mu(\nu) =  \inf_{\gamma\in\Pi(\mu_0,\nu)}\ -\kappa_\mu^{-1} \int \langle \T(x_0)-x_0, y-x_0\rangle\ \mathrm{d}\gamma(x_0,y).
    \end{equation}
\end{corollary}
This problem is equivalent to the OT problem
\begin{equation}
    \inf_{\gamma\in\Pi(\mu_0,\nu)}\ \int \|\T(x_0)-x_0-y\|_2^2\ \mathrm{d}\gamma(x_0,y),
\end{equation}
and can thus be solved using classical OT solvers.

In the specific case where $\mu_0$ is a Dirac, we can also leverage that any geodesic is a ray, and that the optimal coupling is of the form $\gamma=\mu_0\otimes \mu_1$.

\begin{corollary} \label{corollary:busemann_dirac}
    Let $\mu_0=\delta_{x_0}$ where $x_0\in \R^d$, and $\mu_1\in\cP_2(\R^d)$. Let $\nu\in\cP_2(\R^d)$, then %
    \begin{equation} \label{eq:busemann_dirac}
        B^\mu(\nu) = \inf_{\gamma\in\Pi(\mu_1,\nu)}\ - \kappa_\mu^{-1}\int \langle x_1-x_0, y-x_0\rangle\ \mathrm{d}\gamma(x_1,y).
    \end{equation}
\end{corollary}
Note that \eqref{eq:busemann_dirac} is equivalent to the OT problem \eqref{eq:w_dist} between $\mu_1$ and $\nu$. Letting $\mu_1=\delta_{x_1}$ with $\theta:=x_1-x_0\in S^{d-1}$ and $\gamma(t)=t\theta$, \eqref{eq:busemann_dirac} is equal to $\int B^\gamma(y)\mathrm{d}\nu(y)$ and thus the Busemann function lifts from $\R^d$ to $\cP_2(\R^d)$.

\subsection{Closed-forms of the Busemann Function}\label{sec:closed_form_busemann}

In the cases mentioned above, we don't have a closed-form for the OT problem, and must therefore solve an optimization problem to compute the corresponding Busemann functions. Nonetheless, we can compute it in closed-form whenever closed-forms for the Wasserstein distance and the geodesics are available. First, we consider 1D distributions leveraging \eqref{eq:1d_wasserstein} and \eqref{eq:geodesics_1d}. %

\begin{proposition} \label{prop:busemann_closed_1d}
    Let $(\mu_t)_{t\ge 0}$ be a unit-speed geodesic ray in $\cP_2(\R)$ (\emph{i.e.} $\kappa_\mu=1$), then for any $\nu\in\cP_2(\R)$,
    \begin{equation} \label{eq:busemann_1d}
        \begin{aligned}
            B^\mu(\nu) %
            &= -\langle F_1^{-1} - F_0^{-1}, F_\nu^{-1}-F_0^{-1}\rangle_{L^2([0,1])}.
        \end{aligned}
    \end{equation}
\end{proposition}

We observe that, up to a sign, \eqref{eq:busemann_1d} corresponds  to the inner product in $L^2([0,1])$ between $F_{1}^{-1}-F_{0}^{-1}$ and $F_{\nu}^{-1} - F_{0}^{-1}$, which are the quantiles centered around $F_{0}^{-1}$, and is directly obtained from the Hilbert structure of the one dimensional Wasserstein space. Consequently, the Busemann function between 1D Gaussians is only an inner product on the product space $\R\times \R_+^*$ of the (centered) means and standard deviations.

\begin{corollary} \label{corr:busemann_1d_gaussian}
    Let $\mu_0=\cN(\mo,\so^2)$, $\mu_1=\cN(\m1,\s1^2)$, $\nu=\cN(m,\sigma^2)$ with $m_0,m_1,m\in \R$, $\sigma_0,\sigma_1,\sigma\in \R_+^*$ such that $\sigma_1\ge \sigma_0$ and $\W_2^2(\mu_0,\mu_1)=1$. Then, 
    \begin{equation} \label{eq:1d_busemann_gaussian}
        B^\mu(\nu) = - (\m1-\mo)(m-\mo) - (\s1-\so)(\sigma-\so). %
    \end{equation}
\end{corollary}

\looseness=-1 More generally, on $\bw(\R^d)$, %
we leverage the closed-forms of the Wasserstein distance  \eqref{eq:closed_form_gaussians} and geodesics \eqref{eq:geodesics_bw}, which remain Gaussian at all time. %

\begin{proposition} \label{prop:closed_form_general_gaussian}
    Let $(\mu_t)_{t\ge 0}$ be a geodesic ray characterized by $\mu_0 = \mathcal{N}(\mo, \So)$ and $\mu_1 = \mathcal{N}(\m1, \S1)$, and such that $\kappa_\mu=1$. 
    Then, for any $\nu=\mathcal{N}(m,\Sigma)$,
    \begin{equation}\label{eq:busemann_gaussians}
        \begin{aligned}
          B^\mu(\nu) = &-\langle \m1-\mo, m-\mo\rangle + \tr\big(\So(A_{\mu_0}^{\mu_1}-I_d)\big) \\
            &- \tr\big((\Sigma^\frac12(\So-\So A_{\mu_0}^{\mu_1} - A_{\mu_0}^{\mu_1}\So + \S1)\Sigma^\frac12)^\frac12\big).
        \end{aligned}
    \end{equation}\noindent
\end{proposition}

When all covariance matrices commute, \emph{e.g.} if they are diagonal matrices, \eqref{eq:busemann_gaussians} simplifies as
\begin{equation}
    \begin{aligned}
        B^\mu(\nu) %
        &= -\langle\m1-\mo,m-\mo\rangle - \langle \S1^\frac12-\So^\frac12, \Sigma^\frac12-\So^\frac12\rangle_F.
    \end{aligned}
\end{equation}
This corresponds to the inner product in the space $\mathbb{R}^d \times S_d(\mathbb{R})$. Moreover, we recover \eqref{eq:1d_busemann_gaussian} in one dimension.

\section{SLICING DATASETS}

Building on the Sliced-Wasserstein distance, a computationally efficient alternative to the Wasserstein distance, we use the Busemann function to construct new sliced distances to compare labeled datasets. %

\subsection{Sliced-Wasserstein Distance}

Given two discrete distributions $\mu^n=\frac1n\sum_{i=1}^n\delta_{x_i}$ and $\nu^n=\frac1n\sum_{j=1}^n\delta_{y_j} \in \cP_2(\R^d)$, the Wasserstein distance between $\mu^n$ and $\nu^n$ can be computed in $\cO(n^3\log n)$, which is too costly for many applications. Thus, several variants have been proposed, such as adding an entropic regularization and using Sinkhorn's algorithm \citep{cuturi2013sinkhorn}, mini-batchs \citep{fatras2020learning,fatras2021minibatch} or low-rank solvers \citep{scetbon2021low}.

\looseness=-1 Another very popular alternative to the Wasserstein distance, which enjoys much better computational properties, is the Sliced-Wasserstein (SW) distance \citep{rabin2012wasserstein, bonneel2015sliced}. It is based on the attractive closed-form of the Wasserstein distance in 1D \eqref{eq:1d_wasserstein}, which can be computed in practice between discrete distributions by sorting the samples and therefore has a complexity of $\cO(n\log n)$. Given a parametric one dimensional projection $P^\theta:\R^d\to \R$ with $\theta\in\Theta$, SW is then defined between $\mu,\nu\in\cP_2(\R^d)$ as the average of the 1D Wasserstein distances between the projected distributions, \emph{i.e.},
\begin{equation} \label{eq:sw}
    \sw_2^2(\mu,\nu) = \int \W_2^2(P^\theta_\#\mu,P^\theta_\#\nu)\ \mathrm{d}\lambda(\theta),
\end{equation}
for $\lambda\in\cP(\Theta)$. In its original formulation, SW is set for $\Theta=S^{d-1}$, \emph{i.e.} the unit hypersphere in dimension $d$, 
$\lambda=\cU(S^{d-1})$, the uniform measure on $S^{d-1}$ and for any $\theta\in S^{d-1}$, $P^\theta(x)=\langle x,\theta\rangle$. However, there exist other variants with different projection schemes, \emph{e.g.} convolutions for image data \citep{nguyen2022revisiting}, manifold-aware projections \citep{bonet2023hyperbolic, bonet2023sliced, bonet2025sliced}, or more general non linear projections \citep{kolouri2019generalized, chen2022augmented}.

\subsection{Comparing Labeled Datasets}

We consider a dataset $\cD=\big\{(x_i,y_i)\big\}_{i=1}^n$ of pairs of samples $x_i\in \R^d$ associated to a label $y_i$ from a set of $C$ classes $\cY = \{1,\dots,C\}$. Class conditional distributions are given for a specific class $y\in\cY$ by $\varphi(y) = \frac{1}{n_y}\sum_{i=1}^n \delta_{x_i}\mathbb{1}_{\{y_i=y\}}$, with $n_y=\sum_{i=1}^n \mathbb{1}_{\{y_i=y\}}$ the number of samples in the class $y$. A dataset can then be represented by a probability distribution over the product space $\R^d\times \cP_2(\R^d)$, \emph{i.e.} as $\frac{1}{n}\sum_{i=1}^n \delta_{(x_i,\varphi(y_i))}\in\cP_2(\R^d\times \cP_2(\R^d)\big)$. 

\citet{alvarez2020geometric} proposed to use 
\begin{equation}\label{eq:otdd_cost}
    d_{\cD}\big((x,y), (x',y')\big)^2 = \|x-x'\|_2^2 + \W_2^2\big(\varphi(y), \varphi(y')\big).
\end{equation}
\looseness=-1 as groundcost of an optimal transport problem on $\cP_2\big(\R^d\times\cP_2(\R^d)\big)$, defining the Optimal Transport Dataset Distance (OTDD), see Appendix \ref{appendix:otdd}. However, OTDD is costly to compute as it requires to solve $\cO(C^2)$ OT problems with $n_y$ samples to compute the groundcost \eqref{eq:otdd_cost}, and a global OT problem with $n=\sum_{y=1}^C n_y$ samples.
This has led to several approximations and variants aimed at reducing the computational cost, see \emph{e.g.} \citep{hua2023dynamic, liu2025wasserstein, nguyen2025lightspeed, bonet2025flowing}. 

\citet{nguyen2025lightspeed} recently proposed a sliced distance for labeled datasets on $\cPProd{\R^d}$. This requires building a projection from $\R^d\times\cP_2(\R^d)$ to $\R$ to project the distribution of pairs $(x_i,\varphi(y_i))$ onto a distribution in $\cP_2(\R)$. 
Their construction can be broken down into combining two projections from $\R^d\to\R$ and from $\cP_2(\R^d)\to \R$ using the Hierarchical Hybrid projection \citep{nguyen2024hierarchical}, which consists of a random linear combination with weights in the sphere. Especially, for any labeled sample $(x,y)\in\R^d\times \cY$, their projection is of the form
\begin{equation} \label{eq:proj_sotdd}
    P^{\alpha,\theta,\lambda}(x,y) = \alpha_1 P^\theta(x) + \sum_{i=1}^k \alpha_{i+1} \cM^{\lambda_i}\big(P^\theta_\#\varphi(y)\big),
\end{equation}
with $P^\theta:\R^d\to \R$ any projection on a line, $\cM^\lambda:\cP_2(\R)\to \R$ the moment transform projection and $\alpha\in S^{k}$. %
Plugging \eqref{eq:proj_sotdd} into the Wasserstein distance term on $\cP_2(\R)$ in \eqref{eq:sw} defines the sliced OTDD distance (SOTDD) on $\cPProd{\R^d}$.

\subsection{Slicing Datasets with Busemann} \label{sec:sliced_busemann}

As the Busemann function allows to project any probability distribution onto $\R$, it is natural to use it as a projection to define Sliced-Wasserstein distances for the purpose of comparing labeled datasets.

From a computational perspective, we want to avoid solving additional OT problems to compute the Busemann functions. Therefore, we propose two new discrepancies based on the closed-forms of the Busemann function for 1D probability distributions \eqref{eq:busemann_1d} and for Gaussians \eqref{eq:busemann_gaussians}. In both cases, as in \citet{nguyen2025lightspeed}, we use the Hierarchical Hybrid projection from \citep{nguyen2024hierarchical} to combine the projection $P^\theta:\R^d\to \R$ of the features and $Q^\eta:\cY\to \R$ of the labels, \emph{i.e.} for $\alpha\in S^1$, we define for $(x,y)\in \R^d\times \cY$,
\begin{equation}
    P^{\alpha,\theta,\eta}\big((x,y)\big) = \alpha_1 P^\theta(x) + \alpha_2 Q^\eta(y).
\end{equation}

\paragraph{Gaussian Approximation.}

\looseness=-1 To leverage the closed-form between Gaussian \eqref{eq:busemann_gaussians}, we use a Gaussian approximation of the classes, with a possible dimension reduction of the features beforehand as in \citep{hua2023dynamic}. Let us denote $\psi:\R^d\to\R^{d'}$ a dimension reduction operator with $d'\ll d$, and, for $\mu\in\cP_2(\R^{d'})$, the Gaussian approximation $\Xi(\mu)=\cN\big(m(\mu),\Sigma(\mu)\big)$ with $m(\mu)=\int x\mathrm{d}\mu(x)$ and $\Sigma(\mu)=\int (x-m(\mu))\otimes (x-m(\mu))\ \mathrm{d}\mu(x)$ the mean and covariance operators. The label projections are then given by $Q^\eta(y) = B^\eta\big(\Xi\big(\psi_\#\varphi(y)\big)\big)$ with $\eta$ a geodesic ray on $\bw(\R^{d'})$.

To define a sliced distance, we also need to sample a valid ray $\eta$ so that the Busemann function is well defined. To do so, we choose to fix $\eta_0=\cN(0,I_d)$, and sample $\eta_1=\cN(m_1,\Sigma_1)$ such that $m_1\in S^{d-1}$, $\Sigma_1\in S_d^{++}(\R)$ with $\Sigma_1^{\frac12}\succeq I_d$ and $\W_2^2(\eta_0, \eta_1)=1$. To enforce $\Sigma_1^{\frac12}\succeq I_d$, we remark that it is equivalent to consider $S=\log_{I_d}(\Sigma_1) = \Sigma_1^{\frac12} - I_d \succeq 0$, where $\log_{I_d}$ is the logarithm map in $S_d^{++}(\R)$. Thus, we sample uniformly $\Delta\in O_d(\R)$ an orthogonal matrix and $\Tilde{\theta}\in S^{d-1}$, and define $S:=\Delta\mathrm{diag}(|\Tilde{\theta}|)\Delta^T$ and $\Sigma_1:=\exp_{I_d}(S) = (I_d+S)^2$. To enforce $\W_2^2(\eta_0,\eta_1)=1$, we normalize $(m_1-m_0, S)$ in the tangent space to have a unit-speed geodesic ray, see Appendix \ref{appendix:sliced_gm}. Given two datasets $\bP,\bQ\in\cPProd{\R^d}$, we define the Sliced-Wasserstein Busemann Gaussian distance ($\swbg$) as
\begin{equation}
    \swbg^2(\bP,\bQ) = \int \W_2^2(P^\vartheta_\#\bP,P^\vartheta_\#\bQ)\ \mathrm{d}\lambda(\vartheta),
\end{equation}
with $\vartheta = (\alpha,\theta, \Tilde{\theta}, \Delta, m_1)$ and $\lambda$ the uniform measure on the resulting product space.

\paragraph{1D Projections.} To leverage the 1D closed-form of the Busemann function \eqref{eq:busemann_1d}, we can first project the class conditional distributions in 1D, and define for $y\in\cY$, $Q^{\eta,\theta}(y)=B^\eta\big(P^\theta_\#\varphi(y)\big)$.
Regarding $\eta$, setting $\eta_0=\delta_0$, we get geodesic rays for any $\eta_1$, and thus set $\eta_1=\cN(m_1,\sigma_1^2)$ such that the speed $\kappa_\eta = \W_2^2(\eta_0, \eta_1) = m_1^2+\sigma_1^2=1$. Then, for $\mu\in\cP_2(\R)$, \eqref{eq:busemann_1d} writes
\begin{equation} \label{eq:busemann_1d_dirac}
    B^\eta(\mu)=-m_1 m(\mu) - \sigma_1 \int_0^1 \phi^{-1}(u)F_\mu^{-1}(u)\ \mathrm{d}u,
\end{equation}
with $\phi$ the cumulative distribution function of $\cN(0,1)$. In practice, the geodesic ray can be sampled using $m_1\sim \cU([-1,1])$ and setting $\sigma_1=\sqrt{1-m_1^2}$. Given $\bP,\bQ\in\cPProd{\R^d}$, we define the Sliced-Wasserstein Busemann 1D Gaussian distance ($\swbdg$) as
\begin{equation}
    \swbdg^2(\bP,\bQ) = \int \W_2^2(P^\vartheta_\#\bP,P^\vartheta_\#\bQ)\ \mathrm{d}\lambda(\vartheta),
\end{equation}
with $\vartheta=(\alpha,\theta,m_1)$ and $\lambda=\cU(S^1\times S^{d-1}\times [-1,1])$.

At the time of submission, we noticed the concurrent work \citep{piening2025slicingWoW} whose sliced construction on $\cPP{\R^d}$ is close to $\swbdg$, even though it is obtained from a different viewpoint. We detail in Appendix \ref{appendix:bg_wow} the relation between the two sliced distances.

\begin{figure}[t]
    \centering
    \includegraphics[width=\linewidth]{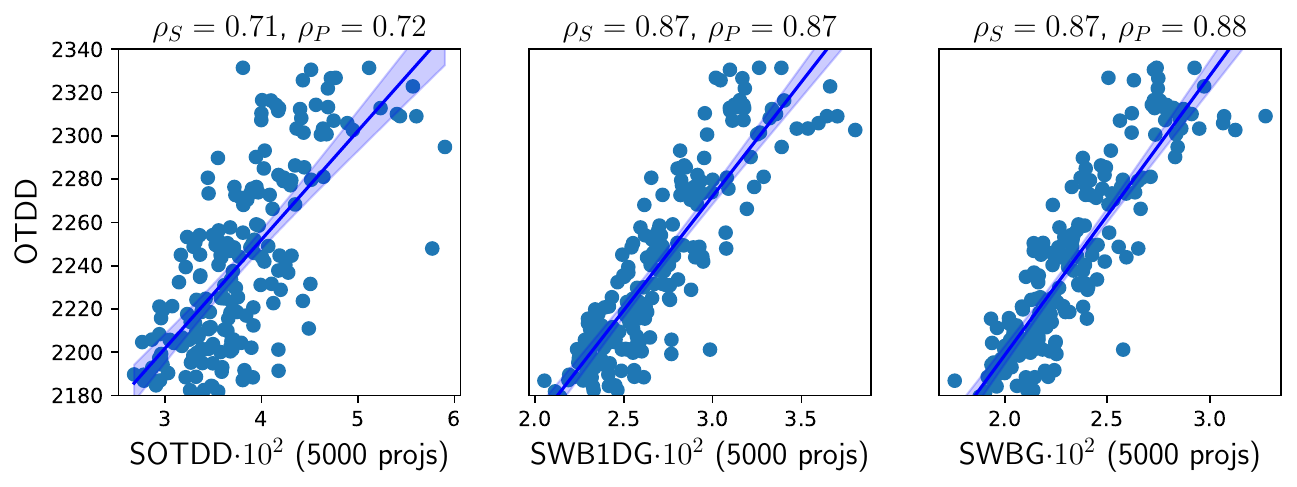}

    \caption{Spearman ($\rho_S)$ and Pearson ($\rho_P$) correlation between $\sotdd$, $\swbdg$, $\swbg$ and OTDD between subdatasets of CIFAR10.}
    \label{fig:correlationcifar10}
\end{figure}

\paragraph{Computational Properties.} 

The sliced distances can be approximated using Monte-Carlo projections. Given $L$ projections, the complexity of $\swbdg$ is $\cO\big(Ln(\log n + d)\big)$, similarly to SOTDD. $\swbg$ is more costly as it requires to compute square roots of matrices, which gives a complexity of $\cO\big(LCd'^{3} + Ln(\log n + d') + Cd'^{2} N_C\big)$ with $N_C=\max_y n_y$. We refer to Appendix \ref{appendix:xp_details} for numerical comparisons. %

\paragraph{Slicing Mixtures.} \looseness=-1 The constructions can also be used to compare distributions on $\cPP{\R^d}$ by setting $\alpha_1=0$. 
In Appendix \ref{appendix:sliced_gm}, we investigate such constructions to compare Gaussian mixtures.

\section{EXPERIMENTS} \label{sec:xps}

\looseness=-1 In this Section, we compare the sliced-based distances on labeled datasets. We first show that $\swbdg$ and $\swbg$ are better replacements for OTDD than SOTDD, as they are more correlated with it. Then, we show that these distances can be used to flow datasets \citep{alvarez2021dataset}, for instance to perform transfer learning. We refer to Appendix \ref{appendix:xp_details} for details\footnote{Code available at \url{https://github.com/clbonet/Busemann_Functions_in_the_Wasserstein_Space}}.

\begin{table}[t]
    \centering
    \caption{Correlation averaged over 10 sets of 50 bootstrapped pairs of subdatasets of CIFAR10, for different number of projections $L$.}
    \resizebox{\linewidth}{!}{
        \begin{tabular}{ccccccccc}
             $L$ & & \multicolumn{3}{c}{Spearman correlation $(\rho_S)$} & & \multicolumn{3}{c}{Pearson correlation ($\rho_P$)}  \\
             & & SOTDD & SWB1DG & SWBG & & SOTDD & SWB1DG & SWBG \\ \toprule
             10 & & $14.0_{\pm 11.3}$ & $\textbf{44.3}_{\pm 10.8}$ & $40.2_{\pm 12.2}$ & & $16.0_{\pm 12.9}$ & $38.6_{\pm 14.6}$ & $\textbf{42.7}_{\pm 9.5}$ \\
             50 & & $30.5_{\pm 12.9}$ & $\textbf{62.6}_{\pm 6.4}$ & $40.4_{\pm 9.8}$ & & $25.2_{\pm 11.4}$ & $\textbf{63.6}_{\pm 6.3}$ & $42.8_{\pm 8.5}$ \\
             100 & & $15.5_{\pm 11.8}$ & $\textbf{71.9}_{\pm 6.4}$ & $68.1_{\pm 7.2}$ & & $21.0_{\pm 11.4}$ & $\textbf{73.9}_{\pm 5.5}$ & $72.8_{\pm 5.4}$ \\
             500 & & $52.1_{\pm 8.1}$ & $\textbf{82.3}_{\pm 2.2}$ & $78.4_{\pm 6.0}$ & & $54.6_{\pm 8.8}$ & $\textbf{83.5}_{\pm 2.1}$ & $79.4_{\pm 7.7}$ \\
             1000 & & $52.0_{\pm 10.9}$ & $83.6_{\pm 4.8}$ & $\textbf{83.7}_{\pm 5.0}$ & & $53.1_{\pm 11.3}$ & $\textbf{85.6}_{\pm 3.5}$ & $84.9_{\pm 4.8}$ \\
             5000 & & $72.2_{\pm 7.5}$ & $88.5_{\pm 4.8}$ & $\textbf{89.3}_{\pm 3.8}$ & & $75.4_{\pm 5.5}$ & $87.8_{\pm 2.8}$ & $\textbf{89.0}_{\pm 2.4}$ \\
             10000 & & $72.6_{\pm 6.1}$ & $82.7_{\pm 4.8}$ & $\textbf{86.7}_{\pm 3.0}$ & & $77.1_{\pm 4.3}$ & $87.3_{\pm 2.8}$ & $\textbf{90.2}_{\pm 2.3}$ \\
             \bottomrule
        \end{tabular}
    }
    \label{tab:correlation_cifar10}
\end{table}

\begin{figure*}[t]
    \includegraphics[width=\linewidth]{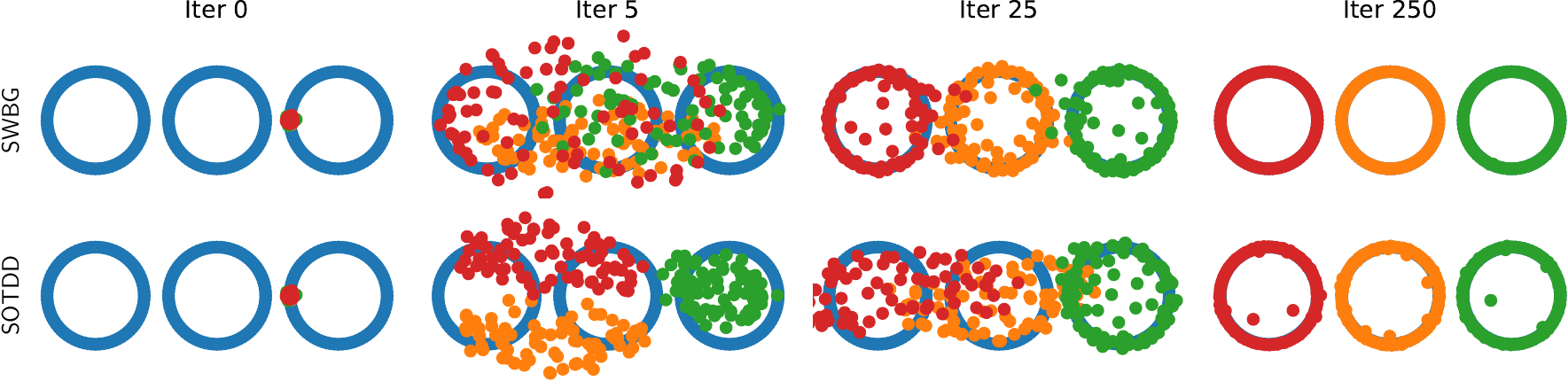}
    \caption{Evolution of the WoW gradient flow of SWBG and SOTDD with the 3-rings dataset as target.}
    \label{fig:flow_rings}
    \vspace{-0.5em}
\end{figure*}

\subsection{Correlation with OTDD}

To show that the sliced distances are suitable proxies to the costly OTDD distance \citep{alvarez2020geometric}, we measure the correlation between OTDD and the sliced distances on image datasets.

\looseness=-1 Following \citep{nguyen2025lightspeed},  we randomly split the CIFAR10 dataset \citep{krizhevsky2009learning} to get subdataset pairs, with sizes ranging from 5000 to 10000 samples, obtaining 200 pairs. Between each pair, we compute OTDD, $\sotdd$, $\swbdg$ and $\swbg$. OTDD and $\sotdd$ are computed using the code shared by \citet{nguyen2025lightspeed}\footnote{at \url{https://github.com/hainn2803/s-OTDD}}. For $\swbg$, we used a TSNE in dimension 10 for $\psi$ using the \texttt{TorchDR} library \citep{torchdr}. The projections are done using convolution projections, as they better capture the spatial structure and are more suitable to compare distributions of images \citep{nguyen2022revisiting}.

We report the results in \Cref{fig:correlationcifar10}, where we first scatter the values obtained in ordinate for OTDD, and in abscissa for the sliced distances with 5000 projections. Then, we report the values of the Spearman and Pearson correlations.%
The Pearson correlation is equal to $\pm1$ when both quantities are linearly correlated while Spearman correlation is equal to $\pm1$ if the quantities are monotonically related. We observe that both $\swbdg$ and $\swbg$ are better correlated to OTDD than $\sotdd$.
We note that the results do not match those in \citep{nguyen2025lightspeed}, where they used only 10 pairs. We hypothesize that for such small number of samples, the Spearman and Pearson correlation are very sensitive to randomness. To verify the robustness of our results, we report in \Cref{tab:correlation_cifar10} the results obtained by bootstrapping 50 pairs of subdatasets, and averaging over 10 experiments. Using this process for different numbers of projections, we see that the Busemann based sliced distances always outperform $\sotdd$, and often under a much smaller number of projections.

\subsection{Flowing Labeled Datasets}

Flowing one dataset onto another is useful to solve tasks ranging from domain adaptation to transfer learning \citep{alvarez2021dataset, hua2023dynamic} or dataset distillation \citep{bonet2025flowing}. This is achieved by minimizing a discrepancy on the space of datasets with respect to a target dataset, and starting from the source dataset. \citet{alvarez2021dataset} proposed to minimize OTDD while \citet{hua2023dynamic} minimized a Maximum Mean Discrepancy. %

\citet{bonet2025flowing} modeled this task as a minimization problem over the space $\cPP{\R^d}$, representing datasets with $n$ samples by class as $\bP=\frac1C\sum_{c=1}^C \delta_{\mu_c}\in \cPP{\R^d}$ with $\mu_c=\frac{1}{n}\sum_{i=1}^n\delta_{x_{i,c}}\mathbb{1}_{\{y_i=c\}}$.  A discrepancy $\bF(\bP)=D(\bP,\bQ)$ with $\bP,\bQ\in \cP_2\big(\cP_2(\R^d)\big)$ can be minimized by a Wasserstein over Wasserstein (WoW) gradient descent on this space. In particular, for $\bP=\frac1C\sum_{c=1}^C \delta_{\mu_c}$, the WoW gradient $\gWw\bF(\bP)$ can be recovered by rescaling the Euclidean gradient of $F(\textbf{x})=\bF(\bP)$ for $\textbf{x}:=(x_{i,c})_{i,c}$ by $nC$, \emph{i.e.} $\gWw\bF(\bP)(\mu_c)(x_{i,c})=nC \nabla F(\textbf{x})_{i,c}$, see \citep[Proposition B.7]{bonet2025flowing}. Given $\bP^k=\frac1C\sum_{c=1}^C \delta_{\mu_c^k}$, $\mu_c^k=\frac1n\sum_{i=1}^n\delta_{x_{i,c}}$, the WoW gradient descent corresponds to updating each particle $x_{i,c}^k$ as
\begin{equation}
    \forall k \ge 0,\  x_{i,c}^{k+1} = x_{i,c}^k - \tau \gWw\bF(\bP^k)(\mu_c^k)(x_{i,c}^k).
\end{equation}

Thus, we propose to minimize the sliced distances using the bijection from $\cPP{\R^d}\to\cPProd{\R^d}$, and performing gradient descent on $\cPP{\R^d}$.

\begin{figure}[t]
    \centering
    \includegraphics[width=\linewidth]{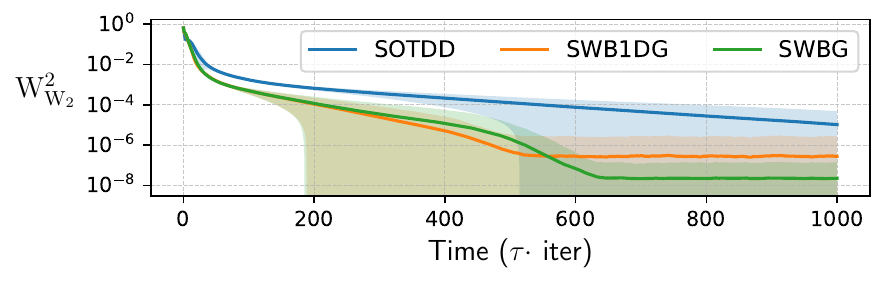}
    \caption{Convergence of the flow towards the 3-rings, averaged over 100 random batches of the target.}
    \label{fig:cv_rings}
\end{figure}

\paragraph{Rings Dataset.}

We first focus on a toy example, where the target dataset contains $C=3$ classes, and each class forms a ring of $n=80$ samples \citep{glaser2021kale,bonet2025flowing}.  We learn a distribution $\bP$ of the same form and with the same number of particles. In \Cref{fig:cv_rings}, we show the convergence of the WoW gradient flows of $\sotdd$, $\swbdg$ and $\swbg$ with the 3-rings dataset as target, averaged over 100 different random batches of the target, and measured with the WoW distance (\emph{i.e.} the OT problem with $\W_2^2$ as ground cost, which we denote $\W_{\W_2}$, see Appendix \ref{appendix:bg_wow}). We use a step size of $\tau=1$ and 1000 iterations. The best performing distance appears to be $\swbg$.  We also report the particles along the flows of $\swbg$ and $\sotdd$ on \Cref{fig:flow_rings}, clearly showing that the flow of $\swbg$ converges faster than the one of $\sotdd$.

\paragraph{Transfer Learning.}

\begin{table}[t]
    \caption{Accuracy of a classifier on augmented datasets for $k\in\{1,5,10,100\}$. M refers to MNIST, F to Fashion MNIST %
    and U to USPS.} %
    \label{tab:results_tf_grid_seach}
    \centering
    \resizebox{\linewidth}{!}{
        \begin{tabular}{cccccc}
             Dataset & $k$ &  Trained on $\mathbb{Q}$ & OTDD & $\swbdg$ & $\sotdd$ \\ \toprule
             \multirow{4}{*}{M to F} 
             & 1 & $26.0_{\pm 5.3}$ & $30.5_{\pm 4.2}$ & $41.3_{\pm 3.4}$ & $\textbf{43.4}_{\pm 2.6}$ \\
             & 5 & $38.5_{\pm 6.7}$ & $59.7_{\pm 1.8}$ & $\textbf{65.5}_{\pm 1.6}$ & $64.5_{\pm 1.2}$ \\
             & 10 & $53.9_{\pm 7.9}$ & $64.0_{\pm 1.4}$ & $66.0_{\pm 0.9}$ & $\textbf{67.7}_{\pm 0.6}$ \\
             & 100 & $71.1_{\pm 1.5}$ & - & $\textbf{74.1}_{\pm 0.6}$ & $72.0_{\pm 1.9}$  \\ \midrule
              \multirow{4}{*}{M to U}
             & 1 & $32.4_{\pm 7.9}$ & $39.5_{\pm 7.9}$ & $45.4_{\pm 4.3}$ & $\textbf{50.1}_{\pm 2.6}$  \\
             & 5 & $51.4_{\pm 9.8}$ & $73.3_{\pm 1.4}$ & $73.5_{\pm 1.4}$ & $\textbf{75.7}_{\pm 0.8}$ \\
             & 10 & $60.3_{\pm 10.1}$ & $72.7_{\pm 2.7}$ & $77.8_{\pm 1.4}$ & $\textbf{80.1}_{\pm 1.0}$ \\
             & 100 & $87.5_{\pm 0.7}$ & - & $\textbf{90.1}_{\pm 0.4}$ & $89.6_{\pm 0.3}$ \\
             \bottomrule
        \end{tabular}
    }
\end{table}

\looseness=-1 We now consider flowing image datasets to solve a $k$-shot transfer learning task \citep{alvarez2021dataset, hua2023dynamic}. In this experiment, we consider a target dataset $\bQ$ with $k$ samples by class, where $k$ is typically small. To improve the classification results, we augment the dataset $\bQ$ by concatenating it with samples flowed from another dataset $\bP_0$, from which we have $n$ samples by class. This is done by minimizing a distance on the space of datasets with respect to $\bQ$, \emph{i.e.} by minimizing $\bF(\bP)=\mathrm{D}(\bP,\bQ)$ starting from $\bP_0$. Any divergence on $\cP_2\big(\R^d\times \cP_2(\R^d)\big)$ may be chosen for $\mathrm{D}$. We focus here on comparing OTDD, SOTDD and SWB1DG for $\mathrm{D}$.

\looseness=-1 We take $\bP_0$ as the MNIST dataset \citep{lecun2010mnist}, and $\bQ$ as Fashion MNIST \citep{xiao2017fashion} or USPS \citep{hull1994database}, and set $k\in \{1,5,10,100\}$ and $n=200$. We minimize $\swbdg$ and $\sotdd$ using the WoW gradient descent with step size $\tau=1$ and momentum $m=0.9$. For the choice of the number of projections for the Monte-Carlo approximation of the sliced distances, and of the number of iterations for the gradient descent, we perform a grid search. We refer to Appendix \ref{appendix:xp_transfer_learning} for more details about the experiment, and for the best parameters selected with the grid search.

Once the dataset $\bP_0$ has been flowed towards $\bQ$, and each class of the flowed dataset $\bP_T$ has been matched to a class of the target dataset $\bQ$ by solving the WoW OT problem, we train a LeNet5 on the augmented concatenated dataset. On \Cref{tab:results_tf_grid_seach}, we report the accuracy of the classifier evaluated on a test set. We report the best results among the grid search for the classifier trained on $\bQ$ augmented with data flowed by minimizing OTDD, $\swbdg$ and $\sotdd$. We also report a baseline where the network was only trained on $\bQ$. The results are averaged over 5 training of the networks, and 3 outputs of the flows for $\swbdg$ and $\sotdd$, and are taken from \citep{bonet2025flowing} for OTDD and the baseline. %

The results are overall comparable between $\swbdg$ and $\sotdd$. Additional details, as well as examples of generated images, are provided in Appendix \ref{appendix:xp_transfer_learning}. In particular, for a large number of iterations, both flows converge to good images. However, note that improved image quality does not necessarily translate into better performance on the transfer task.

Additionally, we report on \Cref{tab:runtime_tf} the runtimes for the transfer learning experiment, averaged over 3 outputs of the flows and trained for 5K epochs. We observe that SOTDD and SWB1DG almost have the same runtime, which is expected as they have the same computational complexity up to constants, and both are much more efficient than OTDD.

\section{CONCLUSION}

\looseness=-1 We studied in this work in which cases the Busemann function is well defined on the Wasserstein space, and how to compute it in practice. More precisely, we identified conditions to  define geodesic rays on the Wasserstein space, showed that the Busemann function can be computed in general by solving an optimal transport problem, and derived closed-form formulas under which it can be computed efficiently. Then, we leveraged these closed-forms to define new efficient Sliced-Wasserstein distances on the space of datasets. Future works will include improving the scalability of $\swbg$, \emph{e.g.} by using Gaussian approximations with low-rank covariances \citep{bouveyron2025scaling}, investigating other applications for the Busemann function such as Principal Component Analysis on the Wasserstein space \citep{cazelles2018geodesic, vesseron2025wasserstein}, or its computation on probabilities over manifolds.

\begin{table}[t]
    \caption{Runtime in seconds for the transfer learning experiment from MNIST to Fashion MNIST.} %
    \label{tab:runtime_tf}
    \centering
    \resizebox{\linewidth}{!}{
        \begin{tabular}{ccccc}
             Dataset & $k$-shot & OTDD & SWB1DG & SOTDD \\ \toprule
             \multirow{4}{*}{M to F} 
             & 1 & $294.53 \pm 5.21$  & $13.53 \pm 0.42$ & $14.07 \pm 0.37$ \\
             & 5 & $1130.89 \pm 108$  & $13.84 \pm 0.34$ & $14.11 \pm 0.21$ \\
             & 10 & $2294.13 \pm 48$  & $14.00 \pm 0.37$ & $14.15 \pm 0.18$ \\
             & 100 & -  & $15.19 \pm 0.49$ & $15.31 \pm 0.49$ \\ 
             \bottomrule
        \end{tabular}
    }
\end{table}

\makeatletter
\subsubsection*{Acknowledgements}

We thank the anonymous reviewers for their valuable comments. CB thanks Pierre-Cyril Aubin-Frankowski for feedbacks on an earlier version of the draft. This work was granted access to the HPC resources of IDRIS under the allocation 2024-AD011015891 made by GENCI. CB and EC acknowledge the support of the Agence nationale de la recherche, through the PEPR PDE-AI project (ANR-23-PEIA-0004). NC was supported by the ANR chair OTTOPIA ANR-20-CHIA-0030 and contributes to  AI Excellence Cluster SequoIA (grant ANR-23-IACL-0009).

\bibliographystyle{plainnat}
\bibliography{references}

\clearpage
\appendix
\thispagestyle{empty}

\onecolumn
\aistatstitle{Supplementary Materials}

\input{appendix}

\end{document}

%% file: macros.tex
\usepackage[dvipsnames]{xcolor}         %
\usepackage{amsmath,amsfonts,amssymb,mathtools,amsthm}
\usepackage{bbm}
\usepackage{booktabs} %
\usepackage{bbold}

\usepackage{graphicx}
\usepackage{float}
\usepackage{subfig}
\usepackage{multirow}

\usepackage[colorlinks, citecolor=blue, linkcolor=red, urlcolor=black, pagebackref=true]{hyperref}
\usepackage{cleveref}

\renewcommand*\backref[1]{\ifx#1\relax \else (Cited on p. #1) \fi}

\newtheorem{definition}{Definition} 

\newtheorem{proposition}{Proposition}

\newtheorem{lemma}{Lemma}
\newtheorem{corollary}{Corollary}

\DeclareMathOperator*{\argmin}{argmin}

\newcommand{\cP}{\mathcal{P}}
\newcommand{\bP}{\mathbb{P}}
\newcommand{\bQ}{\mathbb{Q}}
\newcommand{\bF}{\mathbb{F}}
\newcommand{\bGamma}{\mathbb{\Gamma}}
\newcommand{\Ww}{\W_{\W_2}}
\newcommand{\gWw}{\nabla_{\Ww}}

\newcommand{\cPa}{\mathcal{P}_{2,\mathrm{ac}}(\R^d)}
\newcommand{\cN}{\mathcal{N}}
\newcommand{\cD}{\mathcal{D}}
\newcommand{\cX}{\mathcal{X}}
\newcommand{\cY}{\mathcal{Y}}
\newcommand{\cO}{\mathcal{O}}
\newcommand{\cU}{\mathcal{U}}
\newcommand{\cM}{\mathcal{M}}
\newcommand{\cB}{\mathcal{B}}
\newcommand{\cPP}[1]{\mathcal{P}_2\big(\cP_2(#1)\big)}
\newcommand{\cPBW}{\mathcal{P}_2\big(\bw(\R^d)\big)}
\newcommand{\cPProd}[1]{\mathcal{P}_2\big(#1 \times \cP_2(#1)\big)}

\newcommand{\otdd}{\mathrm{OTDD}}
\newcommand{\sotdd}{\mathrm{SOTDD}}

\newcommand{\W}{\mathrm{W}}
\newcommand{\R}{\mathbb{R}}
\newcommand{\T}{\mathrm{T}}
\newcommand{\X}{\mathrm{X}}

\newcommand{\id}{\mathrm{Id}}
\newcommand{\tr}{\mathrm{Tr}}

\newcommand{\bw}{\mathrm{BW}}
\newcommand{\sw}{\mathrm{SW}}
\newcommand{\swbdg}{\mathrm{SWB1DG}}
\newcommand{\swbg}{\mathrm{SWBG}}
\newcommand{\smwbg}{\mathrm{BGMSW}}
\newcommand{\smwbdg}{\mathrm{B1DGMSW}}

\def\mo{m_{0}}
\def\m1{m_{1}}
\def\so{\sigma_{0}}
\def\s1{\sigma_{1}}
\def\So{\Sigma_{0}}
\def\S1{\Sigma_{1}}

%% file: appendix.tex
\section{BACKGROUND}

We provide in this section additional background on the Busemann function, on the Wasserstein space, on the Wasserstein over Wasserstein space, and finally on Optimal Transport distances to compare labeled datasets.

\subsection{Background on the Busemann Function} \label{app:bg_busemann}

Let $(\X, d)$ be a  geodesic metric space, \emph{i.e.} a metric space, where each $x,y\in\X$ can be linked by a continuous curve $\gamma:[0,1]\to \X$ such that $\gamma(0)=x$, $\gamma(1)=y$ and which satisfies for all $s,t\in [0,1]$, $d\big(\gamma(s),\gamma(t)\big) = |t-s| d\big(\gamma(0),\gamma(1)\big)$.

Suppose that $(\X,d)$ admits geodesic rays, \emph{i.e.} geodesic curves $\gamma:\R_+\to\X$ such that for all $t,s\ge 0$, $d\big(\gamma(t),\gamma(s)\big) = |t-s| d\big(\gamma(0),\gamma(1)\big)$. Well-known spaces in which any geodesic can be extended as a geodesic ray are Hadamard spaces \citep{bridson2013metric}, which are metric space of non-positive curvature, including Hadamard manifolds such as Euclidean spaces, Hyperbolic spaces or the space of Symmetric Positive Definite matrices (SPDs) with appropriate metrics.

The Busemann function $B^\gamma$ associated to the geodesic ray $\gamma$ is defined, for all $x\in\X$, as (see \emph{e.g.} \citep[II. 8.17]{bridson2013metric})
\begin{equation}
    B^\gamma(x) = \lim_{t\to \infty}\ d\big(x,\gamma(t)\big) - d\big(\gamma(0),\gamma(t)\big) = d\big(x, \gamma(t)\big) - t d\big(\gamma(0),\gamma(1)\big).
\end{equation}
This function has attracted a lot of attention as it provides a natural generalization of affine functions, and thus of hyperplanes through its level sets. Indeed, for $\X=\R^d$, $v\in\R^d$ and $\gamma(t)=x + tv$ for all $t\in \R$, the Busemann function is equal to
\begin{equation}
    \forall y\in \R^d,\ B^\gamma(y) = -\left\langle y-x, \frac{v}{\|v\|_2} \right\rangle.
\end{equation}
Its level sets $(B^\gamma)^{-1}(\{t\})$ for $t\in \R$ are called horospheres, and allow to define a generalization of affine hyperplanes beyond the Euclidean space.

The Busemann function can also be computed in closed-form in many spaces, including hyperbolic spaces \citep{chami2021horopca, bonet2023hyperbolic}, the space of SPDs with the Affine-Invariant metric \citep{fletcher2009computing, fletcher2011horoball} or with pullback Euclidean metrics \citep{bonet2025sliced}. However, it has attracted the most attention in spaces where any geodesic is a ray. Thus, in this work, we provide an analysis of this function on the Wasserstein space, which has non-negative curvature, and in which not all geodesics can therefore be extended as rays.

In its original formulation, the Busemann function does not depend on the speed of the geodesic $d\big(\gamma(0),\gamma(1)\big)$ \citep{criscitiello2025horospherically}. For instance, let $\X=\cM$ be a Hadamard manifold. If we consider two geodesic rays $\gamma$ and $\Tilde{\gamma}$ both starting from $x\in\cM$ and with respective speed $v\in T_x\cM$ and $\frac{v}{\|v\|_x}\in T_x\cM$, \emph{i.e.} $\gamma(t)=\exp_x(tv)$ and $\Tilde{\gamma}(t)=\exp_x(tv/\|v\|_x)$, then $B^\gamma(y)=B^{\Tilde{\gamma}}(y)$ for any $y\in \cM$. Indeed, let $y\in \cM$,
\begin{equation}
    \begin{aligned}
        B^{\Tilde{\gamma}}(y) &= \lim_{t\to +\infty}\ d\left(y, \exp_x\left(t\frac{v}{\|v\|_x}\right)\right) - t \\
        &= \lim_{s\to +\infty}\ d\big(y, \exp_x(sv)\big) - s\|v\|_x \qquad \qquad \qquad \text{($s\leftarrow t/\|v\|_x$)} \\
        &= \lim_{s\to +\infty}\ d\big(y, \exp_x(sv)\big) - s d\big(\gamma(0),\gamma(1)\big) \\
        &= B^\gamma(y).
    \end{aligned}
\end{equation}
Thus, it can always be assumed that the geodesics are of unit speed, up to a normalization of its speed. This is in particular the case on the Wasserstein space, leveraging its Riemannian structure.

In term of applications, the Busemann function has been used to perform classification with prototypes \citep{ghadimi2021hyperbolic, durrant2023hmsn, berg2024horospherical, berg2025multi}, to define boundary conditions on manifolds for classification \citep{fan2023horospherical, doorenbos2024hyperbolic} or define layers of neural networks \citep{wang2021laplacian, sonoda2022fully, nguyen2025neural}. It has also been used to define projections on geodesics subspaces to perform Principal Component Analysis on Hyperbolic spaces \citep{chami2021horopca} or on geodesics to define Sliced-Wasserstein distances on manifolds \citep{bonet2023hyperbolic,bonet2025sliced}. Recently, it has also received attention to define notions of convexity and gradients on Hadamard manifolds \citep{criscitiello2025horospherically} and more generally on Hadamard spaces \citep{goodwin2024subgradient}.

\subsection{Background on the Wasserstein space} \label{app:bg_wasserstein_space}

We provide here some additional background on the Wasserstein space and on measure theory by recalling the disintegration of a measure on product spaces.

\paragraph{Optimal Transport.} 

We recall that $\cP_2(\R^d)=\{\mu\in\cP(\R^d),\ \int \|x\|_2^2\ \mathrm{d}\mu(x)<\infty\}$. Then, the Wasserstein distance between $\mu,\nu\in\cP_2(\R^d)$ is defined as
\begin{equation}
    \W_2^2(\mu,\nu) = \inf_{\gamma\in\Pi(\mu,\nu)}\ \int \|x-y\|_2^2\ \mathrm{d}\gamma(x,y),
\end{equation}
with $\Pi(\mu,\nu)$ the set of couplings between $\mu$ and $\nu$. Defining $\pi^1:(x,y)\mapsto x$ and $\pi^2:(x,y)\mapsto y$ the projections on the coordinates, and $\#$ the push forward operator which satisfies for any measurable map $\T:\R^d\to \R^{d'}$ and Borelian $A\in\cB(\R^{d'})$, $(\T_\#\mu)(A) = \mu\big(\T^{-1}(A)\big)$, $\Pi(\mu,\nu)=\{\gamma\in\cP_2(\R^d\times \R^d),\ \pi^1_\#\gamma=\mu,\ \pi^2_\#\gamma=\nu\}$.

Note that Optimal Transport costs can be more generally defined between measures on any measurable spaces $\cX,\cY$ and for any cost $c:\cX\times \cY\to \R$ lower semi-continuous, using the Kantorovich formulation, \emph{i.e.} for $\mu\in\cP(\cX)$, $\nu\in\cP(\cY)$,
\begin{equation}
    \W_c(\mu,\nu) = \inf_{\gamma\in\Pi(\mu,\nu)}\ \int c(x,y)\ \mathrm{d}\gamma(x,y).
\end{equation}

For $\cX=\cY=\R^d$ and $c(x,y)=\|x-y\|_2^2$, $\W_2$ defines a distance, and $(\cP_2(\R^d),\W_2)$ has a formal Riemannian structure. Between $\mu_0,\mu_1\in\cP_2(\R^d)$, we can always define a constant-speed geodesic $t\in[0,1]\mapsto \mu_t$, which satisfies for all $s,t\in [0,1]$, $\W_2(\mu_s,\mu_t)=|t-s|\W_2(\mu_0,\mu_1)$ (see \emph{e.g.} \citep[Theorem 5.27]{santambrogio2015optimal}). In particular, these curves can be written as McCann's displacement interpolation \citep{mccann1997convexity}
\begin{equation}
    \forall t\in [0,1], \mu_t = \big((1-t)\pi^1 + t\pi^2)_\#\gamma,
\end{equation}
with $\gamma\in\Pi_o(\mu,\nu)$ an optimal coupling between $\mu$ and $\nu$.

When $\mu\in\cPa$ is absolutely continuous \emph{w.r.t} the Lebesgue measure, Brenier's theorem \citep{brenier1991polar} states that there is a unique optimal coupling $\gamma$ between $\mu$ and $\nu\in\cP_2(\R^d)$, and that this optimal coupling is supported on a graph of a function, \emph{i.e.} there exists $\T:\R^d\to\R^d$ such that $\T_\#\mu=\nu$ and $\gamma=(\id,\T)_\#\mu$. In this case, the geodesic between $\mu\in\cPa$ and $\nu\in\cP_2(\R^d)$ is of the form
\begin{equation}
    \forall t\in [0,1],\ \mu_t = \big((1-t)\id + t \T)_\#\mu.
\end{equation}
We can also define the notion of exponential map as $\exp_\mu(v) = (\id + v)_\#\mu$ for any $v\in L^2(\mu)$. Using this map, the geodesic can be written as $\mu_t = \exp_\mu\big(t(\T-\id)\big)$ for all $t\in [0,1]$. We can also define its inverse, the logarithm map, as $\log_\mu(\nu)=\T_\mu^\nu-\id$ with $\T_\mu^\nu$ the OT map between $\mu$ and $\nu$.

Thanks to the formal Riemannian structure, we can also define notions of gradients. We refer \emph{e.g.} to \citep{ambrosio2008gradient} for details on their definition and properties.

\paragraph{Bures-Wasserstein Space.}

The Wasserstein distance between two Gaussian has a closed-form, and is named the Bures-Wasserstein distance, \emph{i.e.} for $\mu=\cN(m_\mu,\Sigma_\mu)$, $\nu=\cN(m_\nu,\Sigma_\nu)$ with $m_\mu,m_\nu\in\R^d$ and $\Sigma_\mu,\Sigma_\nu\in S_d^{++}(\R)$,
\begin{equation}
    \W_2^2(\mu,\nu) = \bw^2(\mu,\nu) = \|m_\mu-m_\nu\|_2^2 + \tr\big(\Sigma_\mu + \Sigma_\nu - 2 (\Sigma_\mu^{\frac12}\Sigma_\nu\Sigma_\mu^{\frac12})^{\frac12}\big).
\end{equation}
Gaussian being absolutely continuous measures, there is also a unique OT map which is given by
\begin{equation}
    \forall x\in \R^d,\ \T(x) = m_\nu + A_\mu^\nu(x-m_\mu), \quad \text{with}\quad A_\mu^\nu = \Sigma_\mu^{-\frac12}(\Sigma_\mu^{\frac12}\Sigma_\nu \Sigma_\mu^{\frac12})^{\frac12} \Sigma_\mu^{-\frac12}.
\end{equation}
In particular, since $\T_\#\mu=\nu$, we also have the relation between the covariance matrices $A_\mu^\nu \Sigma_\mu (A_\mu^\nu)^T = \Sigma_\nu$. Moreover, the geodesics are of the form, for any $t\in [0,1]$, $\mu_t = \big((1-t)\id + t\T\big)_\#\mu$. Since the map $x\mapsto (1-t)x+ t \T(x)$ is affine for any $t\in [0,1]$, geodesics stay Gaussian at each time $t$, \emph{i.e.} $\mu_t=\cN(m_t,\Sigma_t)$. Moreover, their closed-forms can be computed, and are given by \citep{altschuler2021averaging}
\begin{equation}
    \begin{cases}
        m_t = (1-t)m_\mu + tm_\nu \\
        \Sigma_t = \big((1-t) I_d + t A_\mu^\nu\big)\Sigma_\mu\big((1-t)I_d + t A_\mu^\nu\big).
    \end{cases}    
\end{equation}

When endowing the space of Gaussian $\bw(\R^d)=\{\cN(m,\Sigma),\ m\in \R^d,\ \Sigma\in S_d^{++}(\R)\}$ with the Bures-Wasserstein distance, the space $(\bw(\R^d), \bw)$ is actually a real Riemannian manifold \citep{bhatia2019bures}, and not just formally. Its tangent space at any $\mu=\cN(m,\Sigma)$ is the space of affine functions with symmetric linear term \citep[Appendix A.1]{diao2023forward}. We identify it here as $T_\mu\bw(\R^d) = \R^d \times S_d(\R)$ with $S_d(\R)$ the space of symmetric matrices in $\R^{d\times d}$. Using this identification, the Riemannian metric is at any $(m_\mu,\Sigma_\mu)\in\R^d\times S_d^{++}(\R)$, $(m,S)\in \R^d\times S_d(\R)$, $\|(m,S)\|_{m_\mu,\Sigma_\mu}^2 = \|m\|_2^2 + \|S\|_{\Sigma_\mu}^2$ where $\|S\|_{\Sigma_\mu}^2 = \tr(S\Sigma_\mu S)$ \citep{takatsu2011wasserstein}. We can also define the notion of exponential map at $\mu=\cN(m_\mu,\Sigma_\mu)$ as, for any $(m, S)\in \R^d\times S_d(\R)$,
\begin{equation}
    \exp_{\mu}(m, S) = \cN\big(m_\mu + m, (I_d + S)\Sigma_\mu(I_d + S)\big).
\end{equation}
As the mean part is Euclidean, we will often just focus on the covariance part, and write $\exp_\Sigma(S)=(I_d + S)\Sigma(I_d + S)$. We can also define the logarithm map, for $\mu=\cN(m_\mu,\Sigma_\mu)$ and $\nu=\cN(m_\nu,\Sigma_\nu)$, as
\begin{equation}
    \log_\mu(\nu) = (m_\nu-m_\mu, A_\mu^\nu - I_d).
\end{equation}
Similarly, we can write $S=\log_{\Sigma_{\mu}}(\Sigma_\nu)=A_\mu^\nu - I_d$. As it is a Riemannian manifold, we can also define notions of Bures-Wasserstein gradients, see \emph{e.g.} \citep{lambert2022variational, diao2023forward}.

This space is in particular of non-negative curvature \citep{takatsu2011wasserstein}, and thus not geodesically complete. Therefore, not any geodesic can be extended towards infinity.

\paragraph{Disintegration.}

We also recall the definition of the disintegration, see \emph{e.g.} \citep[Theorem 5.3.1]{ambrosio2008gradient}, which will be useful in subsequent proofs.

\begin{definition}[Disintegration of a measure] \label{def:disintegration}
    Let $(Y,\mathcal{Y})$ and $(Z,\mathcal{Z})$ be measurable spaces, and~$(X,\mathcal{X})=(Y\times Z,\mathcal{Y}\otimes\mathcal{Z})$ the product measurable space. Then, for~$\mu\in\mathcal{P}(X)$, we denote the marginals as $\mu_Y = \pi^Y_\#\mu$ and $\mu_Z=\pi^Z_\#\mu$, where $\pi^Y$ (respectively $\pi^Z$) is the projection on $Y$ (respectively Z). Then, a~family $\big(K(y,\cdot)\big)_{y\in\mathcal{Y}}$ is a disintegration of $\mu$ if for all $y\in Y$, $K(y,\cdot)$ is a measure on $Z$, for~all $A\in\mathcal{Z}$, $K(\cdot,A)$ is measurable and:
    \begin{equation*}
        \forall g\in C(X),\ \int_{Y\times Z} g(y,z)\ \mathrm{d}\mu(y,z) = \int_Y\int_Z g(y,z)K(y,\mathrm{d}z)\ \mathrm{d}\mu_Y(y),
    \end{equation*}
    where $C(X)$ is the set of continuous functions on $X$. We can note $\mu=\mu_Y\otimes K$. $K$ is a probability kernel if for all $y\in Y$, $K(y,Z)=1$.
\end{definition}
The~disintegration of a measure actually corresponds to conditional laws in the context of probabilities. In the case where $X=\mathbb{R}^d$, we have existence and uniqueness of the disintegration (see \citep[Box 2.2]{santambrogio2015optimal} or \citep[Chapter 5]{ambrosio2008gradient} for the more general case).

\subsection{Backgound on the Wasserstein over Wasserstein Space} \label{appendix:bg_wow}

When working with probability over probability distributions $\cPP{\R^d}$, a natural metric is to use the OT distance with $\W_2$ as groundcost, which we call the Wasserstein over Wasserstein (WoW) distance, \emph{i.e.} for any $\bP,\bQ\in\cPP{\R^d}$,
\begin{equation}
    \Ww^2(\bP,\bQ) = \inf_{\Gamma\in\Pi(\bP,\bQ)} \ \int \W_2^2(\mu,\nu)\ \mathrm{d}\Gamma(\mu,\nu).
\end{equation}
This defines a distance on $\cPP{\R^d}$ \citep{nguyen2016borrowing}, and the resulting space has also a geodesic structure \citep{bonet2025flowing, pinzi2025totally}. Moreover, several recent works have investigated the analog of Brenier's theorem on this space \citep{emami2025optimal, pinzi2025totally, beiglbock2025brenier}.

\paragraph{WoW Gradients.}

\citet{bonet2025flowing} recently defined a notion of gradient on $\big(\cPP{\R^d}, \Ww\big)$, see \citep[Definition 3.3]{bonet2025flowing}. In their paper, they give a more general definition on manifolds, for clarity, we report here the definition on $\cPP{\R^d}$. For any $\gamma\in\cP_2(\R^d\times \R^d)$, denote $\phi^1(\gamma)=\pi^1_\#\gamma$ and $\phi^2(\gamma)=\pi^2_\#\gamma$. Moreover, let $ L^2\big(\bP, T\cP_2(\R^d)\big)=\{\xi:\cP_2(\R^d)\to T\cP_2(\R^d),\ \int \|\xi(\mu)\|_{L^2(\mu)}^2\ \mathrm{d}\bP(\mu) < \infty\}$.

\begin{definition}[WoW Gradient on $\cPP{\R^d}$]
    Let $\bF:\cPP{\R^d}\to \R$. The WoW gradient of $\bF$ at $\bP\in\cPP{\R^d}$, if it exists, is defined as the map $\xi\in L^2\big(\bP, T\cP_2(\R^d)\big)$, which satisfies for any $\bQ\in\cPP{\R^d}$ and $\bGamma\in \{\bGamma\in \cPP{\R^d\times\R^d},\ \phi^1_\#\bGamma=\bP, \ \phi^2_\#\bGamma=\bQ,\ \iint \|x-y\|_2^2\ \mathrm{d}\gamma(x,y)\mathrm{d}\bGamma(\gamma)=\Ww^2(\bP,\bQ)\}$,
    \begin{equation}
        \bF(\bQ) = \bF(\bP) + \iint \langle \xi(\pi^1_\#\gamma)(x), y-x\rangle\ \mathrm{d}\gamma(x,y)\mathrm{d}\bGamma(\gamma) + o\big(\Ww(\bP,\bQ)\big).
    \end{equation}
    In the following, we note $\gWw\bF(\bP)=\xi$ such a gradient.
\end{definition}

Let $\bF:\cPP{\R^d}\to \R$ be a function such that there exists for any discrete measure of the form $\bP=\frac1C\sum_{c=1}^C\delta_{\mu_c^n}$ with $\mu_c^n = \frac1n\sum_{i=1}^n \delta_{x_{i,c}}$, a function $F:\R^{n\times C}\to \R$ satisfying $\bF(\bP)=F(\textbf{x})$ with $\textbf{x}=(x_{i,c})_{i,c}$. The WoW gradient of $\bF$, if well defined, can be obtained by rescaling the Euclidean gradient of $F$. More precisely, $\gWw\bF(\bP)(\mu_c^n)(x_{i,c}) = nC \nabla F(\textbf{x})_{i,c}$ \citep[Proposition B.7]{bonet2025flowing}. In practice, $\nabla F$ can be obtained using backpropagation.

This gradient allows to perform gradient descent on $(\cPP{\R^d},\Ww)$, by the scheme, for any $\tau>0$, 
\begin{equation}
    \forall k\ge 0,\ \bP_{k+1} = \exp_{\bP_k}\big(-\tau \gWw\bF(\bP_k)\big).
\end{equation}
For $\bP = \frac1C\sum_{c=1}^C \delta_{\mu_c}$, the scheme can be obtained by applying to each particle $x_{i,c}^k$ the update,
\begin{equation}
    \forall i\in\{1,\dots,n\},\ c\in\{1,\dots,C\}, \ k\ge 0,\ x_{i,c}^{k+1} = x_{i,c}^k - \tau \gWw\bF(\bP_k)(\mu_c^k)(x_{i,c}^k).
\end{equation}

\paragraph{Computational Properties and Variants.}

The computation of $\Ww$ can be costly. Indeed, for $\bP=\frac1C\sum_{c=1}^C\delta_{\mu_c^n}$ and $\bQ=\frac1C\sum_{c=1}^C \delta_{\nu_c^n}$ two discrete distributions with $\mu_c^n,\nu_c^n\in\cP_2(\R^d)$ empirical distributions with $n$ samples, it is required to first compute $\cO(C^2)$ OT distances with $n$ samples, and a final OT distance with $C$ samples. In general, $C\ll n$, and thus the computational complexity is $\cO(C^2 n^3\log n)$. 

To alleviate this computational burden, several approximations can be used. On one hand, it is possible to use a less costly distance as groundcost, such as the Sliced-Wasserstein distance \citep{baouan2025optimal, piening2025slicing}, reducing the complexity to $\cO(C^2 Ln\log n + C^3\log C)$. One could also use Linear OT \citep{wang2013linear, liu2025wasserstein} as groundcost, hence allowing to compute only $2C$ OT problems instead of $\cO(C^2)$ and reducing the complexity to $\cO(C n^3 \log n)$. This is particularly appealing when there are lots of classes.

\citet{piening2025slicingWoW} very recently proposed a doubled slicing distance to compare distributions on $\cPP{\R^d}$. Given a projection $P^\theta:\R^d\to\R$ for $\theta\in \Theta$, they first project in 1D the distributions $\mu\in\cP_2(\R^d)$ using $\varphi^\theta(\mu)=P^\theta_\#\mu\in\cP_2(\R)$, \emph{i.e.} for $\bP\in\cPP{\R^d}$, $\varphi^\theta_\#\bP\in\cPP{\R}$. Then, they project $\bP^\theta :=\varphi^\theta_\#\bP$ in $\cP_2(L^2([0,1])$ using, for any $\mu\in\cP_2(\R)$, $\phi(\mu)=F_\mu^{-1}$, and use the Sliced-Wasserstein distance on the Hilbert space $L^2([0,1])$ \citep{han2023sliced}. More precisely, they define for $\bP,\bQ\in\cPP{\R^d}$
\begin{equation}
    \mathrm{DSW}_2^2(\bP,\bQ)= \int_{S^{d-1}}\sw_2^2(\phi_\#\varphi^\theta_\#\bP,\phi_\#\varphi^\theta_\#\bQ)\ \mathrm{d}\lambda(\theta),
\end{equation}
with the Sliced-Wasserstein distance defined on $\cP_2\big(L^2([0,1])\big)$. For SW on this space, they use Gaussian on $\cP_2\big(L^2([0,1])\big)$ as slicing measure, which are sampled using Gaussian processes.

Note that given $f\in L^2([0,1])$ and $\theta\in \Theta$, $\mu\sim\bP$ is projected on $\R$ by the map $Q^{f,\theta}(\mu)=\langle f, F_{P^\theta_\#\mu}^{-1}\rangle_{L^2([0,1])}$. While for $\swbdg$, given a geodesic ray $\eta$ on $\cP(\R)$ and $\theta\in \Theta$, the projection is $Q^{\eta,\theta}(\mu)=B^\eta(P^\theta_\#\mu) = -\langle F_1^{-1}-F_0^{-1}, F_{P^\theta_\#\mu}^{-1} - F_0^{-1}\rangle_{L^2([0,1])}$ using \eqref{eq:busemann_1d} and noting $F_0^{-1}$ and $F_1^{-1}$ the quantile functions of $\eta_0$ and $\eta_1$. Thus both $\mathrm{DSW}$ and $\swbdg$ are very similar as they use an inner product on $L^2([0,1])$ for their projection. However, they differ on how to sample the directions and on the projections. In particular, in $\swbdg$, we only sample directions which produce valid geodesic rays, and take elements from $L^2([0,1])$ which are difference of left continuous and non-decreasing functions, while $\mathrm{DSW}$ can sample on the full space $L^2([0,1])$.

\subsection{Background on Optimal Transport Dataset Distances} \label{appendix:otdd}

A labeled dataset $\cD=\{(x_i,y_i)\}_{i=1}^n$ with $x_1,\dots,x_n\in\cX=\R^d$ the features and $y_1,\dots,y_n\in \cY=\{1,\dots,C\}$ their associated labels, which we suppose here to be discrete, can be represented as a probability distribution $\mu_\cD$ over $\cX\times \cY$, \emph{i.e.} $\mu_\cD = \frac{1}{n}\sum_{i=1}^n \delta_{(x_i,y_i)}$. Thus, a natural way to compare labeled datasets is through distances on $\cP(\cX\times \cY)$. 

\paragraph{OTDD.}

Optimal transport distances can be defined on this space if provided a suitable groundcost on $\cX\times \cY$. While there is usually a natural distance on $\cX$, it is less clear which cost to use on $\cY$ as the labels of the classes might be chosen arbitrarily in practice. Thus, \citet{alvarez2020geometric} proposed to embed labels $y\in\cY$ in the space $\cP_2(\cX)$ through their conditional distributions, \emph{i.e.} using an embedding $\varphi:\cY\to\cP_2(\cX)$ defined as $\varphi(y) = \frac{1}{n_y}\sum_{i=1}^n \delta_{x_i}\mathbb{1}_{\{y_i=y\}}$ where $n_y=\sum_{i=1}^n \mathbb{1}_{\{y_i=y\}}$ is the cardinal of the class $y$, and to represent labeled datasets as distributions on $\cX\times \cP_2(\cX)$. One natural groundcost in this space is then the geodesic distance on the product space, defined for any $(x,y), (x',y') \in \cX\times \cY$, as
\begin{equation}
    d_{\cD}\big((x,y), (x',y')\big)^2 = \|x-x'\|_2^2 + \W_2^2\big(\varphi(y), \varphi(y')\big).
\end{equation}
Then, \citet{alvarez2020geometric} proposed to compare two datasets $\cD, \cD'$ using optimal transport with this ground cost, called the Optimal Transport Dataset Distance (OTDD):
\begin{equation}
    \otdd^2(\cD, \cD') = \inf_{\gamma\in\Pi(\mu_\cD,\mu_{\cD'})} \int d_{\cD}\big((x,y),(x',y')\big)^2\ \mathrm{d}\gamma\big((x,y), (x',y')\big).
    \label{def:otdd}
\end{equation}
For $C$ classes, and a maximum of $n_C$ samples by class, OTDD requires solving $\cO(C^2)$ OT problems with $n_C$ samples, and a final OT problem of $Cn_C$ samples, which leads to the complete complexity of $\cO(C^3 n_C^3)$. It is thus a very costly distance to compute. \citet{alvarez2020geometric} hence proposed to approximate it using an entropic regularization for the final OT problem \citep{cuturi2013sinkhorn}, and a Gaussian approximation for the $C^2$ smaller OT problems, reducing the complexity to $\cO\big(Cn_Cd^2+C^2d^3 + \varepsilon^{-2}n_C^2C^2\log (n_C C)\big)$ \citep{dvurechensky2018computational}, which remains quite costly.

\paragraph{Variants of OTDD.}

This prohibitive computational cost motivated the introduction of variants of OTDD. For instance, \citet{liu2025wasserstein} proposed to embed the labels in $\R^d$ using Multidimensional Scaling methods, and to use Linear Optimal Transport \citep{wang2013linear}, which allows computing only $2C$ OT problems with $n_C$ samples. \citet{hua2023dynamic} proposed to use dimension reduction on the labels to do the Gaussian approximation in a lower dimensional space, and to compare the datasets with a Maximum Mean Discrepancy (MMD), while \citet{bonet2025flowing} proposed to represent datasets on $\cPP{\cX}$ and to compare them with a suitable MMD on this space. Note that from this point of view, any distance introduced in Appendix \ref{appendix:bg_wow} could be used to compare datasets.

\paragraph{Slicing OTDD.}

\citet{nguyen2025lightspeed} recently proposed a Sliced-Wasserstein distance on the space $\cPProd{\cX}$ to compare labeled datasets. This requires to construct a projection from $\cX\times\cP_2(\cX)$ to $\R$ to be able to project the distribution of pairs $(x_i,\varphi(y_i))$ onto a distribution in $\cP_2(\R)$. Their construction is based on a projection of the form, for any $(x,y)\in \cX\times \cY$,
\begin{equation}
    P^{\alpha,\theta,\lambda}(x,y) = \alpha_1 P^\theta(x) + \sum_{i=1}^k \alpha_{i+1} \cM^{\lambda_i}\big(P^\theta_\#\varphi(y)\big),
\end{equation}
with $\alpha\in S^{k}$, $P^\theta:\cX\to \R$ and $\cM^\lambda:\cP_2(\R)\to \R$ the moment transform projection, defined for $\lambda\in \mathbb{N}$ and $\mu\in\cP_\lambda(\R)$ as
\begin{equation}
    \cM^{\lambda}(\mu) = \int \frac{x^\lambda}{\lambda!}\ \mathrm{d}\mu(x).
\end{equation}
The random linear combination corresponds to the Hierarchical Hybrid projection \citep{nguyen2024hierarchical}, and allows to combine projection on different space in order to define a projection on a product space. In practice, the $\lambda$ are sampled using a zero-truncated Poisson distribution, but this projection can be numerically unstable when $\lambda$ is too big because of the $\lambda!$.

\section{SLICING GAUSSIAN MIXTURES} \label{appendix:sliced_gm}

\subsection{Background on the Wasserstein over Bures-Wasserstein Space}

Gaussian mixtures can be represented as discrete probability distributions on the space of Gaussian distributions \citep{chen2018optimal, delon2020wasserstein}, \emph{i.e.} as discrete distributions $\bP\in\cPBW$ of the form $\bP=\sum_{k=1}^K \alpha_k \delta_{\mu_k}$ with $\mu_k=\cN(m_k,\Sigma_k)$, $m_k\in \R^d$, $\Sigma_k\in S_d^{++}(\R)$. 
Thus, a natural distance to compare Gaussian mixtures is given by the OT distance with $\bw^2$ as groundcost:
\begin{equation}
    \W_{\bw}^2(\bP, \bQ) = \inf_{\Gamma\in\Pi(\bP,\bQ)}\ \int \bw^2(\mu,\nu)\ \mathrm{d}\Gamma(\mu,\nu).
\end{equation}

This defines a distance on $\cPBW$. As $\bw(\R^d)$ is a Riemannian manifold, $\cPBW$ has a Riemannian structure, and we can define notions such as geodesics or gradients.

\paragraph{Computational Properties.}

To compute it between two discrete Gaussian mixtures $\bP=\frac1K\sum_{k=1}^K \alpha_k\delta_{\mu_k}$ and $\bQ=\sum_{k=1}^K \beta_k\delta_{\nu_k}$ with $\mu_k=\cN(m^\mu_k,\Sigma^\mu_k)$ and $\nu_k=\cN(m^\nu_k,\Sigma^\nu_k)$, it is required to compute $\cO(K^2)$ BW distances, which has a complexity of $\cO(K^2 d^3)$.

\paragraph{Variants.}

To alleviate this computational burden, several methods were proposed. First, \citet{nguyen2024summarizing} proposed to compare Gaussian mixtures seeing them as distributions over the product space $\R^d\times S_d^{++}(\R)$, and endowing $S_d^{++}(\R)$ with the Log-Euclidean metric as in \citep{bonet2023sliced,bonet2025sliced}. However, it is not specifically designed for Gaussian mixtures, and thus they also proposed to use a doubly SW distance.

Let $\bP\in\cPBW$, $\theta\in S^{d-1}$ and $\varphi^\theta:\cP_2(\R^d)\to\cP_2(\R)$ defined as $\varphi^\theta(\mu)=P^\theta_\#\mu$. For any $\mu=\cN(m_\mu,\Sigma_\mu)\sim\bP$, $\varphi^\theta(\mu)=\cN(\langle m_\mu,\theta\rangle, \theta^T\Sigma_\mu \theta)$. Thus, $\bP^\theta=\varphi^\theta_\#\bP \in \cP_2\big(\bw(\R)\big)$. Moreover, the space of 1D Gaussian can be identified as a product space over the means and standard deviations $\R\times\R_+^*$. For $\mu_\theta=\cN(m_\theta,\Sigma_\theta)\in\bw(\R)$, let $\Xi:\bw(\R)\to \R^2$ such that $\Xi(\mu_\theta)=(m_\theta, \sigma_\theta)$. Then $\Xi_\#\bP^\theta\in\cP_2(\R^2)$.
\citet{piening2025slicing,nguyen2024summarizing} proposed to define the Doubly Mixture Sliced-Wasserstein distance (DMSW) between $\bP,\bQ\in\cPBW$ as
\begin{equation}
    \mathrm{DMSW}^2(\bP,\bQ) = \int_{S^{d-1}} \sw_2^2\left(\Xi_\#\varphi^\theta_\#\bP,\Xi_\#\varphi^\theta_\#\bQ\right)\ \mathrm{d}\lambda(\theta),
\end{equation}
where the inner SW is between the distributions of the mean and covariances in $\R^2$.

\subsection{Slicing Gaussian Mixtures with Busemann}

As the Busemann function allows to project any probability distribution onto $\R$, it is natural to use it as a projection to define Sliced-Wasserstein distances for the purpose of comparing mixtures, and in particular mixtures of Gaussian. We discuss here how we can define Sliced-Wasserstein distances based on the closed-forms of the Busemann function in 1D \eqref{eq:busemann_1d} and between Gaussian distributions \eqref{eq:busemann_gaussians}. Note that both constructions are very close to the one presented in \Cref{sec:sliced_busemann}, but specialized to mixtures of Gaussian. We refer to Appendix \ref{appendix:xps_gmms} for numerical experiments.

\paragraph{Busemann on Gaussian.} Leveraging the closed-form of the Busemann function between Gaussians \eqref{eq:busemann_gaussians}, we can project any atom of a Gaussian mixture on $\R$. To define a sliced distance, we only need to construct geodesic rays on which the Busemann function is well defined. We choose to fix $\eta_0=\cN(0,I_d)$, and sample $\eta_1=\cN(m_1,\Sigma_1)$ such that $m_1\in S^{d-1}$, $\Sigma_1\in S_d^{++}(\R)$ with $\Sigma_1^{\frac12}\succeq I_d$ and $\W_2^2(\eta_0, \eta_1)=1$. 

To enforce $\Sigma_1^{\frac12}\succeq I_d$, we remark that it is equivalent to $\log_{I_d}(\Sigma_1) = \Sigma_1^{\frac12} - I_d \succeq 0$, where $\log_{I_d}$ is the logarithm map in $S_d^{++}(\R)$ (see Appendix \ref{app:bg_wasserstein_space} for the definitions of exponential and logarithm map on the Bures-Wasserstein space). Thus we sample $S=\Delta\mathrm{diag}(|\theta|)\Delta^T\in S_d^{++}(\R)$ with $\Delta\in O_d(\R)$ an orthogonal matrix, $\theta \in S^{d-1}$ and define $\Sigma_1 := \exp_{I_d}(S) = (I_d + S)^2$.
To enforce $\W_2^2(\eta_0,\eta_1)=1$, we first observe that
\begin{equation}
    \W_2^2(\eta_0,\eta_1) = \|\log_{\eta_0}(\eta_1)\|_{L^2(\eta_0)}^2 = \|m_1-m_0\|_2^2 + \|\log_{\Sigma_0}(\Sigma_1)\|_{\Sigma_0}^2 = \|m_1\|_2^2 + \|S\|_{\Sigma_0}^2,
\end{equation}
where $\|S\|_{\Sigma_0}^2 = \tr(S\Sigma_0 S)$ \citep{takatsu2011wasserstein}. Therefore, we simply normalize the vectors in the tangent space to obtain $\W_2^2(\eta_0,\eta_1)=1$, \emph{i.e.}
\begin{equation}
    \begin{cases}
        m_1 = \frac{m_1}{\sqrt{\|m_1\|_2^2 + \|S\|_{\Sigma_0}^2}} \\
        S = \frac{S}{\sqrt{\|m_1\|_2^2 + \|S\|_{\Sigma_0}^2}}.
    \end{cases}
\end{equation}

Defining $\lambda = \cU(S^{d-1} \times O_d(\R) \times S^{d-1}), \ \vartheta=(m_1,\Delta,\theta)$ and $\eta^\vartheta$ the resulting geodesic ray starting from $\eta_0$ and passing through $\eta_1$, we define for $\bP,\bQ$ Gaussian mixtures, the Busemann Gaussian Mixture Sliced-Wasserstein distance ($\smwbg$) as
\begin{equation}
    \smwbg^2(\bP,\bQ) = \int \W_2^2(B^{\eta^{\vartheta}}_\#\bP, B^{\eta^{\vartheta}}_\#\bQ)\ \mathrm{d}\lambda(\vartheta).
\end{equation}

\paragraph{Busemann on 1D Gaussian.}

$\smwbg$ requires to compute the Busemann function between Gaussian distributions, which might be computationally heavy in high dimension. Thus, we also propose a second distance, by first projecting the Gaussian in 1D, and then leveraging the closed-form of the Busemann function for 1D Gaussians \eqref{eq:1d_busemann_gaussian}. More precisely, we consider $\bP=\sum_{k=1}^K \alpha_k \delta_{\mu^k}$ with $\mu^k=\cN(m_k,\Sigma_k)$. Denoting $\varphi^\theta(\mu)=P^\theta_\#\mu$ where $P^\theta=\langle \theta,\cdot\rangle$, we have $\bP^\theta:=\varphi^\theta_\#\bP = \sum_{k=1}^K \alpha_k\delta_{P^\theta_\#\mu_k}$ with $P^\theta_\#\mu_k = \cN(\langle m_k, \theta\rangle, \theta^T \Sigma_k \theta)$ a one dimensional Gaussian distribution. 

For the choice of the geodesic ray, we set $\eta_0=\cN(0,1)$ and $\eta_1=\cN(m_1,\sigma_1^2)$, where we sample $m_1\in [-1,1]$ uniformly, and fix $\sigma_1=1 + \sqrt{1-m_1^2}$ to enforce $\sigma_1\ge \sigma_0$ and $\W_2^2(\eta_0,\eta_1)=1$. Defining $\lambda = \cU(S^{d-1} \times [-1,1])$, we define the Busemann 1D Gaussian Mixture Sliced-Wasserstein distance ($\smwbdg$) as
\begin{equation}
    \smwbdg^2(\bP,\bQ) = \int \W_2^2(B^{\eta^{m_1}}_\#\varphi^\theta_\#\bP, B^{\eta^{m_1}}_\#\varphi^\theta_\#\bQ)\ \mathrm{d}\lambda(\theta, m_1).
\end{equation}

We notice that this distance resembles $\mathrm{DMSW}$. Indeed, for $\mathrm{DMSW}$, the projection of $\mu=\cN(m,\Sigma)\sim \bP$ on $\R$ is given, for $\theta_d\in S^{d-1},\theta_2\in S^1$, by 
\begin{equation}
    Q^{\theta_d,\theta_1}(\mu) = \left\langle \theta_2, \begin{pmatrix} \langle m,\theta_d\rangle \\ \theta_d^T \Sigma\theta_d \end{pmatrix} \right\rangle,
\end{equation}
and $\theta_2,\theta_d$ are sampled uniformly on the sphere. On the other hand, for $\smwbdg$, the projection is, for $\eta$ a geodesic ray on $\bw(\R)$ and $\theta\in S^{d-1}$,
\begin{equation}
    \begin{aligned}
        Q^{\eta,\theta}(\mu) = B^\eta(P^\theta_\#\mu) &= B^\theta\big(\cN(\langle m,\theta\rangle, \theta^T \Sigma\theta)\big) \\
        &= -\left\langle \begin{pmatrix} m_1 \\ \sigma_1 - 1 \end{pmatrix}, \begin{pmatrix} \langle m,\theta\rangle \\ \theta^T \Sigma\theta - 1 \end{pmatrix} \right\rangle.
    \end{aligned}
\end{equation}
Therefore $\mathrm{DMSW}$ and $\smwbdg$ are very similar. The main difference is that the directions in $\R^2$ for $\smwbdg$ are centered around $(0,1)$.

\subsection{Properties of our Proposed Distances}

The proofs of this section can be found in Appendix \ref{appendix:proofs_appendix}.

\paragraph{Theoretical Properties.}

We first have that $\smwbg$ is a pseudo distance.
\begin{proposition}
    $\smwbg$ is a pseudo-distance on $\bw(\R^d)$.
\end{proposition}
Showing that $\smwbg$ is a distance would require to show that the Busemann function on $\bw(\R^d)$ allows defining an injective Radon transform, which is outside the scope of this work.

Concerning $\smwbdg$, we can exploit that the 1D Gaussian space is actually Euclidean, and thus we can show that it is a well-defined distance, which is also bounded by $\W_{\bw}$.

\begin{proposition} \label{prop:bound_bdgmsw}
    For any $\bP,\bQ\in \cP_2\big(\bw(\R^d)\big)$,
    \begin{equation}
        \smwbdg^2(\bP,\bQ) \le \W_{\bw}^2(\bP,\bQ).
    \end{equation}
\end{proposition}

\begin{proposition} \label{prop:bdgmsw_distance}
    $\smwbdg$ is a distance on the space of discrete Gaussian mixtures $\bigcup_{K>0} \mathrm{GMM}_d(K)$, with $\mathrm{GMM}_d(K)$ the set of Gaussian mixtures on $\R^d$ with $K$ components.
\end{proposition}

\paragraph{Computational Properties.}

On a computational point of view, we approximate $\smwbg$ by a Monte-Carlo approximation with $L$ projections. Given discrete mixtures $\bP_k=\sum_{k=1}^K \alpha_k \delta_{\mu_k}$ and $\bQ_k=\sum_{k=1}^K \beta_k \delta_{\nu_k}$, $\smwbg$ requires to project $K$ Gaussian with the Busemann function, which has a complexity of $\cO(LKd^3)$. Then, it is also required to solve a 1D OT transport problem with $K$ samples. Thus, the full complexity is $\cO\big(LK(d^3+\log K)\big)$ which can be costly in high dimension. On the other hand, $\smwbdg$ can be approximated with a complexity of $\cO\big(LK(\log K + d)\big)$, which is much cheaper that $\smwbg$.

\section{ADDITIONAL RESULTS} \label{app:add_result}

In the case of arbitrary 1D Gaussian distributions, in addition to \Cref{corr:1d_gaussian_geod}, we can obtain the largest interval over which the geodesic can be extended, even in the direction $t<0$. The proof can be found in Appendix \ref{appendix:proofs_appendix}.
\begin{proposition} \label{prop:extension_1dgaussian_ray}
    Let $\mu_0=\mathcal{N}(m_0,\sigma_0^2)$ and $\mu_1=\mathcal{N}(m_1, \sigma_1^2)$ two Gaussian distributions such that $\sigma_1 > \sigma_0$. %
    Then, the geodesic $t\mapsto \mu_t$ is well defined on $]-\frac{\sigma_0}{\sigma_1-\sigma_0}, +\infty[$. By symmetry, if $\sigma_1 < \sigma_0$, the geodesic is well defined on $]-\infty, \frac{\sigma_0}{\sigma_0-\sigma_1}[$.
\end{proposition}
In the setting of Proposition \ref{prop:extension_1dgaussian_ray}, the geodesic is given by $\mu_t=\cN\big((1-t)m_0 + t m_1, ((1-t)\sigma_0 + t\sigma_1)^2\big)$. Thus, we observe that the geodesic property breaks at time $t=-\sigma_0/(\sigma_1-\sigma_0)$, for which $(1-t)\sigma_0+t\sigma_1=0$, \emph{i.e.} the geodesic reaches a Dirac. This time is given by \Cref{prop:extension_1dgaussian_ray} for $\sigma_1\neq\sigma_0$. Note that in the limit case $\sigma_0=\sigma_1$, the geodesic is defined on $\R$ and is a translation. %

We also notice that the curve $t\mapsto \cN\big((1-t)m_0 + t m_1, ((1-t)\sigma_0 + t\sigma_1)^2\big)$ could be extended on $\R$, by allowing $(1-t)\sigma_0+t\sigma_1\le 0$. This curve is geodesic on the right hand side and left hand side of $t_e = \mathrm{sign}(\sigma_1-\sigma_0)\frac{\sigma_0}{\sigma_1-\sigma_0}$. Thus, it is piecewise geodesic on both sides of the Dirac.

\paragraph{Projections on Geodesic Rays in the 1D Gaussian Case.}

Let $\mu_0=\cN(m_0,\sigma_0^2)$, $\mu_1=\cN(m_1,\sigma_1^2)$ with $\sigma_1>\sigma_0$ and such that $\W_2^2(\mu_0,\mu_1)=1$. As discussed in \Cref{section:bg_busemann}, the projection of $\nu=\cN(m,\sigma^2)$ on the geodesic given by $\mu_t = \cN\big((1-t)m_0 + t m_1, ((1-t)\sigma_0 + t\sigma_1)^2\big)$ is $P^\mu(\nu)= \mu_{-B^\mu(\nu)}$. However, if $-B^\mu(\nu)<0$, this projection might be out of the original geodesic ray passing through $\mu_0$ and $\mu_1$. In particular, by \Cref{prop:extension_1dgaussian_ray}, we know that this happens when $B^\mu(\nu)>\frac{\sigma_0}{\sigma_1-\sigma_0}$. 

Here we work with a unit-speed geodesic ray, that is $\W_2^2(\mu_0,\mu_1)=(m_1-m_0)^2 + (\sigma_1-\sigma_0)^2=1$ and thus $m_1-m_0,\sigma_1-\sigma_0 \in [-1,1]$. We have two limiting cases. The first one is $\sigma_0=\sigma_1$ for which the geodesic ray is actually a line and can be extended to $\mathbb{R}$, which we recover here as $-\frac{\sigma_0}{\sigma_1-\sigma_0}\xrightarrow[\sigma_1 \to \sigma_0^+]{}-\infty$. The second one is $\sigma_1 = 1+\sigma_0$ for which the ray can be extended to $[-\sigma_0, +\infty[$ and corresponds to a dilation. Moreover, in this case, since $\sigma_1=1+\sigma_0$ and $m_1=m_0$, we note that any 1D Gaussian will be projected on the geodesic since, for any $\nu=\cN(m,\sigma^2)$, 
\begin{equation}
    \begin{aligned}
        B^\mu(\nu) &= -(m-m_0)(m_1-m_0) - (\sigma-\sigma_0)(\sigma_1-\sigma_0) = -(\sigma-\sigma_0),
    \end{aligned}
\end{equation}
and thus the projection coordinate is $-B^\mu(\nu) = \sigma-\sigma_0 < -\sigma_0 \iff \sigma < 0$, which is not possible.

Likewise, for $\sigma_1<\sigma_0$, the geodesic can be extended towards $-\infty$ and in the case of $m_0=m_1$, the distributions are also necessarily well projected on the geodesic since $\sigma_1=\sigma_0-1$ and $B^\mu(\nu)=\sigma-\sigma_0$. Thus, $-B^\mu(\nu)=\sigma_0-\sigma > \sigma_0 \iff -\sigma>0$.

We illustrate these observations on \Cref{fig:illu_centered_gaussian}. We choose $\mu_0=\mathcal{N}(0,1)$ and $\mu_1=\mathcal{N}(0, \sigma_1^2)$ with $\sigma_1=\frac12$ or $\mu_1=\mathcal{N}(0,\sigma_1^2)$ with $\sigma_1=\frac32$. In the first case, this does not define a geodesic ray (even if the geodesic can be extended toward $-\infty$), but it does in the second case. We plot the projections of several Gaussian $\nu_i=\mathcal{N}(0,\sigma_i^2)$ (where the $\sigma_i$ are plotted in the line $t=0$). We see that every points are projected on the geodesic as expected by \Cref{prop:extension_1dgaussian_ray}.

On \Cref{fig:illu_non_centered_gaussian}, we choose $\mu_0=\mathcal{N}(0,1)$ and $\mu_1=\mathcal{N}(m_1, \sigma_1^2)$ with $m_1=-\sqrt{1-(\sigma_1-\sigma_0)^2}$. We observe that some points projected onto the extended part of the geodesic are less consistent with the geometry.

\begin{figure}[t]
    \centering

    \hspace*{\fill}
    \subfloat[$\sigma_0=1$, $\sigma_1=\frac12$]{\includegraphics[width=0.3\columnwidth]{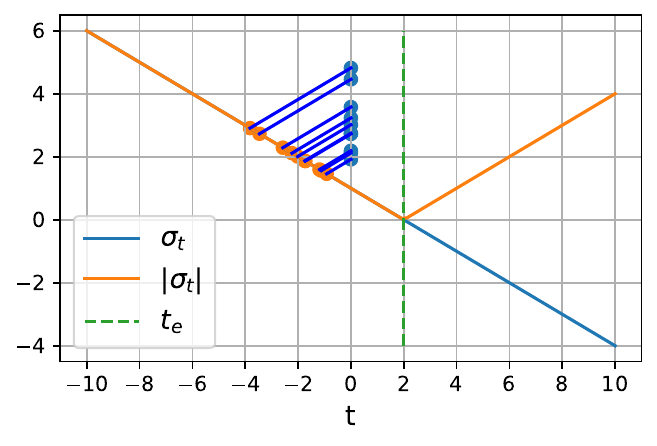}} \hfill
    \subfloat[$\sigma_0=1$, $\sigma_1=\frac32$]{\includegraphics[width=0.3\columnwidth]{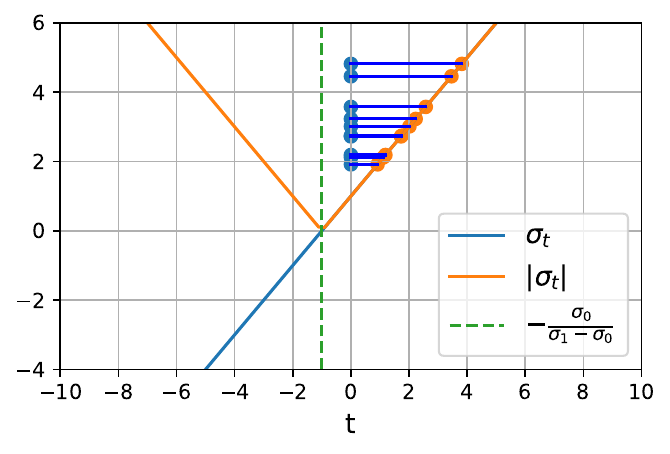}} \hfill
    \hspace*{\fill}
    \caption{Projections of centered 1D Gaussian $\mathcal{N}(0,\sigma_i^2)$ (blue points) on the geodesic $\mathcal{N}(0,\sigma_t)$ starting at $\sigma_0=1$ and passing through $\sigma_1=\frac12$ (\textbf{Left}) and $\sigma_1=\frac32$ (\textbf{Right}).}
    \label{fig:illu_centered_gaussian}
\end{figure}

\begin{figure}[t]
    \centering

    \hspace*{\fill}
    \subfloat[$\sigma_0=1$, $\sigma_1=\frac12$]{\includegraphics[width=0.3\columnwidth]{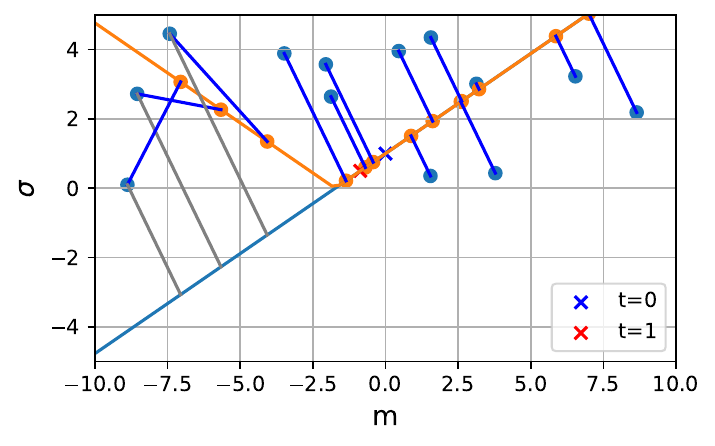}} \hfill
    \subfloat[$\sigma_0=1$, $\sigma_1=\frac32$]{\includegraphics[width=0.3\columnwidth]{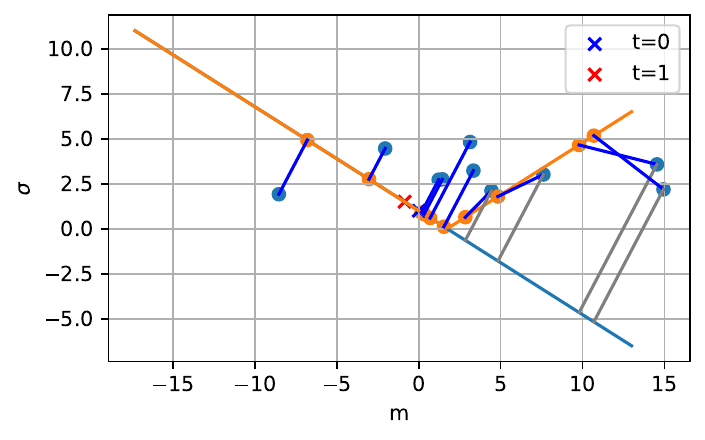}} \hfill
    \hspace*{\fill}
    \caption{Projections of 1D Gaussian $\mathcal{N}(m_i,\sigma_i^2)$ (blue points) on the geodesic $\mathcal{N}(m_t,\sigma_t)$ starting at $\sigma_0=1$, $m_0=0$ and passing through $\sigma_1=\frac12$ (\textbf{Left}) and $\sigma_1=\frac32$ (\textbf{Right}) with $m_1=-\sqrt{1-(\sigma_1-\sigma_0)^2}$.}
    \label{fig:illu_non_centered_gaussian}
\end{figure}

    We finally notice that when the first moments of $\mu_0$ and $\mu_1$ coincide, then the conditions to have geodesic rays are similar to conditions to have $\mu_0$ smaller than $\mu_1$ in the convex order, see \emph{e.g.} \citep[Theorem 4 and 6]{muller2001stochastic} for Gaussians or \citep{shu2020hopf} in 1D for $\mu_0\in\cPa$.

\section{PROOFS}

\subsection{Proofs of \Cref{section:geodesic_rays}}

\subsubsection{Proof of \Cref{prop:geodesic_rays}} \label{proof:prop_geodesic_rays}

\begin{proof}[Proof of \Cref{prop:geodesic_rays}]

    Thanks to Brenier's theorem \citep{brenier1991polar} and since $\mu_0$ is absolutely continuous with respect to the Lebesgue measure, there is a unique OT map $T$ between $\mu_0$ and $\mu_1$, and $T$ is the gradient of a convex function, \emph{i.e.} $T=\nabla u$ with $u$ convex.

    First, let us suppose that the OT map $T$ between $\mu_0$ and $\mu_1$ is the gradient of a 1-convex function $u$. Let $\mu:t\mapsto \mu_t:=\big((1-t)\id + tT)_\#\mu_0 = \big((1-t)\id + t\nabla u\big)_\#\mu_0$. Then, on one hand, we have
    \begin{equation}
        \W_2^2(\mu_s,\mu_t) \le (t-s)^2\W_2^2(\mu_0,\mu_1).
    \end{equation}
    Indeed, let $\gamma^*\in\Pi_o(\mu_0,\mu_1)$ be an optimal coupling. Then, necessarily, denoting $\pi^s(x,y) = (1-s) x + s y$, we have for any $s,t\in\mathbb{R}$, $(\pi^s,\pi^t)_\#\gamma^* \in \Pi(\mu_s,\mu_t)$. Therefore,
    \begin{equation}
        \begin{aligned}
            \W_2^2(\mu_s,\mu_t) &\le \int \|x-y\|_2^2\ \mathrm{d}(\pi^s,\pi^t)_\#\gamma^*(x,y) \\
            &= \int \|(1-s)x+sy - (1-t)x-ty\|_2^2\ \mathrm{d}\gamma^*(x,y) \\
            &= (s-t)^2 \W_2^2(\mu_0,\mu_1).
        \end{aligned}
    \end{equation}
    Then, let $\alpha\ge 1$ and $0\le s<t\le \alpha$. By the triangular inequality and the previous inequality, we have
    \begin{equation}
        \begin{aligned}
            \W_2(\mu_0, \mu_\alpha) &\le \W_2(\mu_0, \mu_s) + \W_2(\mu_s,\mu_t) + \W_2(\mu_t, \mu_\alpha) \\
            &= (s+\alpha-t)\W_2(\mu_0,\mu_1) + \W_2(\mu_s,\mu_t).
        \end{aligned}
    \end{equation}
    If $x\mapsto (1-\alpha)\frac{\|x\|_2^2}{2} + \alpha u(x)$ is convex (\emph{i.e.} $u$ is $\frac{\alpha-1}{\alpha}$-convex), then its gradient $x\mapsto (1-\alpha) x + \alpha\nabla u(x)$ is the Monge map between $\mu_0$ and $\mu_\alpha$ as $\mu_\alpha = \big((1-\alpha)\id + \alpha \nabla u\big)_\#\mu_0$, and thus $\W_2^2(\mu_0, \mu_\alpha) = \alpha^2 \W_2(\mu_0, \mu_1)$. Hence, we obtain
    \begin{equation}
        \W_2(\mu_0, \mu_\alpha) = \alpha \W_2(\mu_0,\mu_1) \le (s+\alpha-t)\W_2(\mu_0,\mu_1) + \W_2(\mu_s,\mu_t) \iff (t-s)\W_2(\mu_0,\mu_1) \le \W_2(\mu_s,\mu_t).
    \end{equation}
    It allows to conclude that $\W_2(\mu_s,\mu_t) = |t-s|\W_2(\mu_0,\mu_1)$ for all $s,t\in [0,\alpha ]$.
    In order to extend the result on $\mathbb{R}_+$, it has to be true for any $\alpha \ge 1$, which corresponds to $u$ being 1-convex. Thus, we can conclude that $t\mapsto \mu_t$ is a geodesic ray.

    For the inverse implication, suppose that $\mu_t = \big((1-t)\id +t T\big)_\#\mu_0$ is a geodesic ray. Then, for all $s\ge 0$,
    \begin{equation}
        \begin{aligned}
            \W_2^2(\mu_s, \mu_0) &= s^2 \W_2^2(\mu_0, \mu_1) \\
            &= \int \| s (x - \nabla u(x))\|_2^2 \ \mathrm{d}\mu_0(x) \\
            &= \int \| x - (1-s)x - s \nabla u(x) \|_2^2 \ \mathrm{d}\mu_0(x) \\
            &= \int \| x - T_s(x) \|_2^2 \ \mathrm{d}\mu_0(x),
        \end{aligned}
    \end{equation}
    where $(T_s)_\#\mu_0 = \mu_s$ with $T_s:x\mapsto (1-s)x+s\nabla u(x)$. By Brenier's theorem, since the OT map is unique and necessarily the gradient of a convex functions, we have that $T_s = \nabla u_s$ with $u_s: x \mapsto (1-s)\frac{\|x\|_2^2}{2} + s u(x) = \frac{\|x\|_2^2}{2} + s \left(u(x)-\frac{\|x\|_2^2}{2} \right)$ convex. 
    Thus, for all $s\ge 0$, 
    \begin{equation}
        I_d + s (\nabla^2 u - I_d) \succeq 0 \iff \nabla^2 u - I_d \succeq - \frac{1}{s} I_d.
    \end{equation}
    It is true for all $s\ge 0$, hence taking the limit $s\to\infty$, we obtain  $\nabla^2 u - I \succeq 0$, \emph{i.e.} $u$ is 1-convex.
\end{proof}

\subsubsection{Proof of \Cref{prop:1d_geodesic_rays}} \label{proof:prop_1d_geodesic_rays}

\begin{proof}[Proof of \Cref{prop:1d_geodesic_rays}]
    By \citep[Equation 7.2.8]{ambrosio2008gradient}, the quantile of $\mu_t$ is $F_t^{-1} = (1-t)F_0^{-1} + t F_1^{-1}$. Then, we know that $F_t^{-1}$ is a quantile function if and only if it is non-decreasing and left-continuous. As a linear combination of left-continuous function, it is always left-continuous. It suffices then to find conditions under which $F_t^{-1}$ is non-decreasing for all $t\ge 0$.
    Let $0<m<m'<1$, then
    \begin{equation}
        \begin{aligned}
            F_t^{-1}(m) -F_t^{-1}(m') = F_0^{-1}(m)-F_0^{-1}(m') + t\big(F_1^{-1}(m)-F_0^{-1}(m) - F_1^{-1}(m') + F_0^{-1}(m')\big),
        \end{aligned}
    \end{equation}
    and hence, $\forall t\ge 0$
    \begin{equation}
        \begin{aligned}
            \forall m'>m,\ F_t^{-1}(m)-F_t^{-1}(m') \le 0 &\iff \forall m'>m,\  F_1^{-1}(m)-F_0^{-1}(m) \le F_1^{-1}(m')-F_0^{-1}(m') \\ &\iff F_1^{-1}-F_0^{-1} \text{ non-decreasing}. \\
        \end{aligned}
    \end{equation}
\end{proof}

\subsubsection{Proof of \Cref{corr:1d_discrete_geod}} \label{proof:corr_1d_discrete_geod}

\begin{proof}[Proof of \Cref{corr:1d_discrete_geod}]    
    We apply \Cref{prop:1d_geodesic_rays}, and thus the resulting geodesic is a ray if and only if $F_1^{-1}-F_0^{-1}$ is non-decreasing, which is true if and only if for all $j>i$,
    \begin{equation}
        \begin{aligned}
            F_1^{-1}\left(\frac{i}{n}\right) - F_0^{-1}\left(\frac{i}{n}\right) \le F_1^{-1}\left(\frac{j}{n}\right) - F_0^{-1}\left(\frac{j}{n}\right) & \iff F_1^{-1}\left(\frac{i}{n}\right) - F_1^{-1}\left(\frac{j}{n}\right) \le F_0^{-1}\left(\frac{i}{n}\right) - F_0^{-1}\left(\frac{j}{n}\right) \\
            &\iff y_i - y_j \le x_i - x_j.
        \end{aligned}
    \end{equation}
\end{proof}

\subsubsection{Proof of \Cref{corr:1d_gaussian_geod}} \label{proof:corr_1d_gaussian_geod}

\begin{proof}[Proof of \Cref{corr:1d_gaussian_geod}]
    Let $\mu_0=\mathcal{N}(m_0,\sigma_0^2)$ and $\mu_1=\mathcal{N}(m_1, \sigma_1^2)$ with $m_0,m_1\in\mathbb{R}$ and $\sigma_0,\sigma_1\in\mathbb{R}_+$. It is well known that for $p\in [0,1]$, $F_{0}^{-1}(p) = m_0 + \sigma_0 \phi^{-1}(p)$ where $\phi^{-1}$ denotes the quantile function of the standard Gaussian distribution $\mathcal{N}(0,1)$. In this case, for $0<p<p'<1$, we observe that 
    \begin{equation}
        F_0^{-1}(p')-F_0^{-1}(p) = \sigma_0 \big(\phi^{-1}(p')-\phi^{-1}(p)\big),
    \end{equation}
    and therefore
    \begin{equation}
        \begin{aligned}
            (F_1^{-1}-F_0^{-1})(p') - (F_1^{-1}-F_0^{-1})(p) &= \big(F_1^{-1}(p')-F_1^{-1}(p)\big)-\big(F_0^{-1}(p')-F_0^{-1}(p)\big) \\
             &= (\sigma_1-\sigma_0)\big(\phi^{-1}(p')-\phi^{-1}(p)\big).
        \end{aligned}
    \end{equation}
    Since $\phi^{-1}$ is non-decreasing, $F_1^{-1}-F_0^{-1}$ is non-decreasing if and only if $\sigma_0\le \sigma_1$. Thus, by \Cref{prop:1d_geodesic_rays}, $\sigma_0\le\sigma_1$ is a sufficient condition to define a geodesic ray starting from $\mu_0$ and passing through $\mu_1$.
\end{proof}

\subsubsection{Proof of \Cref{corr:gaussian_geod}} \label{proof:corr_gaussian_geod}

\begin{proof}[Proof of \Cref{corr:gaussian_geod}]
    Let $\mu_0=\mathcal{N}(m_0,\Sigma_0)$ and $\mu_1=\mathcal{N}(m_1,\Sigma_1)$ with $m_0,m_1\in\mathbb{R}^d$ and $\Sigma_0,\Sigma_1$ symmetric positive definite matrices. The Monge map between $\mu_0$ and $\mu_1$ is \citep[Remark 2.31]{peyre2019computational} 
    \begin{equation}
        \forall x\in \mathbb{R}^d,\ \T(x) = A(x-m_0) + m_1,
    \end{equation}
    where $A=\Sigma_0^{-\frac12}\big(\Sigma_0^{\frac12}\Sigma_1\Sigma_0^{\frac12}\big)^\frac12 \Sigma_0^{-\frac12}$. Let $u:x\mapsto \frac12 \langle Ax,x\rangle + \langle m_1-Am_0,x\rangle = \frac12 \|A^{\frac12}x\|_2^2 + \langle m_1-Am_0,x\rangle$. Note that we have $\nabla u = \T$. Let us denote $g:x\mapsto u(x) - \frac{\|x\|_2^2}{2}$. Then, $u$ is 1-convex if and only if $\nabla^2 g \succeq 0$ (with $\succeq$ the partial order, also called the Loewner order), \emph{i.e.}
    \begin{equation} \label{eq:geodesic_rays_gaussians}
        \begin{aligned}
            \nabla^2 g(x) = A-I_d \succeq 0 &\iff A \succeq I_d \\
            &\iff \Sigma_0^{\frac12} A \Sigma_0^{\frac12} \succeq \Sigma_0^{\frac12} I_d \Sigma_0^{\frac12} \quad \text{\emph{e.g.} by \citep[Lemma V.1.5]{bhatia2013matrix}} \\
            &\iff \big(\Sigma_0^{\frac12}\Sigma_1\Sigma_0^{\frac12}\big)^{\frac12} \succeq \Sigma_0.
        \end{aligned}
    \end{equation}
\end{proof}

\subsection{Proofs of \Cref{section:busemann}} \label{proofs:busemann}

\subsubsection{Proof of \Cref{prop:busemann_general}} \label{proof:prop_busemann_general}

First, we prove the following lemma relating the OT problem between the measures $\mu_t$ and $\nu$, and the problem over couplings of $(\mu_0,\mu_1,\nu)$.
\begin{lemma} \label{lemma:ot_geod}
    Let $(\mu_t)_{t\ge 0}$ be a geodesic ray. Let $t\ge 0$ and $\nu\in\mathcal{P}_2(\mathbb{R}^d)$. Then, 
    \begin{equation}
        \W_2^2(\mu_t,\nu) = \inf_{\gamma\in\Pi(\mu_t,\nu)}\ \int \|x-y\|_2^2\ \mathrm{d}\gamma(x,y) = \inf_{\Tilde{\gamma}\in \Gamma(\mu_0,\mu_1,\nu)}\ \int \|(1-t)x_0+tx_1 -y\|_2^2\ \mathrm{d}\Tilde{\gamma}(x_0,x_1,y),
    \end{equation}
    where $\Gamma(\mu_0,\mu_1,\nu) = \{\Tilde{\gamma}\in\Pi(\mu_0,\mu_1,\nu),\ \pi^{1,2}_\#\Tilde{\gamma}\in\Pi_o(\mu_0,\mu_1)\}$ and $\pi^{1,2}:(x_0,x_1,y)\mapsto (x_0,x_1)$ is the projection onto the first two coordinates.
\end{lemma}

\begin{proof}
    On one hand, let $\Tilde{\gamma}\in\Gamma(\mu_0,\mu_1,\nu)$. Then, $\gamma_t := \big((1-t)\pi^1 + t\pi^2, \pi^3\big)_\#\Tilde{\gamma}\in\Pi(\mu_t,\nu)$, and we have 
    \begin{equation}
        \begin{aligned}
            \W_2^2(\mu_t,\nu) &= \inf_{\gamma\in\Pi(\mu_t,\nu)}\ \int \|x-y\|_2^2\ \mathrm{d}\gamma(x,y) \\
            &\le \int \|x-y\|_2^2 \ \mathrm{d}\big((1-t)\pi^1 + t\pi^2, \pi^3\big)_\#\Tilde{\gamma}(x,y) \\
            &= \int \|(1-t) x_0 + tx_1 - y\|_2^2\ \mathrm{d}\Tilde{\gamma}(x_0,x_1,y),
        \end{aligned}
    \end{equation}
    and therefore, by taking the infimum over $\Tilde{\gamma}\in\Gamma(\mu_0,\mu_1,\nu)$ on the right term,
    \begin{equation}
        \W_2^2(\mu_t, \nu) \le \inf_{\Tilde{\gamma}\in \Gamma(\mu_0,\mu_1,\nu)}\ \int \|(1-t)x_0+tx_1 - y\|_2^2\ \mathrm{d}\Tilde{\gamma}(x_0,x_1,y).
    \end{equation}

    On the other hand, let $\sigma_t \in \Pi(\mu_t,\nu)$. By disintegration of $\sigma_t$ with respect to its first marginal, there exists a probability kernel $K_t$ such that $\sigma_t = \mu_t \otimes K_t$, \emph{i.e.}, such that for any test function $h$, 
    \begin{equation}
        \int h(x,y)\ \mathrm{d}\sigma_t(x,y) =  \iint h(x,y)\ K_t(x,\mathrm{d}y)\mathrm{d}\mu_t(x).
    \end{equation}
    Let $\gamma^* \in \Pi_o(\mu_0,\mu_1)$ an optimal plan, and define $\pi$ as the measure verifying for any test function $h$
    \begin{equation}
        \int h(x_0,x_1,y)\ \mathrm{d}\pi(x_0,x_1,y) = \iint h(x_0,x_1,y)\ K_t\big((1-t)x_0+tx_1, \mathrm{d}y\big)\ \mathrm{d}\gamma^*(x_0, x_1).
    \end{equation}
    We now verify that $\pi\in\Gamma(\mu_0,\mu_1,\nu)$. On one hand,
    \begin{equation}
        \begin{aligned}
            \int h(y)\mathrm{d}\pi(x_0,x_1,y) &= \iint h(y)\ K_t\big((1-t)x_0 + tx_1, \mathrm{d}y\big)\ \mathrm{d}\gamma^*(x_0,x_1) \\
            &= \iint h(y)\ K_t(x_t, \mathrm{d}y)\ \mathrm{d}\mu_t(x_t) \quad \text{since $\mu_t = \big((1-t)\pi^1 + t\pi^2\big)_\#\gamma^*$} \\
            &= \int h(y)\ \mathrm{d}\sigma_t(x,y) \quad \text{by definition of the disintegration} \\
            &= \int h(y)\ \mathrm{d}\nu(y),
        \end{aligned}
    \end{equation}
    and thus $\pi^3_\#\pi=\nu$. Moreover,
    \begin{equation}
        \begin{aligned}
            \int h(x_0,x_1)\ \mathrm{d}\pi(x_0,x_1,y) &= \iint h(x_0,x_1)\ K_t\big((1-t)x_0+tx_1, \mathrm{d}y\big)\ \mathrm{d}\gamma^*(x_0,x_1) \\
            &= \int h(x_0,x_1)\ \mathrm{d}\gamma^*(x_0,x_1),
        \end{aligned}
    \end{equation}
    and thus $\pi^{1,2}_\#\pi=\gamma^*\in\Pi_o(\mu_0,\mu_1)$. Therefore, we can write
    \begin{equation}
        \begin{aligned}
            &\inf_{\Tilde{\gamma}\in\Gamma(\mu_0, \mu_1, \nu)}\ \int \|(1-t)x_0 +tx_1 - y\|_2^2\ \mathrm{d}\Tilde{\gamma}(x_0,x_1,y) \\ &\quad \le \int \|(1-t)x_0 + tx_1 - y\|_2^2\ \mathrm{d}\pi(x_0,x_1,y) \\
            &\quad = \iint \|(1-t)x_0 + tx_1 - y\|_2^2\ K_t\big((1-t)x_0+tx_1, \mathrm{d}y\big)\ \mathrm{d}\gamma^*(x_0,x_1) \\
            &\quad = \iint \|x_t-y\|_2^2\ K_t(x_t,\mathrm{d}y)\mathrm{d}\mu_t(x_t) \\
            &\quad = \int \|x-y\|_2^2\ \mathrm{d}\sigma_t(x,y) \quad \text{by definition of the disintegration.}
        \end{aligned}
    \end{equation}
    Hence, taking the infimum on the right hand side, we deduce
    \begin{equation}
        \inf_{\Tilde{\gamma}\in\Gamma(\mu_0, \mu_1, \nu)}\ \int \|(1-t)x_0 +tx_1 - y\|_2^2\ \mathrm{d}\Tilde{\gamma}(x_0,x_1,y) \le \W_2^2(\mu_t, \nu),
    \end{equation}
    and we can conclude.
\end{proof}

Let us now show that the infimum over $\Gamma(\mu_0,\mu_1,\nu)$ is attained and thus is a minimum. For this purpose, we need prove the following technical results on $\Gamma(\mu_0,\mu_,\nu)$.

\begin{lemma} \label{lemma:gamma_tight}
    Let $\mu_0,\mu_1,\nu\in\mathcal{P}_2(\mathbb{R}^d)$. Any sequence $(\Tilde{\gamma}_n)_n$ in the space $\Gamma(\mu_0,\mu_1,\nu)=\{\Tilde{\gamma}\in\Pi(\mu_0,\mu_1,\nu),\ \pi^{1,2}_\#\Tilde{\gamma}\in\Pi_o(\mu_0,\mu_1)\}$ is tight.
\end{lemma}

\begin{proof}
    Following the proof of \citep[Theorem 1.7]{santambrogio2015optimal}, the singleton $\{\mu_0\}$, $\{\mu_1\}$ and $\{\nu\}$ are tight and for all $\varepsilon>0$, there exist compacts $K_0,K_1,K_\nu$ such that $\mu_0(\mathbb{R}^d\setminus K_0) < \frac{\varepsilon}{3}$, $\mu_1(\mathbb{R}^d\setminus K_1) < \frac{\varepsilon}{3}$ and $\nu(\mathbb{R}^d\setminus K_\nu) < \frac{\varepsilon}{3}$. Therefore, 
    \begin{multline}
            \Tilde{\gamma}^*_{n}\big((\mathbb{R}^d\times\mathbb{R}^d\times \mathbb{R}^d)\setminus (K_0\times K_1 \times K_\nu)\big) \\ \le \Tilde{\gamma}_{n}^*\big((\mathbb{R}^d\setminus K_0)\times \mathbb{R}^d \times \mathbb{R}^d\big) + \Tilde{\gamma}_{n}^*\big(\mathbb{R}^d \times (\mathbb{R}^d\setminus K_1)\times \mathbb{R}^d\big) + \Tilde{\gamma}_{n}^*\big(\mathbb{R}^d\times\mathbb{R}^d \times (\mathbb{R}^d\setminus K_\nu)\big) < \varepsilon.
    \end{multline}
\end{proof}

\begin{lemma}
    Let $\mu_0,\mu_1,\nu\in\Gamma(\mu_0,\mu_1,\nu)$. Then,
    \begin{equation}
        \inf_{\Tilde{\gamma}\in\Gamma(\mu_0,\mu_1,\nu)}\ \int \|(1-t)x_0+tx_1-y\|_2^2\ \mathrm{d}\Tilde{\gamma}(x_0,x_1,y) = \min_{\Tilde{\gamma}\in\Gamma(\mu_0,\mu_1,\nu)}\ \int \|(1-t)x_0+tx_1-y\|_2^2\ \mathrm{d}\Tilde{\gamma}(x_0,x_1,y).
    \end{equation}
\end{lemma}

\begin{proof}
    First, since any sequence in $\Gamma(\mu_0,\mu_1,\nu)$ is tight from \Cref{lemma:gamma_tight}, we have that $\Gamma(\mu_0,\mu_1,\nu)$ is a compact subset of $\Pi(\mu_0,\mu_1,\nu)$ (see \emph{e.g.} \citep[Proof of Theorem 1.4]{santambrogio2015optimal}). Then, by \citep[Lemma 1.6]{santambrogio2015optimal}, $J:\Tilde{\gamma}\mapsto \int \|(1-t)x_0+tx_1-y\|_2^2\ \mathrm{d}\Tilde{\gamma}(x_0,x_1,y)$ is lower semi-continuous for the weak convergence of measures and by Weierstrass theorem (see \emph{e.g.} \citep[Box 1.1]{santambrogio2015optimal}), the infimum is attained.
\end{proof}

\begin{proof}[Proof of \Cref{prop:busemann_general}]
    Let $t\ge 0$ and $\nu\in\mathcal{P}_2(\mathbb{R}^d)$. Using \Cref{lemma:ot_geod}, we have that
    \begin{equation}
        \begin{aligned}
            \W_2^2(\mu_t,\nu) &= \inf_{\Tilde{\gamma}\in\Gamma(\mu_0, \mu_1, \nu)}\ \int \|(1-t)x_0 +tx_1 - y\|_2^2\ \mathrm{d}\Tilde{\gamma}(x_0,x_1,y) \\
            &= \inf_{\Tilde{\gamma}\in\Gamma(\mu_0,\mu_1,\nu)}\ \int \|x_0-y\|_2^2\ \mathrm{d}\Tilde{\gamma}(x_0,x_1,y) + t^2 \W_2^2(\mu_0,\mu_1) - 2t \int \langle x_1-x_0, y-x_0\rangle\ \mathrm{d}\Tilde{\gamma}(x_0,x_1,y) \\
        &= t^2 \W_2^2(\mu_0,\mu_1) \left(1 + \inf_{\Tilde{\gamma}\in\Gamma(\mu_0,\mu_1,\nu)}\ \left[ \frac{1}{t^2 \W_2^2(\mu_0,\mu_1)} \int \|x_0-y\|_2^2\ \mathrm{d}\Tilde{\gamma}(x_0,x_1,y) \right.\right. \\ &\left.\left. \quad- \frac{2}{t \W_2^2(\mu_0,\mu_1)}\int \langle x_1-x_0, y-x_0\rangle\ \mathrm{d}\Tilde{\gamma}(x_0,x_1,y)\right]\right).
        \end{aligned}
    \end{equation}
    Thus, we have:
    \begin{equation}
        \begin{aligned}
            &\W_2(\mu_t,\nu) - t\W_2(\mu_0,\mu_1) \\
            &= t\W_2(\mu_0,\mu_1) \\ &\quad \sqrt{1 + \frac{2}{t\W_2^2(\mu_0,\mu_1)} \inf_{\Tilde{\gamma}\in\Gamma(\mu_0,\mu_1,\nu)} \left[ \frac{1}{2t}\int \|x_0-y\|_2^2 \ \mathrm{d}\Tilde{\gamma}(x_0,x_1,y) - \int \langle x_1-x_0, y-x_0\rangle\ \mathrm{d}\Tilde{\gamma}(x_0,x_1,y)\right]} \\ &\quad - t\W_2(\mu_0,\mu_1) \\
            &\underset{t\to\infty}{=} t\W_2(\mu_0,\mu_1) \\ &\quad \left(1 + \frac{1}{t\W_2^2(\mu_0,\mu_1)} \inf_{\Tilde{\gamma}\in\Gamma(\mu_0,\mu_1,\nu)} \left[\frac{1}{2t}\int \|x_0-y\|_2^2 \ \mathrm{d}\Tilde{\gamma}(x_0,x_1,y) - \int \langle x_1-x_0, y-x_0\rangle\ \mathrm{d}\Tilde{\gamma}(x_0,x_1,y) \right] + o(t^{-1})\right) \\ &\quad - t\W_2(\mu_0,\mu_1) \\
            &\underset{t\to\infty}{=}  \frac{1}{\W_2(\mu_0,\mu_1)} \inf_{\Tilde{\gamma}\in\Gamma(\mu_0,\mu_1,\nu)} \left(\frac{1}{2t}\int \|x_0-y\|_2^2 \ \mathrm{d}\Tilde{\gamma}(x_0,x_1,y) - \int \langle x_1-x_0, y-x_0\rangle\ \mathrm{d}\Tilde{\gamma}(x_0,x_1,y)\right).
        \end{aligned}
    \end{equation}

    To conclude, we need to show that we can pass to the limit. First, let $\Tilde{\gamma}_t^*$ be defined as
    \begin{equation}
        \Tilde{\gamma}_t^*\in\argmin_{\Tilde{\gamma}\in\Gamma(\mu_0,\mu_1,\nu)} \left(\frac{1}{2t}\int \|x_0-y\|_2^2 \ \mathrm{d}\Tilde{\gamma}(x_0,x_1,y) - \int \langle x_1-x_0, y-x_0\rangle\ \mathrm{d}\Tilde{\gamma}(x_0,x_1,y)\right),
    \end{equation}
    and let
    \begin{equation}\label{eq:argmin_proof}
        \Tilde{\gamma}^*\in\argmin_{\Tilde{\gamma}\in\Gamma(\mu_0,\mu_1,\nu)}  - \int \langle x_1-x_0, y-x_0\rangle\ \mathrm{d}\Tilde{\gamma}(x_0,x_1,y).
    \end{equation}
    By definition of $\Tilde{\gamma}_t^*$ and $\Tilde{\gamma}^*$, we have the following inequality:
    \begin{multline} \label{eq:ineq_t}
        \frac{1}{2t}\int \|x_0-y\|_2^2 \ \mathrm{d}\Tilde{\gamma}_t^*(x_0,x_1,y) - \int \langle x_1-x_0, y-x_0\rangle\ \mathrm{d}\Tilde{\gamma}_t^*(x_0,x_1,y) \\ \le \frac{1}{2t}\int \|x_0-y\|_2^2 \ \mathrm{d}\Tilde{\gamma}^*(x_0,x_1,y) - \int \langle x_1-x_0, y-x_0\rangle\ \mathrm{d}\Tilde{\gamma}^*(x_0,x_1,y).
    \end{multline}

    Let $(t_n)_n$ be a sequence such that $t_n\to \infty$. Any sequence in $\Gamma(\mu_0,\mu_1,\nu)$ is tight by \Cref{lemma:gamma_tight}. 
    Hence, by Prokhorov's theorem, we can extract a subsequence $\Tilde{\gamma}_{t_{\varphi(n)}}^*$ converging in law towards $\gamma_\infty\in\Gamma(\mu_0,\mu_1,\nu)$. Thus, passing to the limit in \eqref{eq:ineq_t}, we have
    \begin{equation}\label{ineq:inner_product_infty}
        -\int \langle x_1-x_0, y-x_0\rangle\ \mathrm{d}\gamma_\infty(x_0,x_1,y) \le -\int \langle x_1-x_0, y-x_0\rangle\ \mathrm{d}\Tilde{\gamma}^*(x_0,x_1,y).
    \end{equation}
    But by definition, $\Tilde{\gamma}^*$ is optimal in \eqref{eq:argmin_proof}, therefore \eqref{ineq:inner_product_infty} is an equality. We can conclude that
    \begin{equation}
        B^\mu(\nu) = \inf_{\Tilde{\gamma}\in\Gamma(\mu_0,\mu_1,\nu)}\ -\frac{1}{\W_2(\mu_0,\mu_1)} \int \langle x_1-x_0,y-x_0\rangle\ \mathrm{d}\Tilde{\gamma}(x_0,x_1,y).
    \end{equation}
\end{proof}

\subsubsection{Proof of \Cref{corollary:busemann_ot_map}} \label{proof:corollary_busemann_ot_map}

\begin{proof}[Proof of \Cref{corollary:busemann_ot_map}]
    Let $\mu_0$ absolutely continuous \emph{w.r.t.} the Lebesgue measure, and $T$ the gradient of a 1-convex function such that $\mu_1=T_\#\mu_0$. Let us show that in this case $\Gamma(\mu_0,\mu_1,\nu)=\{(\pi^1, T\circ \pi^1, \pi^2)_\#\gamma,\ \gamma\in\Pi(\mu_0,\nu)\}$, where we recall that $\Gamma(\mu_0,\mu_1,\nu)=\{\Tilde{\gamma}\in\Pi(\mu_0,\mu_1,\nu),\ \pi^{1,2}_\#\Tilde{\gamma}\in\Pi_o(\mu_0,\mu_1)\}$.

    On one hand, let $\gamma\in\Pi(\mu_0,\nu)$ and $\Tilde{\gamma} = (\pi^1, T\circ\pi^1, \pi^2)_\#\gamma$. Then, we verify easily that the marginals are satisfied, \emph{i.e.}, $\pi^1_\#\Tilde{\gamma}=\mu_0$, $\pi^2_\#\Tilde{\gamma}=\mu_1$ and $\pi^3_\#\Tilde{\gamma}=\nu$. Moreover,
    \begin{equation}
        \int h(x_0,x_1)\ \mathrm{d}\Tilde{\gamma}(x_0,x_1,y) = \int h\big(x,T(x)\big)\ \mathrm{d}\gamma(x_0,y) = \int h\big(x,T(x)\big)\ \mathrm{d}\mu_0(x),
    \end{equation}
    and hence $\pi^{1,2}_\#\Tilde{\gamma}=(\id,T)_\#\mu_0\in\Pi_o(\mu_0,\mu_1)$. Thus, $\{(\pi^1, T\circ \pi^1, \pi^2)_\#\gamma,\ \gamma\in\Pi(\mu_0,\nu)\} \subset \Gamma(\mu_0,\mu_1,\nu)$.

    On the other hand, let $\Tilde{\gamma}\in\Gamma(\mu_0,\mu_1,\nu)$. Thus, we know that $\pi^{1,2}_\#\Tilde{\gamma}=(\id, T)_\#\mu_0$. Additionally, by the disintegration theorem, there exists a probability kernel $K$ such that $\Tilde{\gamma}=\pi^{1,2}_\#\Tilde{\gamma} \otimes K = (\id, T)_\#\mu_0 \otimes K$, \emph{i.e.}
    \begin{equation}
        \begin{aligned}
            \int h(x_0,x_1,y)\ \mathrm{d}\Tilde{\gamma}(x_0,x_1,y) &= \iint h(x_0,x_1,y)\ K\big((x_0, x_1), \mathrm{d}y\big)\ \mathrm{d}(\pi^{1,2}_\#\Tilde{\gamma})(x_0,x_1) \\
            &= \iint h(x_0, T(x_0), y)\ K\big((x_0, T(x_0)), \mathrm{d}y\big)\ \mathrm{d}\mu_0(x_0).
        \end{aligned}
    \end{equation}
    Denoting $\Tilde{K}(x_0,\mathrm{d}y)=K\big((x_0, T(x_0)), \mathrm{d}y\big)$ and defining $\gamma=\mu_0\otimes \Tilde{K}$, we obtain
    \begin{equation}
        \begin{aligned}
            \int h(x_0,x_1,y) \ \mathrm{d}\Tilde{\gamma}(x_0,x_1,y) &= \iint h(x_0, T(x_0), y)\ \Tilde{K}(x_0,\mathrm{d}y)\mathrm{d}\mu_0(x_0) \\
            &= \int h(x_0, T(x_0), y)\ \mathrm{d}\gamma(x_0, y) \\
            &= \int h(x_0,x_1,y)\ \mathrm{d}(\pi^1,T\circ\pi^1,\pi^2)_\#\gamma(x_0,x_1,y).
        \end{aligned}
    \end{equation}
    Thus, $\Tilde{\gamma} = (\pi^1, T\circ \pi^1, \pi^2)_\#\gamma$. Moreover, we can verify that $\gamma\in \Pi(\mu_0,\nu)$ as $\pi^1_\#\gamma=\mu_0$ by definition of $\gamma$ and 
    \begin{equation}
        \begin{aligned}
            \int h(y)\ \mathrm{d}\gamma(x,y) &= \iint h(y)\ \Tilde{K}(x, \mathrm{d}y)\mathrm{d}\mu_0(x) \\
            &= \iint h(y) \ K\big((x, T(x)), \mathrm{d}y\big)\mathrm{d}\mu_0(x) \\
            &= \int h(y)\ \mathrm{d}\Tilde{\gamma}(x_0,x_1,y) \\
            &= \int h(y)\ \mathrm{d}\nu(y).
        \end{aligned}
    \end{equation}
    Therefore, we can conclude that the two sets are equals, and rewrite the Busemann function from \Cref{prop:busemann_general} in this case as 
    \begin{equation}
        B^\mu(\nu) = \inf_{\gamma\in\Pi(\mu_0,\nu)}\ -\frac{1}{\W_2(\mu_0,\mu_1)} \int \langle T(x_0)-x_0, y-x_0\rangle\ \mathrm{d}\gamma(x_0,y).
    \end{equation}
\end{proof}

\subsubsection{Proof of \Cref{corollary:busemann_dirac}} \label{proof:corollary_busemann_dirac}

\begin{proof}[Proof of \Cref{corollary:busemann_dirac}]
    Let $\mu_0=\delta_{x_0}$ with $x_0\in\mathbb{R}^d$. Recall that $\Gamma(\mu_0, \mu_1, \nu) = \{\Tilde{\gamma}\in\Pi(\mu_0,\mu_1,\nu),\ \pi^{1,2}_\#\Tilde{\gamma}\in\Pi_o(\mu_0,\mu_1)\} = \{\Tilde{\gamma}\in\Pi(\mu_0,\mu_1,\nu),\ \pi^{1,2}_\#\Tilde{\gamma}=\mu_0\otimes \mu_1\}$ since the optimal coupling between $\mu_0=\delta_{x_0}$ and $\mu_1$ is $\mu_0\otimes \mu_1$. Let us show that $\Gamma(\mu_0,\mu_1,\nu) = \{\mu_0 \otimes \gamma,\ \gamma\in\Pi(\mu_1,\nu)\}$ in this case.

    On one hand, let $\gamma\in\Pi(\mu_1,\nu)$ and define $\Tilde{\gamma}=\mu_0\otimes \gamma$. Then, trivially, we have that $\Tilde{\gamma}\in\Pi(\mu_0,\mu_1,\nu)$. Moreover, let us verify that $\pi^{1,2}_\#\Tilde{\gamma}=\mu_0\otimes \mu_1$. For any continuous bounded function $h$,
    \begin{equation}
        \begin{aligned}
            \int h(x_0,x_1)\ \mathrm{d}(\pi^{1,2}_\#\Tilde{\gamma})(x_0,x_1) &= \int h(x_0,x_1)\ \mathrm{d}\Tilde{\gamma}(x_0,x_1,y) \\
            &= \iint h(x_0,x_1)\ \mathrm{d}\gamma(x_1, y)\mathrm{d}\mu_0(x_0) \\
            &=  \iint h(x_0,x_1)\ \mathrm{d}\mu_1(x_1)\mathrm{d}\mu_0(x_0).
        \end{aligned}
    \end{equation}
    Thus, $\{\mu_0\otimes \gamma,\ \gamma\in\Pi(\mu_1,\nu)\}\subset \Gamma(\mu_0,\mu_1,\nu)$.

    On the other hand, let $\Tilde{\gamma}\in\Gamma(\mu_0,\mu_1,\nu)$. Since $\pi^{1,2}_\#\Tilde{\gamma}=\mu_0\otimes \mu_1$, we can disintegrate $\Tilde{\gamma}$ as $\Tilde{\gamma}=(\mu_0\otimes \mu_1)\otimes K$, \emph{i.e.} for any $h$,
    \begin{equation} \label{eq:1st_dis}
        \begin{aligned}
            \int h(x_0,x_1,y)\ \mathrm{d}\Tilde{\gamma}(x_0,x_1,y) &= \iint h(x_0,x_1,y)\  K\big((x_0,x_1), \mathrm{d}y\big)\ \mathrm{d}\mu_1(x_1)\mathrm{d}\mu_0(x_0).
        \end{aligned}
    \end{equation}
    Let us define the distribution $\gamma_{x_0}$ satisfying for $\mu_0$-a.e. $x_0$,
    \begin{equation}
        \int h(x_0,x_1,y)\ \mathrm{d}\gamma_{x_0}(x_1,y) = \iint h(x_0,x_1,y)\ K\big((x_0,x_1), \mathrm{d}y\big)\mathrm{d}\mu_1(x_1).
    \end{equation}
    First, we can verify that for $\mu_0$-a.e. $x_0$, $\pi^1_\#\gamma_{x_0}=\mu_1$ as
    \begin{equation}
        \begin{aligned}
            \int h(x_1) \mathrm{d}\gamma_{x_0}(x_1,y) = \iint h(x_1) K\big((x_0,x_1), \mathrm{d}y\big)\mathrm{d}\mu_1(x_1) = \int h(x_1)\mathrm{d}\mu_1(x_1).
        \end{aligned}
    \end{equation}
    Moreover, we can also disintegrate $\Tilde{\gamma}$ \emph{w.r.t.} $\mu_0$ as $\Tilde{\gamma}=\mu_0\otimes \Tilde{K}$. By uniqueness of the disintegration, we have for $\mu_0$-a.e. $x_0$,
    \begin{equation} \label{eq:equality_dis}
        \iint h(x_1, y) K\big((x_0,x_1), \mathrm{d}y\big)\mathrm{d}\mu_1(x_1) = \int h(x_1, y) \Tilde{K}\big(x_0, (\mathrm{d}x_1,\mathrm{d}y)\big),
    \end{equation}
    \emph{i.e.} $\gamma_{x_0} = \Tilde{K}(x_0,\cdot)$.
    Integrating \emph{w.r.t} $\mu_0$, we get that the left hand side is equal to $\int h\mathrm{d}\Tilde{\gamma}$ by \eqref{eq:1st_dis}. But we also get
    \begin{equation}
        \begin{aligned}
            \iiint h(x_1,y) K\big((x_0,x_1), \mathrm{d}y\big)\mathrm{d}\mu_1(x_1)\mathrm{d}\mu_0(x_0) &= \iint h(x_1,y) \left(\int K\big((x_0,x_1), \mathrm{d}y\big)\mathrm{d}\mu_0(x_0) \right)\ \mathrm{d}\mu_1(x_1) \\
            &= \iint h(x_1,y)\ \Bar{K}(x_1,\mathrm{d}y)\mathrm{d}\mu_1(x_1),
        \end{aligned}
    \end{equation}
    where $\Bar{K}(x_1, \mathrm{d}y):= \int K\big((x_0,x_1), \mathrm{d}y\big)\mathrm{d}\mu_0(x_0)$. Thus, we deduce that $\pi^{2,3}_\#\Tilde{\gamma} = \mu_1 \otimes \Bar{K}$. Integrating the right hand side of \eqref{eq:equality_dis}, we get for any bounded measurable function $h$
    \begin{equation}
        \int h(x_1,y)\ \mathrm{d}\Tilde{\gamma}(x_0,x_1,y) = \int h(x_1,y) \int \mathrm{d}\gamma_{x_0}(x_1,y)\mathrm{d}\mu_0(x_0) = \iint h(x_1,y)\ \Bar{K}(x_1,\mathrm{d}y)\mathrm{d}\mu_1(x_1).
    \end{equation}
    It implies that $\gamma_{x_0} = \mu_1\otimes \Bar{K}$ for $\mu_0$-a.e. $x_0$ and thus that $K$ does not depend on $x_0$. Finally, we can conclude that $\Tilde{\gamma}=\mu_0\otimes \gamma$ with $\gamma = \mu_1 \otimes \Bar{K}\in \Pi(\mu_1,\nu)$.
\end{proof}

\subsubsection{Proof of \Cref{prop:busemann_closed_1d}} \label{proof:prop_busemann_closed_1d}

\begin{proof}[Proof of \Cref{prop:busemann_closed_1d}]

    Let $(\mu_t)_{t\ge 0}$ be a geodesic ray. Recall that $\W_2^2(\mu_0,\mu_1) = \int_0^1\big|F_0^{-1}(u) - F_1^{-1}(u)\big|^2\ \mathrm{d}u = \|F_0^{-1} - F_1^{-1}\|_{L^2([0,1])}^2$. Moreover, the quantile functions of any measure on the geodesic ray is of the form,
    \begin{equation}
        \forall t\ge 0,\ F_t^{-1} = (1-t)F_0^{-1} + tF_1^{-1}.
    \end{equation}
    
    Thus, we have, for any $\nu\in\cP_2(\R)$, $t\ge 0$,
    \begin{equation}
        \begin{aligned}
            \W_2(\nu,\mu_t) - t \W_2(\mu_1,\mu_0) &= \|F_\nu^{-1}-F_t^{-1}\|_{L^2([0,1])} - t \W_2(\mu_1,\mu_0) \\
            &= \|F_\nu^{-1}-(1-t)F_0^{-1} - tF_1^{-1}\|_{L^2([0,1])} - t \W_2(\mu_1,\mu_0)  \\
            &= \|F_\nu^{-1} - F_0^{-1} - t (F_1^{-1} - F_0^{-1})\|_{L^2([0,1])} - t \W_2(\mu_1,\mu_0)  \\
            &= \sqrt{\|F_{\nu}^{-1}-F_0^{-1}\|_{L^2([0,1])}^2 -2t\langle F_\nu^{-1}-F_0^{-1}, F_1^{-1}-F_0^{-1}\rangle_{L^2([0,1])} + t^2 \W_2^2(\mu_1,\mu_0)} \\&\quad - t \W_2(\mu_1,\mu_0)  \\
            &= t \W_2(\mu_1,\mu_0)  \sqrt{1 - \frac{2}{t \W_2^2(\mu_1,\mu_0)} \langle F_{\nu}^{-1}-F_0^{-1}, F_1^{-1}-F_0^{-1}\rangle_{L^2([0,1])} + o\left(\frac{1}{t}\right)} \\ &\quad - t \W_2(\mu_1,\mu_0)  \\
            &\underset{t\to\infty}{=} t \W_2(\mu_1,\mu_0)  \left(1 - \frac{1}{t \W_2^2(\mu_1,\mu_0) } \langle F_\nu^{-1}-F_0^{-1}, F_1^{-1}-F_0^{-1}\rangle_{L^2([0,1])} + o\left(\frac{1}{t}\right)\right) \\ &\quad - t \W_2(\mu_1,\mu_0)  \\
            &\xrightarrow[t\to\infty]{} -\left\langle F_\nu^{-1}-F_0^{-1}, \frac{F_1^{-1}-F_0^{-1}}{\|F_1^{-1}-F_0^{-1}\|_{L^2([0,1])}}\right\rangle_{L^2([0,1])}.
        \end{aligned}
    \end{equation}

    Thus, we can conclude that 
    \begin{equation}
        B^\mu(\nu) = -\left\langle F_\nu^{-1}-F_0^{-1}, \frac{F_1^{-1}-F_0^{-1}}{\|F_1^{-1}-F_0^{-1}\|_{L^2([0,1])}}\right\rangle_{L^2([0,1])}.
    \end{equation}

    For a unit-speed geodesic ray $(\mu_t)_{t\ge 0}$, we have $\|F_1^{-1}-F_0^{-1}\|_{L^2([0,1])}=1$, and we then recover  \eqref{eq:busemann_1d}.
\end{proof}

\subsubsection{Proof of \Cref{corr:busemann_1d_gaussian}} \label{proof:corr_busemann_1d_gaussian}

\begin{proof}[Proof of \Cref{corr:busemann_1d_gaussian}]
    Recall that for any $\eta=\cN(m,\sigma^2)$, for all $u\in [0,1]$, $F_\eta^{-1}(u) = m + \sigma \phi^{-1}(u)$ where $\phi^{-1}$ is the quantile function of $\cN(0,1)$, therefore satisfying $\int_0^1 \phi^{-1}(u)\ \mathrm{d}u=0$ and $\int_0^1 \phi^{-1}(u)^2\ \mathrm{d}u = 1$.

    Thus, let $\mu_0=\cN(\mo,\so^2)$, $\mu_1=\cN(\m1,\s1^2)$, $\nu=\cN(m,\sigma^2)$ with $m_0,m_1,m\in \R$, $\sigma_0,\sigma_1,\sigma\in \R_+^*$ and $\sigma_1\ge \sigma_0$, $\W_2^2(\mu_0,\mu_1)=1$. Then, by \Cref{prop:busemann_closed_1d}, we have
    \begin{equation}
        \begin{aligned}
            B^\mu(\nu) &= - \langle F_1^{-1} - F_0^{-1}, F_\nu^{-1} - F_0^{-1}\rangle_{L^2([0,1])} \\
            &= -\langle (m_1-m_0) + (\sigma_1-\sigma_0)\phi^{-1}, (m-m_0) + (\sigma-\sigma_0)\phi^{-1}\rangle_{L^2([0,1])} \\
            &= - (m_1-m_0)(m-m_0) - (m_1-m_0)(\sigma-\sigma_0) \int_0^1 \phi^{-1}(u)\mathrm{d}u \\ &\quad- (m-m_0)(\sigma_1-\sigma_0) \int_0^1 \phi^{-1}(u)\mathrm{d}u - (\sigma_1-\sigma_0)(\sigma - \sigma_0)\int_0^1 \phi^{-1}(u)^2\ \mathrm{d}u \\
            &= - (\m1-\mo)(m-\mo) - (\s1-\so)(\sigma-\so) \\
            &= - \left\langle 
            \begin{pmatrix}
                \m1-\mo \\ \s1-\so
            \end{pmatrix}, 
            \begin{pmatrix}
                m-\mo \\ \sigma-\so
            \end{pmatrix} \right\rangle.
        \end{aligned}
    \end{equation}

    More generally, if $\W_2(\mu_0,\mu_1)>0$, we have
    \begin{equation}
        B^\mu(\nu) = - \frac{(\m1-\mo)(m-\mo) - (\s1-\so)(\sigma-\so)}{\sqrt{(\m1-\mo)^2 + (\s1-\so)^2}}.
    \end{equation}
\end{proof}

\subsubsection{Proof of \Cref{prop:closed_form_general_gaussian}} \label{proof:prop_closed_form_general_gaussian}

\begin{proof}[Proof of \Cref{prop:closed_form_general_gaussian}]
    We will use here that for any geodesic ray $\gamma$, $\lim_{t\to\infty}\ \frac{d(x,\gamma(t))+t}{2t} = 1$ (cf \citep[II. 8.24]{bridson2013metric}. Then we know that 
    \begin{equation}
        \lim_{t\to \infty}\ \frac{d(x,\gamma(t))^2-t^2}{2t} = \lim_{t\to\infty}\ \big(d(x,\gamma(t))-t\big),
    \end{equation}
    since
    \begin{equation}
        \frac{d(x,\gamma(t))^2-t^2}{2t} = \frac{(d(x,\gamma(t))-t)(d(x,\gamma(t))+t)}{2t} = \big(d(x,\gamma(t))-t\big) \frac{d(x,\gamma(t))+t}{2t} \xrightarrow[t\to \infty]{} B^\gamma(x).
    \end{equation}

    In our case, we have for any $t\ge 0$, $\mu_t =\mathcal{N}(m_t, \Sigma_t)$ where
    \begin{equation}
        \begin{cases}
            m_t = (1-t)\mo+t \m1 \\
            \Sigma_t = \big((1-t)I_d+t A\big)\So \big((1-t)I_d + t A\big),
        \end{cases}
    \end{equation}
    with $A=\So^{-\frac12}(\So^\frac12\S1\So^\frac12)^\frac12\So^{-\frac12}$ (see \emph{e.g.} \citep[Appendix A.1]{altschuler2021averaging}). Then, using $A\So A = \S1$, we have for any $t\ge 0$,
    \begin{align}
        &\frac{\|m_t-m\|_2^2}{2t} = \frac{t}{2} \|\m1-\mo\|_2^2 + \langle \m1-\mo, \mo-m\rangle + O\left(\frac{1}{t}\right), \\
        &\frac{\mathrm{Tr}(\Sigma_t)}{2t} = \frac{t}{2} \tr\big(\So-2\So A+\S1\big) + \tr\big(\So A-\So\big) + O\left(\frac{1}{t}\right) \\
        &\frac{\tr\big((\Sigma^\frac12 \Sigma_t \Sigma^\frac12)^\frac12\big)}{2t} = \frac12 \tr\left(\left(\Sigma^\frac12 (\So-\So A-A\So+\S1)\Sigma^\frac12 + O\left(\frac{1}{t}\right) \right)^\frac12\right).
    \end{align}
    Additionally, by hypothesis,
    \begin{equation}
        \W_2^2(\mu_0,\mu_1) = \|\m1-\mo\|_2^2 + \tr(\So+\S1-2(\So^\frac12 \S1 \So^{\frac12})^\frac12) = 1,
    \end{equation}
    and since
    \begin{equation}
        \So A = \So^{\frac12}(\So^\frac12 \S1 \So^\frac12)^\frac12 \So^{-\frac12},
    \end{equation}
    we get
    \begin{equation}
        \tr(\So A) = \tr\big((\So^\frac12 \S1 \So^{\frac12})^\frac12\big).
    \end{equation}
    Therefore, we obtain
    \begin{equation}
        \begin{aligned}
            &\frac{\W_2^2(\nu,\mu_t)-t^2}{2t} \\ &= \frac{\|m_t-m\|_2^2 + \tr\left(\Sigma_t+\Sigma-2(\Sigma^\frac12 \Sigma_t \Sigma^\frac12)^\frac12\right)-t^2}{2t} \\ 
            &= \frac{t}{2}\left(\|\m1-\mo\|_2^2 + \tr(\So + \S1 - 2 \So A)\right) + \langle \m1-\mo, \mo-m\rangle + \tr\big(\So A-\So\big) \\ &\quad -  \tr\left(\left(\Sigma^\frac12 (\So-\So A-A\So+\S1)\Sigma^\frac12 + O\left(\frac{1}{t}\right) \right)^\frac12\right) - \frac{t}{2} + O\left(\frac{1}{t}\right) \\
            &= \frac{t}{2} \W_2^2(\mu_0,\mu_1) + \langle \m1-\mo, \mo-m\rangle + \tr\big(\So A-\So\big) \\ &\quad -  \tr\left(\left(\Sigma^\frac12 (\So-\So A-A\So+\S1)\Sigma^\frac12 + O\left(\frac{1}{t}\right) \right)^\frac12\right) - \frac{t}{2} + O\left(\frac{1}{t}\right) \\
            &= \langle \m1-\mo, \mo-m\rangle + \tr\big(\So A-\So\big) \\ &\quad -  \tr\left(\left(\Sigma^\frac12 (\So-\So A-A\So+\S1)\Sigma^\frac12 + O\left(\frac{1}{t}\right) \right)^\frac12\right) + O\left(\frac{1}{t}\right) \\
            &\xrightarrow[t\to\infty]{}  \langle \m1-\mo, \mo-m\rangle + \tr\big(\So(A-I_d) \big) - \tr\big((\Sigma^\frac12(\So-\So A - A\So + \S1)\Sigma^\frac12)^\frac12\big).
        \end{aligned}
    \end{equation}
\end{proof}

\subsection{Proofs of Appendix}\label{appendix:proofs_appendix}

\subsubsection{Proof of \Cref{prop:bound_bdgmsw}}

\begin{proof}[Proof of \Cref{prop:bound_bdgmsw}]
    Let $\bP, \bQ\in\cP_2\big(\bw(\R^d)\big)$. Recall that for any $\theta\in S^{d-1}$, $\bP^\theta=\varphi^\theta_\#\bP$, and for any $\mu=\cN(m,\Sigma)\sim \bP$, $\varphi^\theta(\mu)=P^\theta_\#\mu=\cN(\langle m,\theta\rangle, \theta^T\Sigma\theta) = \cN(m_\theta,\sigma_\theta^2)$ where we note $m_{\mu_\theta} = \langle m,\theta\rangle$ and $\sigma_{\mu_\theta}^2 = \theta^T\Sigma\theta$.
    
    Then, for any $\theta\in S^{d-1}$ and $\eta_1=\cN(m_1,\sigma_1^2)$, let $\Gamma\in\Pi_o(\bP,\bQ)$, $\Gamma_\theta=(\varphi^\theta, \varphi^\theta)_\#\Gamma \in \Pi(\bP^\theta,\bQ^\theta)$ and $\gamma_\theta^\eta = (B^\eta,B^\eta)_\#\Gamma_\theta\in \Pi(B^\eta_\#\bP^\theta, B^\eta_\#\bQ^\theta)$. Recall that $\W_2^2(\eta_0,\eta_1)=(m_1-m_0)^2 + (\sigma_1-\sigma_0)^2 = 1$.
    
    Then, we have
    \begin{equation}
        \begin{aligned}
            \W_2^2(B^{\eta}_\#\bP^\theta, B^\eta_\#\bQ^\theta) &= \inf_{\gamma\in\Pi(B^\eta_\#\bP^\theta, B^\eta_\#\bQ^\theta)}\ \int |x-y|^2\ \mathrm{d}\gamma(x,y) \\
            &\le \int |x-y|\ \mathrm{d}\gamma_\theta^\eta(x,y) \\
            &= \int |B^\eta(\mu_\theta) - B^\eta(\nu_\theta)|^2\ \mathrm{d}\Gamma_\theta(\mu_\theta,\nu_\theta) \\
            &= \int \left| \left\langle\begin{pmatrix} m_1-m_0 \\ \sigma_1-\sigma_0 \end{pmatrix}, \begin{pmatrix} m_{\mu_\theta} - m_{\nu_\theta} \\ \sigma_{\mu_\theta} - \sigma_{\nu_\theta} \end{pmatrix} \right\rangle \right|^2 \ \mathrm{d}\Gamma_\theta(\mu_\theta,\nu_\theta)  \\
            &\le \W_2^2(\eta_0,\eta_1) \cdot \int \W_2^2(\mu_\theta,\nu_\theta)\ \mathrm{d}\Gamma(\mu_\theta,\nu_\theta) \quad \text{by Cauchy-Schwartz} \\
            &= \int \W_2^2\big(\varphi^\theta(\mu), \varphi^\theta(\nu)\big)\ \mathrm{d}\Gamma(\mu,\nu) \\
            &= \int \W_2^2(P^\theta_\#\mu, P^\theta_\#\nu)\ \mathrm{d}\Gamma(\mu,\nu) \\ 
            &\le \int \W_2^2(\mu,\nu)\ \mathrm{d}\Gamma(\mu,\nu) \quad \text{since $P^\theta$ 1-Lipschitz} \\
            &= \W_{\bw}^2(\bP,\bQ) \quad \text{since $\Gamma\in\Pi_o(\bP,\bQ)$}.
        \end{aligned}
    \end{equation}
    By integrating \emph{w.r.t.} $\lambda$, we can conclude that
    \begin{equation}
        \smwbdg^2(\bP,\bQ) \le \W_{\bw}^2(\bP,\bQ).
    \end{equation}
\end{proof}

\subsubsection{Proof of \Cref{prop:bdgmsw_distance}} \label{proof:prop_bdgmsw_distance}

\begin{proof}[Proof of \Cref{prop:bdgmsw_distance}]
    The symmetric, positivity, and triangular inequality are clear by classical arguments.

    For positive definiteness, let $\bP,\bQ \in \bigcup_{K\ge 0} \mathrm{GMM}_d(K)$ such that $\smwbdg(\bP,\bQ)=0$. Then, there exists $K\ge 0$ such that $\bP=\sum_{k=1}^K \alpha_k \delta_{\mu_k}$ and $\bQ=\sum_{k=1}^K \beta_k \delta_{\nu_k}$. 
    
    First, note that we can rewrite $\smwbdg$ as 
    \begin{equation}
        \smwbdg^2(\bP,\bQ)=\int_{S^{d-1}} \sw_2^2(\Xi_\#\bP_\theta,\Xi_\#\bQ_\theta)\big)\ \mathrm{d}\lambda(\theta),
    \end{equation}
    with $\Xi(\cN(m,\sigma^2))=(m,\sigma)$, and $\bP^\theta=\varphi^\theta_\#\bP$.
    
    Thus, $\smwbdg(\bP,\bQ)=0$ implies that for $\lambda$-almost all $\theta\in S^{d-1}$, $\sw_2^2(\bP^\theta, \bQ^\theta)=0$. However, $\sw$ is a Pullback-Euclidean Sliced-Wasserstein distance. Thus, by \citep[Proposition 26]{bonet2025sliced}, it is a distance and thus $\bP_\theta=\bQ_\theta$. Moreover, this also implies that $\W_2^2(\sum_{k=1}^K \alpha_k \delta_{P^\theta_\#\mu_k}, \sum_{k=1}^K \beta_k \delta_{P^\theta_\#\nu_k})=0$, and thus by integrating \emph{w.r.t} $\theta\in S^{d-1}$, the Sliced-Wasserstein distance between the mixtures seen in $\cP_2(\R^d)$ is equal to 0. Thus, as SW is a distance, we can conclude that $\bP=\bQ$.
\end{proof}

\subsubsection{Proof of \Cref{prop:extension_1dgaussian_ray}} \label{proof:prop_extension_1dgaussian_ray}

\begin{proof}[Proof of \Cref{prop:extension_1dgaussian_ray}]
    Let $\mu_0=\mathcal{N}(m_0,\sigma_0^2)$ and $\mu_1=\mathcal{N}(m_1, \sigma_1^2)$ such that %
    $\sigma_1> \sigma_0$. Extending the geodesic between $\mu_0$ and $\mu_1$ on $]1-\alpha, 0]$ for $\alpha > 1$ is equivalent to extending the geodesic between $\mu_1$ and $\mu_0$ on $[0, \alpha[$. Thus, we first find a condition to extend the geodesic between $\mu_1$ and $\mu_0$.

    The Monge map $\Tilde{T}$ between $\mu_1$ and $\mu_0$ is defined for all $x\in\mathbb{R}$ as $\Tilde{T}(x) = \frac{\sigma_0}{\sigma_1}(x-m_1) + m_0 = h'(x)$ with $h:x\mapsto \frac{\sigma_0}{2\sigma_1}(x-m_1)^2 + m_0 x$.  Then, by \citep[Section 4]{natale2022geodesic}, we know that we can extend the geodesic linking $\mu_1$ to $\mu_0$ on $[0,\alpha[$ for $\alpha\ge 1$ if and only if $h$ is $\frac{\alpha-1}{\alpha}$-convex, \emph{i.e.} if and only if
    \begin{equation} \label{eq:inequality_alpha}
        h''(x) - \frac{\alpha -1}{\alpha} \ge 0 \iff \frac{\sigma_0}{\sigma_1} \ge \frac{\alpha-1}{\alpha} \iff \frac{\sigma_1}{\sigma_0} \le \frac{\alpha}{\alpha-1}.
    \end{equation}

    Therefore, we deduce that we can extend the geodesic ray starting from $\mu_0$ and passing through $\mu_1$ at $t=1$ on $]-(\alpha-1), +\infty[$ if and only if $\frac{\alpha}{\alpha - 1} \ge \frac{\sigma_1}{\sigma_0} \ge 1$ (the last inequality results from the geodesic ray condition $\sigma_1\ge \sigma_0$).
    Thus, we find that the largest possible value $\alpha> 1$ satisfying inequality \eqref{eq:inequality_alpha} is $\frac{\sigma_1}{\sigma_1-\sigma_0}$ as
    \begin{equation}
        \begin{aligned}
            \frac{\alpha}{\alpha -1} \ge \frac{\sigma_1}{\sigma_0} \iff \alpha \frac{\sigma_0-\sigma_1}{\sigma_0} \ge - \frac{\sigma_1}{\sigma_0} \iff \alpha \le \frac{\sigma_1}{\sigma_1-\sigma_0},
        \end{aligned}
    \end{equation}
    and for $\alpha=\frac{\sigma_1}{\sigma_1-\sigma_0}$, %
    $1-\alpha = -\frac{\sigma_0}{\sigma_1-\sigma_0}$. Hence the geodesic ray can be extended at least over the interval $]-\frac{\sigma_0}{\sigma_1-\sigma_0}, +\infty[$.
\end{proof}

\section{EXPERIMENTS ON LABELED DATASETS} \label{appendix:xp_details}

We report here experimental details of the experiments done in \Cref{sec:xps}. We begin with comparing the runtime of the different sliced distances. Then, we detail the experiment of correlation between the sliced distances and OTDD, as well as the gradient flows and transfer learning experiments. All the experiments are done on a Nvidia v100 GPU.

\begin{figure}[t]
    \centering
    \includegraphics[width=0.5\linewidth]{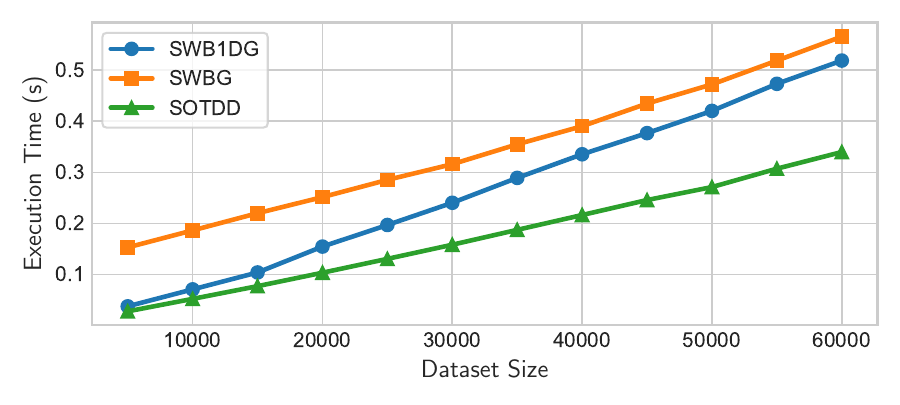}
    \caption{Comparison of the runtime in second between $\sotdd$, $\swbdg$ and $\swbg$ on subsets of MNIST.}
    \label{fig:runtime_sotdd_vs_swbdg}
\end{figure}

\subsection{Runtimes}

The theoretical computational complexity between $\sotdd$, $\swbdg$ and $\swbg$ is about the same \emph{w.r.t} the total number of samples $n$, \emph{i.e.} $\cO\big(Ln(\log n + d)\big)$. We verify this on \Cref{fig:runtime_sotdd_vs_swbdg} by plotting the runtime between subsets of MNIST with different number of samples. We observe the same asymptotic runtime, which are super linear. The runtimes reported are averaged over 10 tries, but appear stable. For $\swbg$, we used a dimension reduction in $\R^{10}$ with TSNE.

\subsection{Correlation} \label{appendix:xp_correlation}

Drawing inspiration from the experiment of \citep{nguyen2025lightspeed}, we consider OTDD to be the ideal distance between datasets and we aim to approximate it, or at least, we want to obtain bahavior similar to that of OTDD while being more efficient to compute. To assess the similarity between the sliced distances and OTDD, we measure their correlation.

Our protocol is the following, we first subsample a dataset, with random batchs of size between 5000 and 10000 samples. Then, we compute both the OTDD distance and the sliced distances between the pairs of batchs. Finally, we compute the Pearson and Spearman correlations to have a quantified value of the correlation. For pairs of data $(x_i,y_i)_{i=1}^n$, noting $\Bar{x}^n = \frac1n\sum_{i=1}^n x_i$ and $\Bar{y}^n = \frac1n\sum_{i=1}^n y_i$, the Pearson correlation $\rho_P$ is defined as
\begin{equation}
    \rho_P = \frac{\sum_{i=1}^n (x_i-\Bar{x}^n)(y_i-\Bar{y}^n)}{\sqrt{\sum_{i=1}^n (x_i - \Bar{x}^n)^2} \sqrt{\sum_{i=1}^n (y_i - \Bar{y}^n)^2}} = \frac{\mathrm{Cov}(X, Y)}{\sigma_X\sigma_Y},
\end{equation}
and is equal to $\pm 1$ when both quantities are linearly correlated. The Spearman correlation $\rho_S$ is defined similarly, but between the order statistics, \emph{i.e.} $(x_{\sigma_X(i)}, y_{\sigma_Y(i)})_{i=1}^n$ with $x_{\sigma_X(1)}\le \dots \le x_{\sigma_X(n)}$ and $y_{\sigma_Y(1)}\le \dots \le y_{\sigma_Y(n)}$ the sorted samples. This allows to recover if the quantities are correlated also in a non affine way. We use \texttt{SciPy} \citep{2020SciPy-NMeth} for their computation in practice.

OTDD and SOTDD are computed using the code from \citet{nguyen2025lightspeed} available at \url{https://github.com/hainn2803/s-OTDD}. For OTDD, the inner distances between the classes are computed by using the \texttt{POT} solver \citep{flamary2021pot, flamary2024pot} of the Wasserstein distance. The outer loss is computed using  Sinkhorn with an entropic regularization of $\varepsilon=10^{-3}$ and is debiased. SOTDD is computed using 5 moments. We report the results of all the sliced distances with 5000 projections on CIFAR10 on \Cref{fig:correlationcifar10}, and with $L\in\{10,50,100,500,1000,5000\}$ projections on CIFAR10 on \Cref{fig:correlation_cifar10_full} and on MNIST on \Cref{fig:correlation_mnist_full}. On MNIST, we use in every cases linear projections, while on CIFAR10, we use convolution projections \citep{nguyen2022revisiting}. We use the code of \citep{nguyen2025lightspeed} for the choice of the random convolutions. For SWBG, we embed labels as Gaussian in $\R^{10}$ using TSNE, with the \texttt{TorchDR} library \citep{torchdr}.

To obtain reasonably meaningful and statistically significant results, we report the results for 200 pairs of datasets on Figures \ref{fig:correlationcifar10}, \ref{fig:correlation_mnist_full} and \ref{fig:correlation_cifar10_full}. In each case, we observe that SWB1DG and SWBG outperform in general SOTDD. Moreover, both the Spearman correlation $\rho_S$ and the Pearson correlation $\rho_P$ are close to 1, indicating an affine correlation. To emphasize this linear correlation, we also fit a linear regression on each of these figures.

Finally, we report the results by averaging over 10 sets of 50 bootstrapped pairs over the 200 initial ones. We report the results for different number of projections on \Cref{tab:projs_vs_corr} on MNIST and FashionMNIST. We observe again the superior results of SWB1DG and SWBG over Sliced-OTDD for a much smaller number of projections.

\begin{figure}[H]
    \centering
    \includegraphics[width=0.45\linewidth]{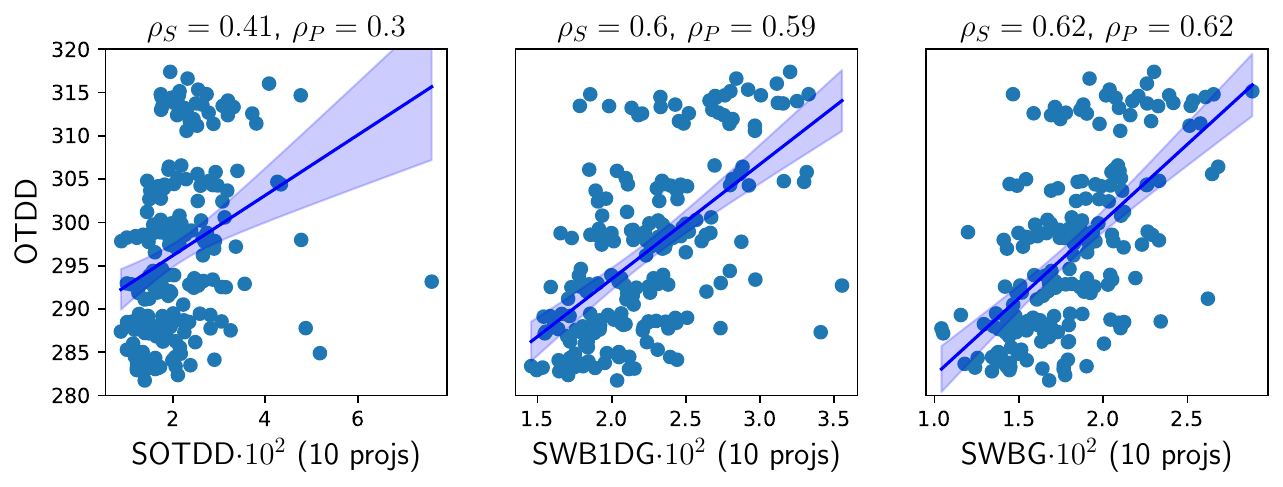}
    \includegraphics[width=0.45\linewidth]{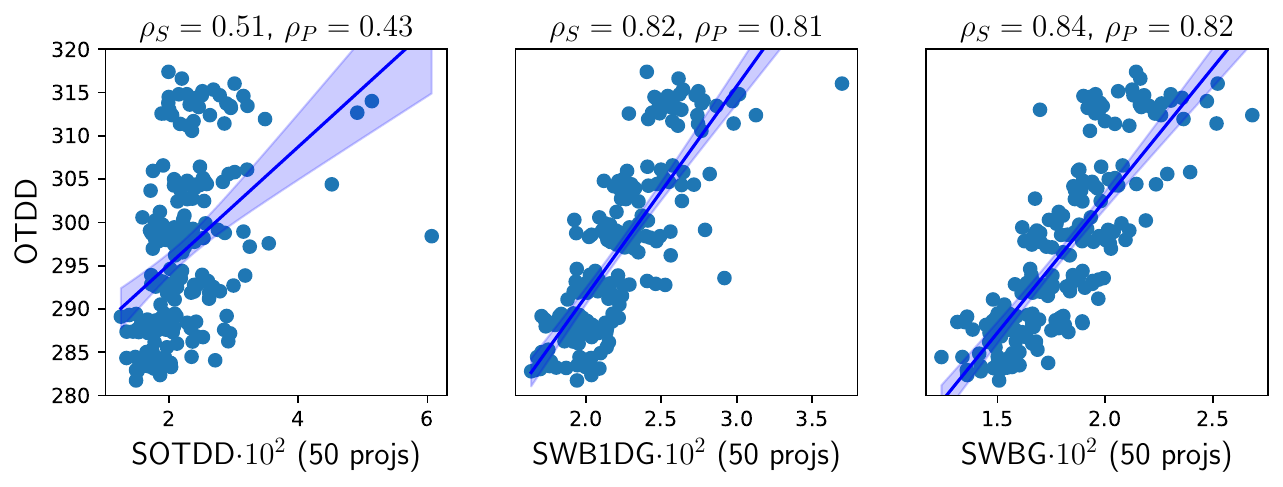}
    \includegraphics[width=0.45\linewidth]{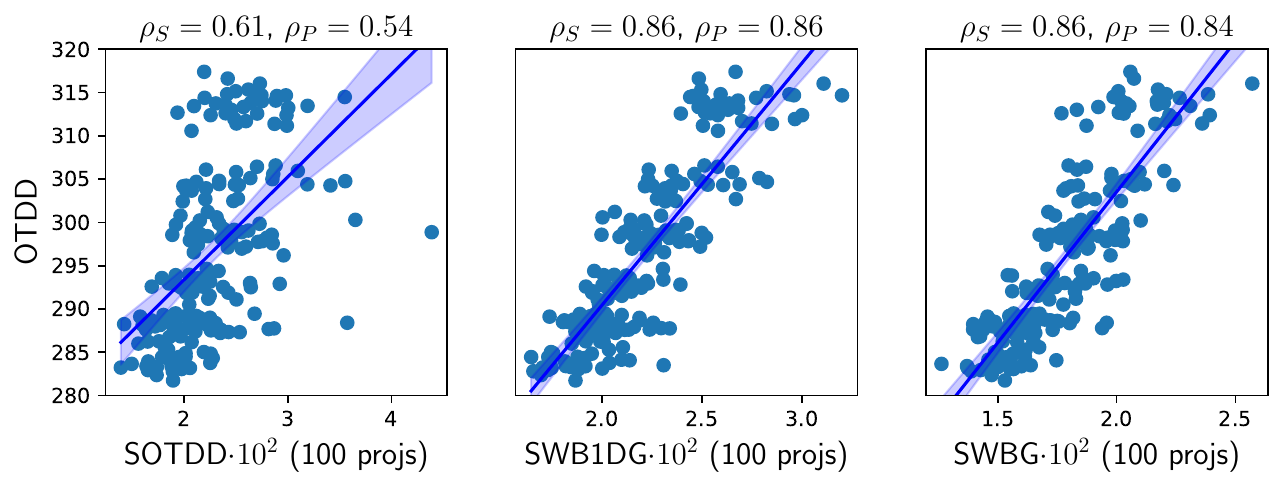}
    \includegraphics[width=0.45\linewidth]{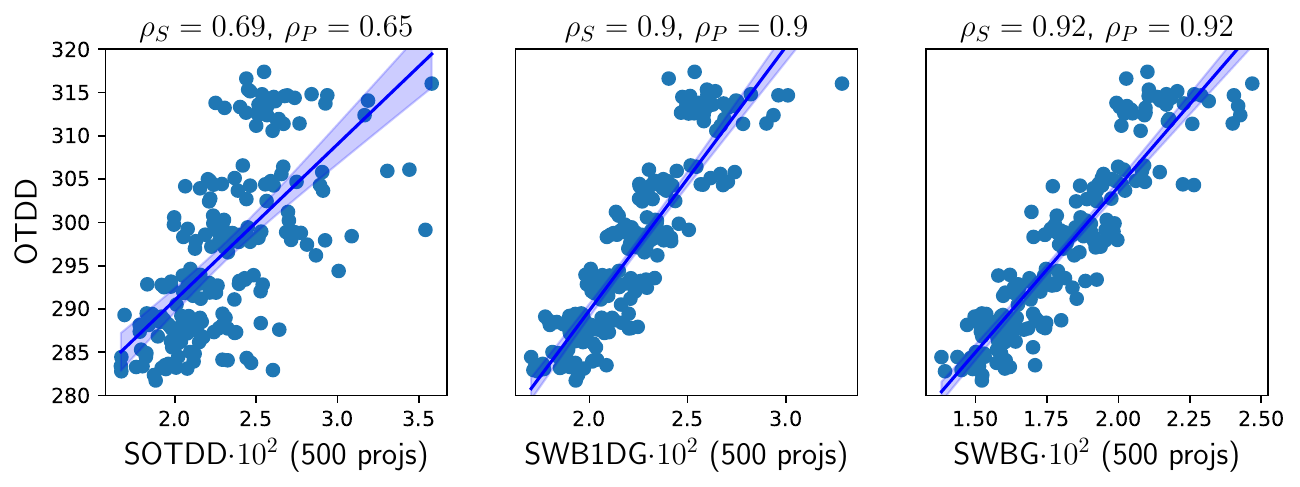}
    \includegraphics[width=0.45\linewidth]{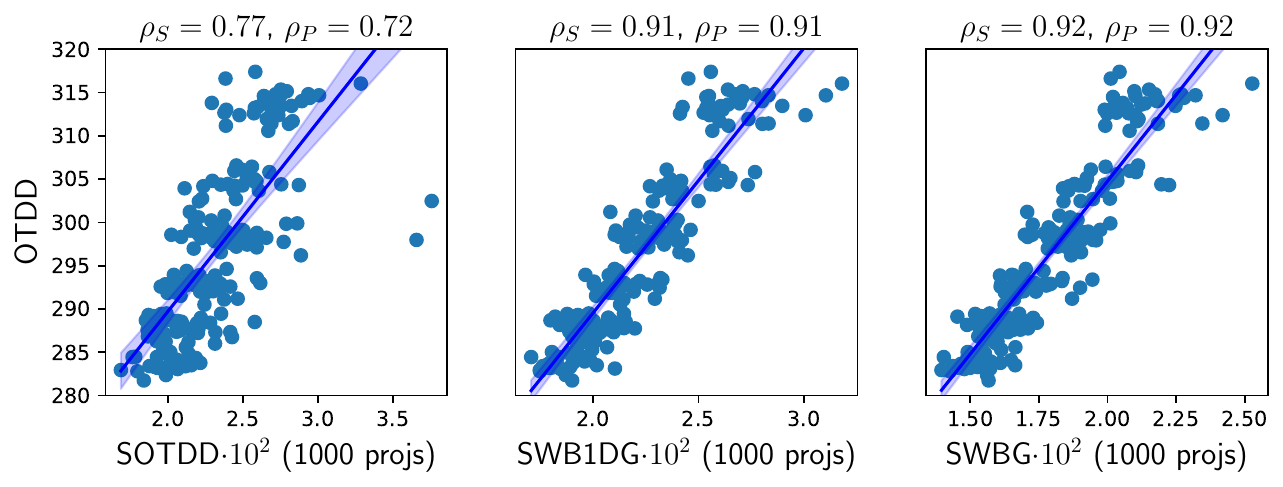}
    \includegraphics[width=0.45\linewidth]{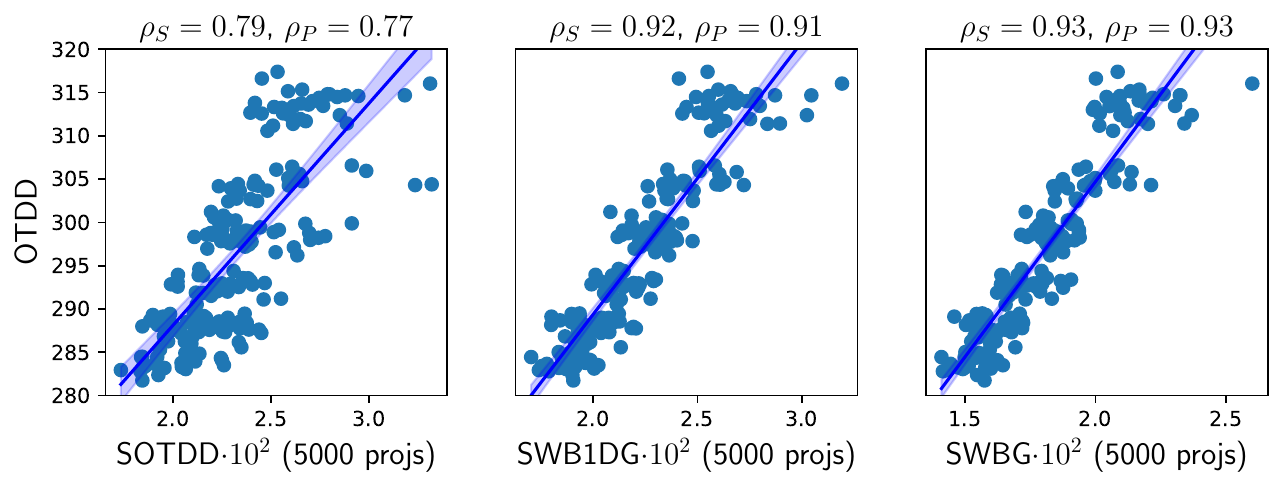}
    \caption{Distance correlation between s-OTDD, SWB1DG, SWBG and OTDD (exact) between subdatasets of MNIST.}
    \label{fig:correlation_mnist_full}
\end{figure}

\begin{figure}[H]
    \centering
    \includegraphics[width=0.45\linewidth]{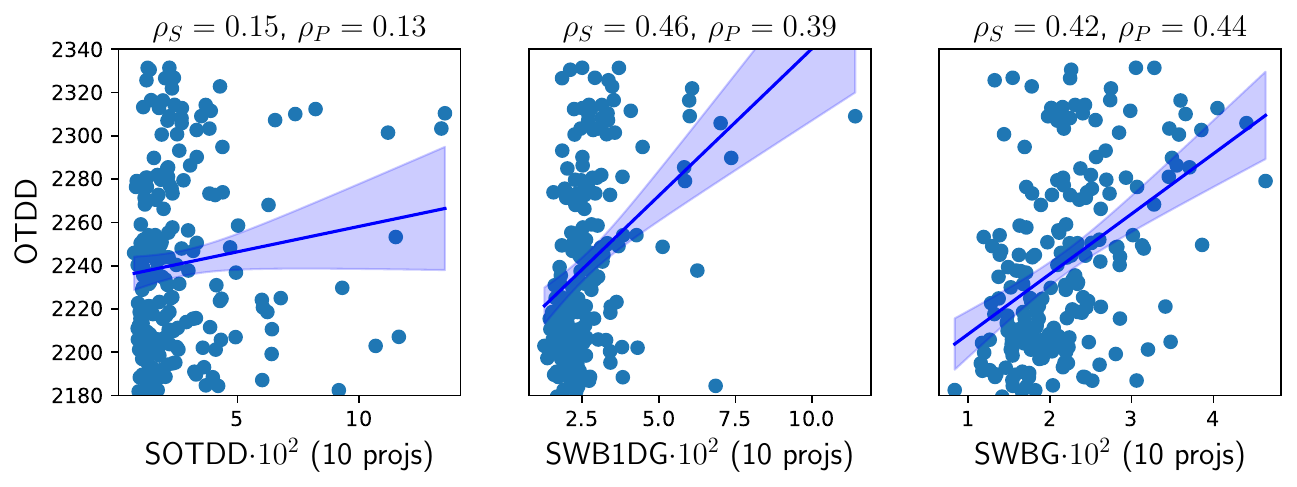}
    \includegraphics[width=0.45\linewidth]{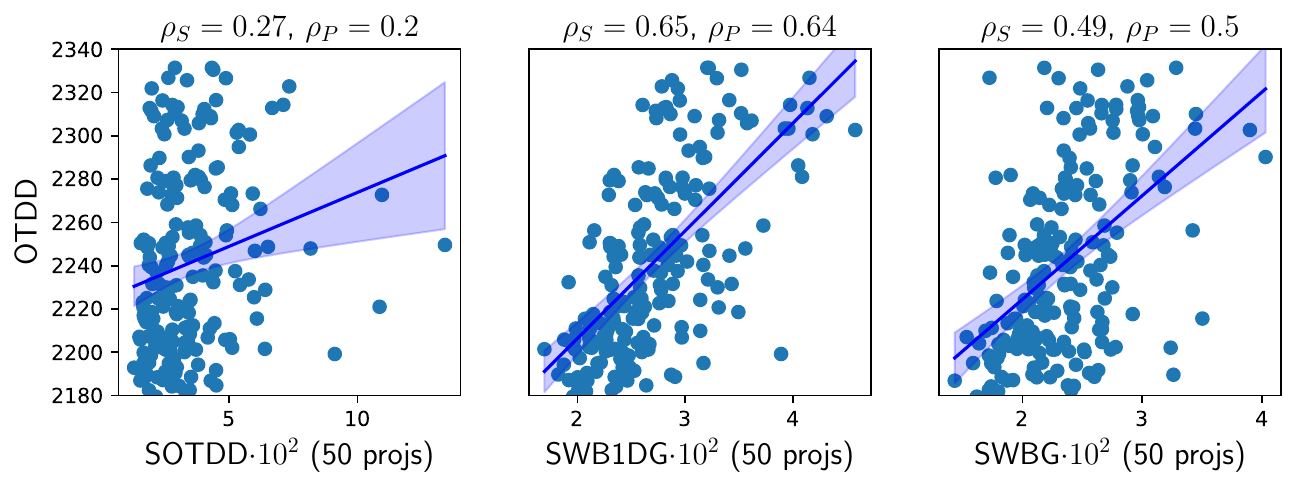}
    \includegraphics[width=0.45\linewidth]{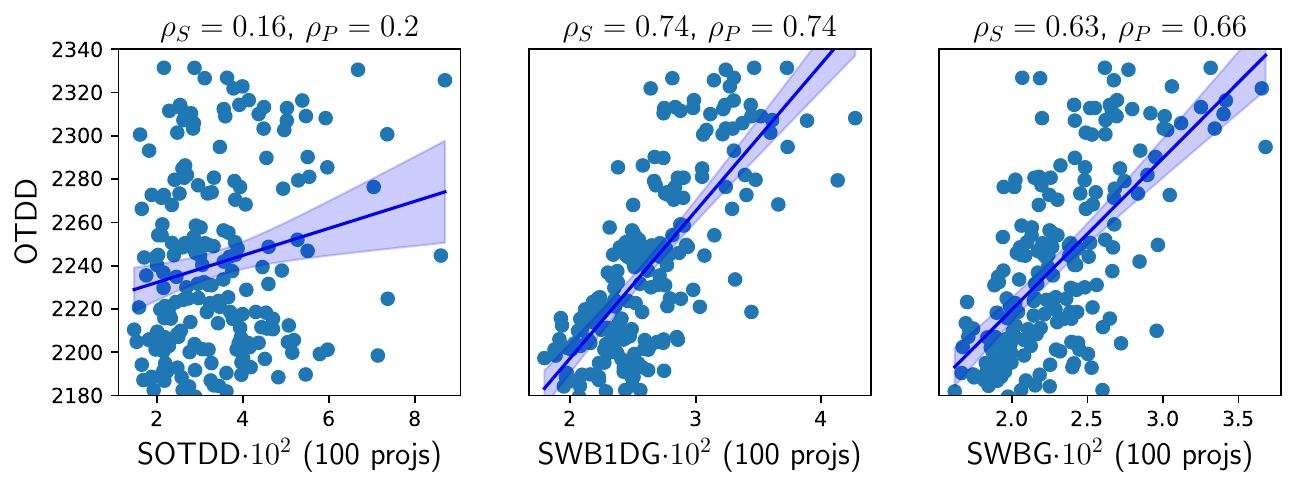}
    \includegraphics[width=0.45\linewidth]{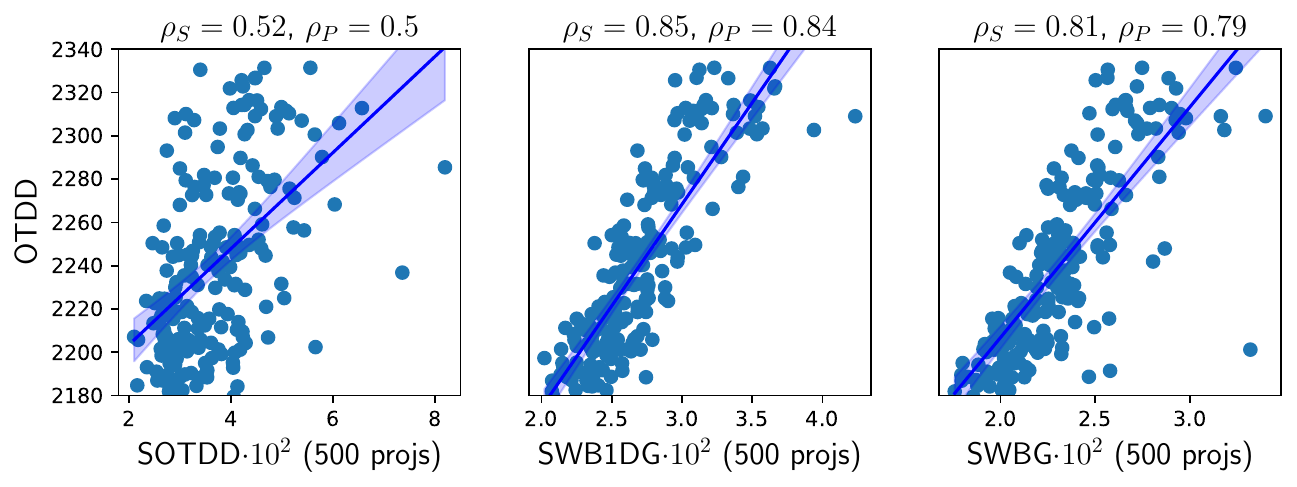}
    \includegraphics[width=0.45\linewidth]{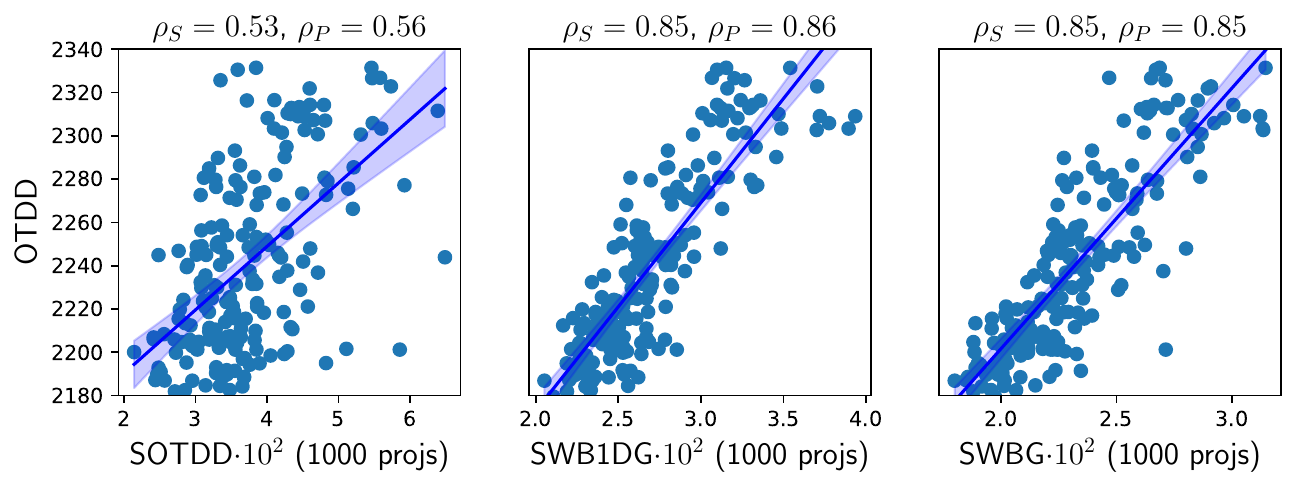}
    \includegraphics[width=0.45\linewidth]{Figures/Correlation/correlation_CIFAR10_5000projs.pdf}
    \caption{Distance correlation between s-OTDD, SWB1DG, SWBG and OTDD (exact) between subdatasets of CIFAR10.}
    \label{fig:correlation_cifar10_full}
\end{figure}

\begin{table*}[t]
    \centering
    \caption{Correlation for different number of projections averaged over 10 datasets of 50 bootstrapped pairs.}
    \resizebox{\linewidth}{!}{
        \begin{tabular}{ccccccccccccccccc}
             Projections & & \multicolumn{7}{c}{MNIST} & & \multicolumn{7}{c}{CIFAR10} \\ 
             & & \multicolumn{3}{c}{Spearman correlation $(\rho_S)$} & & \multicolumn{3}{c}{Pearson correlation ($\rho_P$)} & & \multicolumn{3}{c}{Spearman correlation $(\rho_S)$} & & \multicolumn{3}{c}{Pearson correlation ($\rho_P$)}  \\
             & & SOTDD & SWB1DG & SWBG & & SOTDD & SWB1DG & SWBG & & SOTDD & SWB1DG & SWBG & & SOTDD & SWB1DG & SWBG \\ \toprule
             10 & & $44.4_{\pm 6.4}$ & $57.8_{\pm 11.3}$ & $\textbf{61.5}_{\pm 10.4}$ & & $36.2_{\pm 12.2}$ & $60.2_{\pm 12.2}$ & $\textbf{61.3}_{\pm 13.1}$ & & $14.0_{\pm 11.3}$ & $\textbf{44.3}_{\pm 10.8}$ & $40.2_{\pm 12.2}$ & & $16.0_{\pm 12.9}$ & $38.6_{\pm 14.6}$ & $\textbf{42.7}_{\pm 9.5}$ \\
             50 & & $42.5_{\pm 8.9}$ & $81.6_{\pm 4.5}$ & $\textbf{83.4}_{\pm 2.4}$ & & $39.8_{\pm 11.5}$ & $81.7_{\pm 4.2}$ & $\textbf{82.8}_{\pm 2.0}$ & & $30.5_{\pm 12.9}$ & $\textbf{62.6}_{\pm 6.4}$ & $40.4_{\pm 9.8}$ & & $25.2_{\pm 11.4}$ & $\textbf{63.6}_{\pm 6.3}$ & $42.8_{\pm 8.5}$ \\
             100 & & $62.6_{\pm 7.4}$ & $84.0_{\pm 6.2}$ & $\textbf{84.3}_{\pm 3.5}$ & & $55.9_{\pm 9.5}$ & $\textbf{86.0}_{\pm 3.8}$ & $83.7_{\pm 3.4}$ & & $15.5_{\pm 11.8}$ & $\textbf{71.9}_{\pm 6.4}$ & $68.1_{\pm 7.2}$ & & $21.0_{\pm 11.4}$ & $\textbf{73.9}_{\pm 5.5}$ & $72.8_{\pm 5.4}$ \\
             500 & & $67.0_{\pm 7.5}$ & $90.4_{\pm 1.7}$ & $\textbf{91.2}_{\pm 1.6}$ & & $64.6_{\pm 7.0}$ & $90.4_{\pm 2.1}$ & $\textbf{91.6}_{\pm 2.0}$ & & $52.1_{\pm 8.1}$ & $\textbf{82.3}_{\pm 2.2}$ & $78.4_{\pm 6.0}$ & & $54.6_{\pm 8.8}$ & $\textbf{83.5}_{\pm 2.1}$ & $79.4_{\pm 7.7}$ \\
             1000 & & $77.7_{\pm 4.1}$ & $89.6_{\pm 1.5}$ & $\textbf{91.0}_{\pm 1.5}$ & & $75.6_{\pm 8.3}$ & $91.5_{\pm 1.3}$ & $\textbf{92.1}_{\pm 1.5}$ & & $52.0_{\pm 10.9}$ & $83.6_{\pm 4.8}$ & $\textbf{83.7}_{\pm 5.0}$ & & $53.1_{\pm 11.3}$ & $\textbf{85.6}_{\pm 3.5}$ & $84.9_{\pm 4.8}$ \\
             5000 & & $78.8_{\pm 6.2}$ & $91.4_{\pm 1.8}$ & $\textbf{92.6}_{\pm 1.6}$ & & $77.8_{\pm 5.8}$ & $91.4_{\pm 1.5}$ & $\textbf{93.0}_{\pm 1.6}$ & & $72.2_{\pm 7.5}$ & $88.5_{\pm 4.8}$ & $\textbf{89.3}_{\pm 3.8}$ & & $75.4_{\pm 5.5}$ & $87.8_{\pm 2.8}$ & $\textbf{89.0}_{\pm 2.4}$ \\
             10000 & & $78.7_{\pm 3.9}$ & $89.8_{\pm 0.2}$ & $\textbf{91.3}_{\pm 1.5}$ & & $78.9_{\pm 4.3}$ & $91.1_{\pm 2.0}$ & $\textbf{92.6}_{\pm 1.7}$ & & $72.6_{\pm 6.1}$ & $82.7_{\pm 4.8}$ & $\textbf{86.7}_{\pm 3.0}$ & & $77.1_{\pm 4.3}$ & $87.3_{\pm 2.8}$ & $\textbf{90.2}_{\pm 2.3}$ \\
             \bottomrule
        \end{tabular}
    }
    \label{tab:projs_vs_corr}
\end{table*}

\subsection{Gradient Flows on Rings} \label{appendix:rings}

The ring dataset is composed of 3 rings, where each ring is seen as a class of 80 samples. Thus, the target is $\bQ=\frac13 \delta_{\hat{\nu}_1} + \frac13\delta_{\hat{\nu}_2}+\frac13\delta_{\hat{\nu}_3}$ with $\hat{\nu}_c = \frac{1}{n}\sum_{i=1}^n \delta_{x_{i,c}}$ and $n=80$. On \Cref{fig:flow_rings}, we evenly sample the points on each ring, and thus have always the same target. On \Cref{fig:cv_rings}, we sampled 80 points uniformly on each ring, which means they may not be evenly spaced. In addition, we averaged the results over 100 different samples of the target.

We compare on \Cref{fig:cv_rings} the convergence of the WoW gradient flows of $\sotdd$, $\swbdg$ and $\swbg$ towards the target. These flows are approximated by performing a WoW gradient descent (see Appendix \ref{appendix:bg_wow} for details) over $\bP=\frac13\delta_{\mu_1}+\frac13\delta_{\mu_2}+\frac13\delta_{\mu_3}$ with each $\mu_c$ of the form $\mu_c=\frac1n\sum_{i=1}^n \delta_{x_{i,c}}$ for $n=80$, with 1000 iterations, and a step size of $\tau=1$.

We also tried minimizing $\sotdd$, $\swbdg$ and $\swbg$ with $\alpha_1=0$, \emph{i.e.} comparing the measures on $\cPP{\R^d}$ instead of on $\cPProd{\R^d}$. However, while it may be enough to compare the distributions, the sliced distances were not capable in this case to flow the points towards the rings. The flows only recovered roughly the means, and the first moments, but the particles did not exactly match the target rings. Thus, it is important to use the representations on $\cPProd{\R^d}$ to flow datasets with these distances.

\begin{table}[t]
    \centering
    \caption{Best hyperparameters on the $k$-shot transfer learning experiments for $\swbdg$ and $\sotdd$.}
    \resizebox{\linewidth}{!}{
        \begin{tabular}{ccccccccccccccccccccccc}
            & & \multicolumn{9}{c}{$\swbdg$} & & & \multicolumn{9}{c}{$\sotdd$} \\
            & & \multicolumn{4}{c}{MNIST to FMNIST} & & \multicolumn{4}{c}{MNIST to USPS} & & &\multicolumn{4}{c}{MNIST to FMNIST} & & \multicolumn{4}{c}{MNIST to USPS} \\
            & & $k=1$ & $k=5$ & $k=10$ & $k=100$ & & $k=1$ & $k=5$ & $k=10$ & $k=100$ & & & $k=1$ & $k=5$ & $k=10$ & $k=100$ & & $k=1$ & $k=5$ & $k=10$ & $k=100$ \\ \toprule
            Projections & & 5K & 5K & 1K & 10K & & 10K & 10K & 1K & 10K & & & 10K & 5K & 10K & 1K & & 5K & 10K & 10K & 10K \\
            Gradient steps & & 1K & 10K & 10K & 20K & & 1K & 20K & 20K & 20K & & & 10K & 10K & 10K & 20K & & 1K & 10K & 10K & 20K \\
            \bottomrule
        \end{tabular}
    }
        \label{tab:hyperparams_tf_swbdg}
\end{table}

\subsection{$k$-shot Learning} \label{appendix:xp_transfer_learning}

We provide the details of the $k$-shot learning experiment, inspired from \citep{alvarez2021dataset}. In this task, we want to learn a classifier on a dataset from which we have access to $k$ samples by class, where $k$ is typically small. \citet{alvarez2021dataset} proposed to solve this task by flowing a source dataset, from which we have access to more samples by class, towards the dataset of interest, hence augmenting each class with new additional images. 

Let $\cD^*$ be the dataset of interest and $C\in \mathbb{N}^*$ its number of classes. Let us denote $\bQ=\frac{1}{C}\sum_{c=1}^C \delta_{\nu_c^k} = \psi(\mu_{\cD^*})$ with $\nu_c^k=\frac1k \sum_{\ell=1}^k \delta_{y_\ell^{c}}$ the target dataset, considered here on $\cPP{\R^d}$. \citet{alvarez2021dataset} proposed to flow the datasets by minimizing $\bF(\bP)=\otdd(\bP,\bQ)$ starting from $\bP_0=\psi(\mu_{\cD_0})$ with $\mu_{\cD_0}$ an initial dataset with $n\gg k$ samples by class, \emph{i.e.} $\bP_0 = \frac1C\sum_{c=1}^C \delta_{\mu_c^{n,0}}$ and $\mu_c^{n,0} = \frac1n \sum_{i=1}^n \delta_{x_i^{c,0}}$. \citet{alvarez2021dataset} proposed several strategies to minimize $\otdd$. For instance, approximating the label distributions by Gaussian, they minimized it using gradient flows on the product space $\R^d \times \R^d \times S_d^{++}(\R)$, by flowing simultaneously the samples, the means and the covariances in a decouple way.

Following \citep{alvarez2021dataset}, several works proposed to solve this task. In particular, \citet{hua2023dynamic} solved it by minimizing the MMD with a Gaussian kernel over the product space $\R^d \times \R^p \times S_p^{++}(\R)$, using a dimension reduction to obtain embedding of the labels in a space of dimension $p\ll d$, and using a Riemannian gradient descent on the Bures-Wasserstein space for the covariance part. More recently, \citet{bonet2025flowing} proposed to minimize an MMD on $\cPP{\R^d}$ with kernels on $\cP_2(\R^d)$. In particular, they used the SW-Riesz kernel $K(\mu,\nu)=-\sw_2(\mu,\nu)$, and used a Wasserstein over Wasserstein Gradient descent to minimize it, endowing $\cPP{\R^d}$ with $\Ww$.

In this work, we also perform the gradient descent in $(\Ww, \cPP{\R^d})$ following the theory derived in \citep{bonet2025flowing}, and minimize $\sotdd$ and $\swbdg$. The scheme to minimize $\bF(\bP)=\mathrm{D}(\bP,\bQ)$ with $\mathrm{D}$ any differentiable divergence on $\cPP{\R^d}$ is, for any $\ell\ge 0$, $c \in \{1,\dots,C\}$ and $i\in \{1,\dots,n\}$, and step size $\tau>0$,
\begin{equation}
    x_{i}^{c, \ell+1} = x_i^{c,\ell} - \tau \gWw \bF(\bP^\ell)(\mu_c^{n,\ell})(x_i^{c,\ell}),
\end{equation}
where the Wasserstein over Wasserstein gradient $\gWw \bF(\bP^\ell)(\mu_c^{n,\ell})(x_i^{c,\ell})$  is obtained by rescaling the Euclidean gradient of $F(\textbf{x}^\ell) = \bF(\bP^\ell)$ for $\textbf{x}_c^\ell = (x_i^{c,\ell})_{i,c}$, \emph{i.e.} $\gWw \bF(\bP^\ell)(\mu_c^{n,\ell})(x_i^{c,\ell}) = nC\nabla F(\textbf{x}_c^\ell)$ \citep[Proposition B.7]{bonet2025flowing}. In practice, $\nabla F$ is obtained by backpropagation.

\begin{figure}[t]
    \centering
    \hspace*{\fill}
    \subfloat{\includegraphics[width=0.2\linewidth]{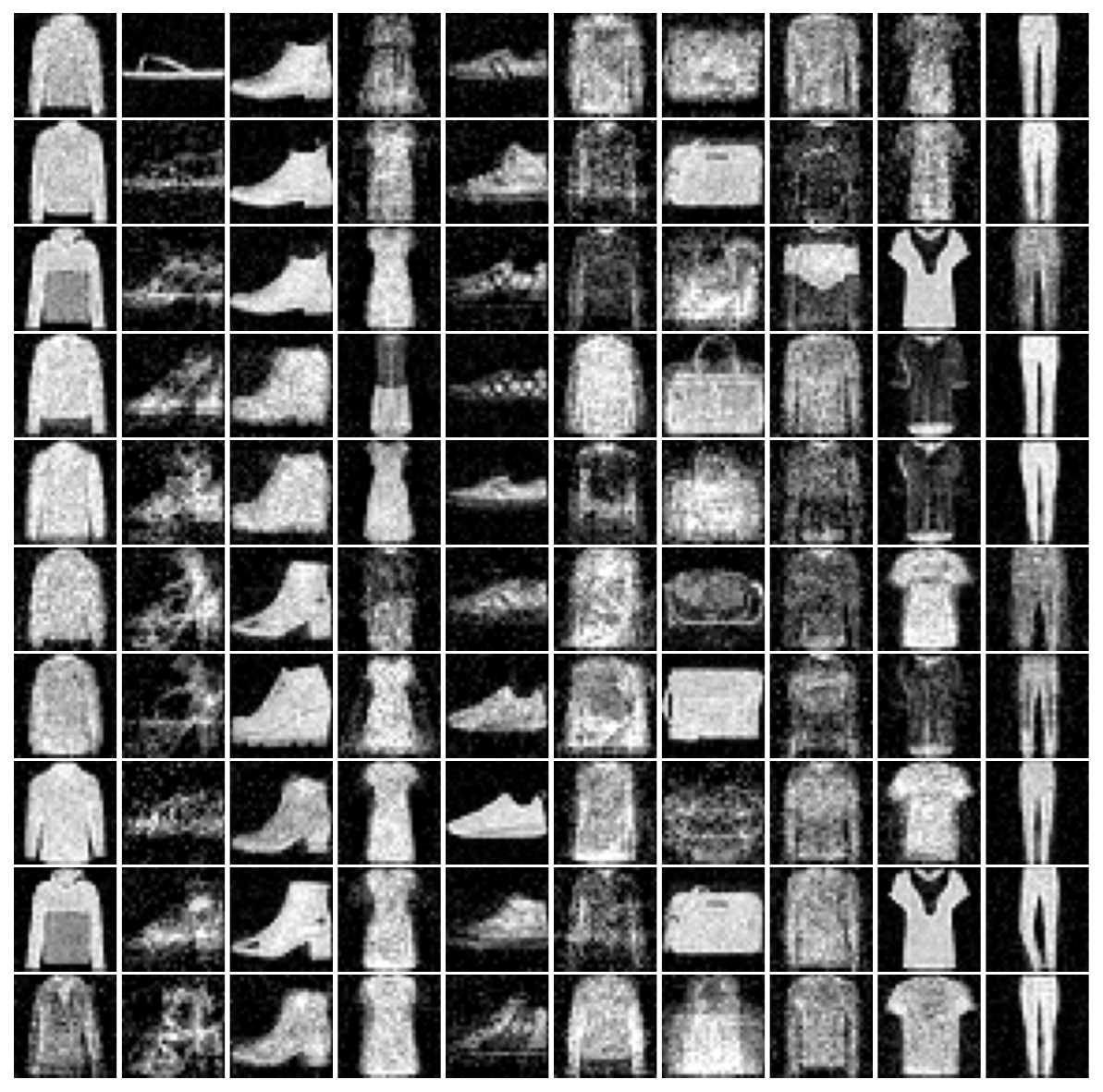}} \hfill
    \subfloat{\includegraphics[width=0.2\linewidth]{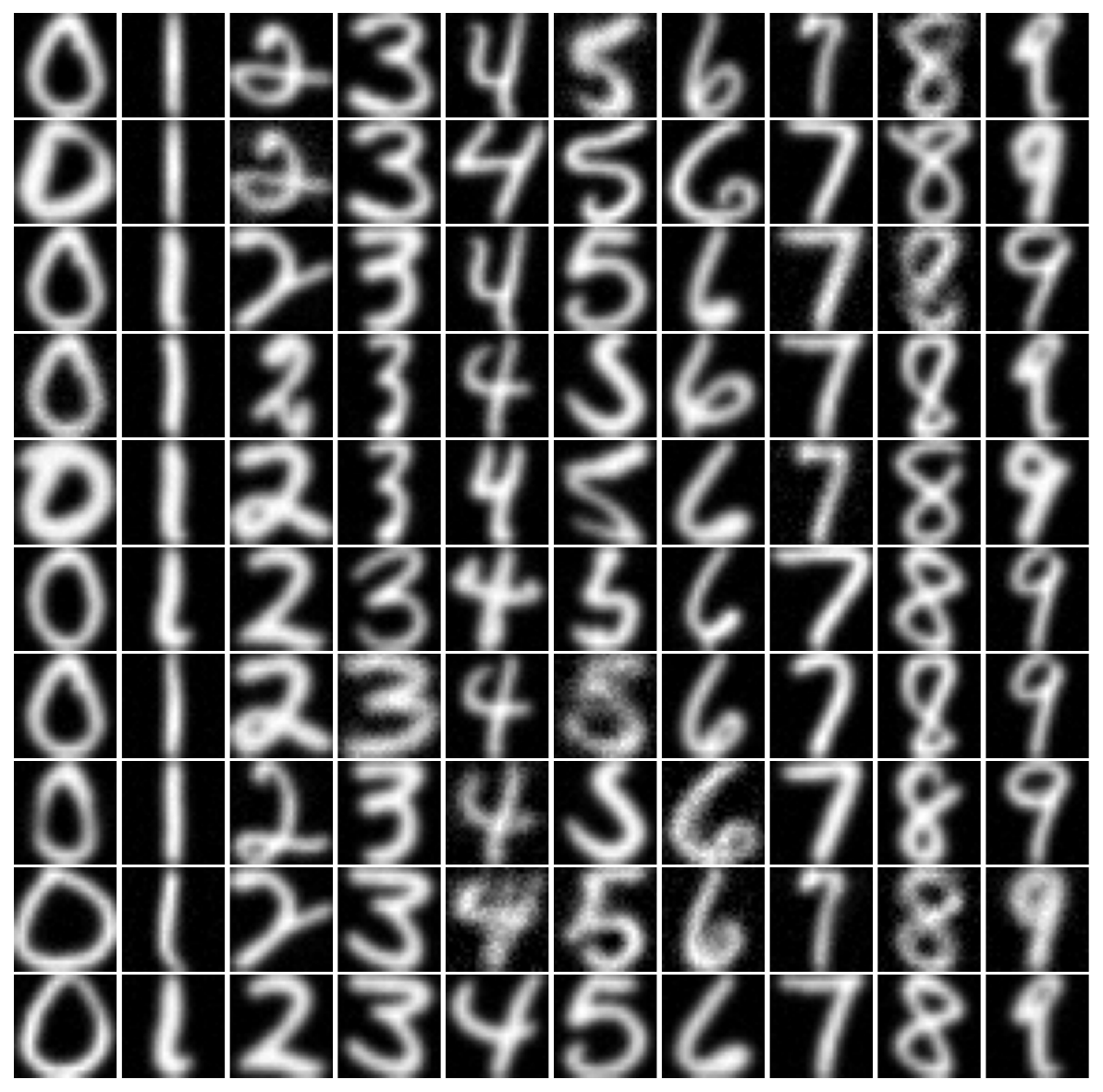}}
    \hfill
    \subfloat{\includegraphics[width=0.2\linewidth]{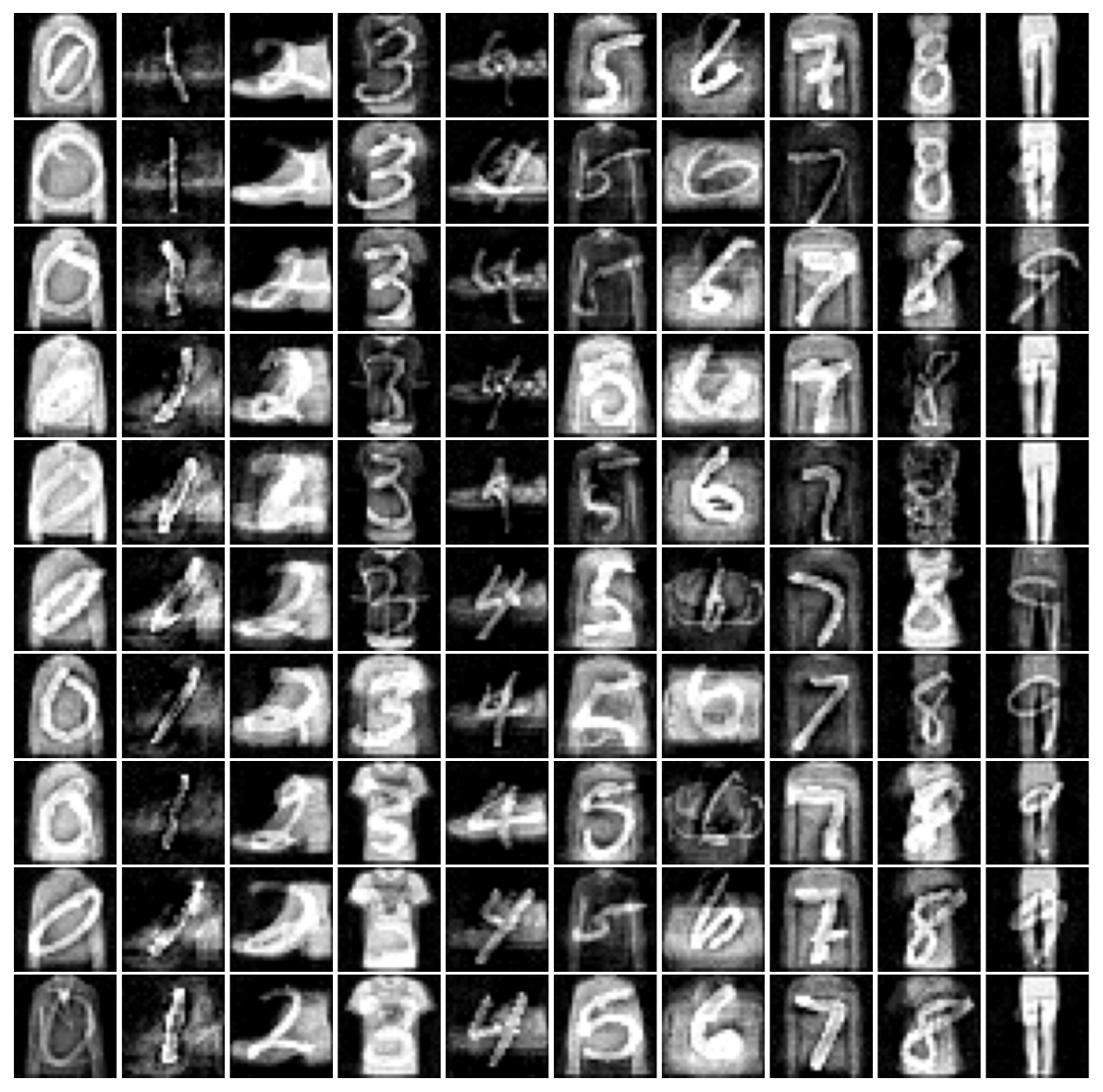}} \hfill
    \subfloat{\includegraphics[width=0.2\linewidth]{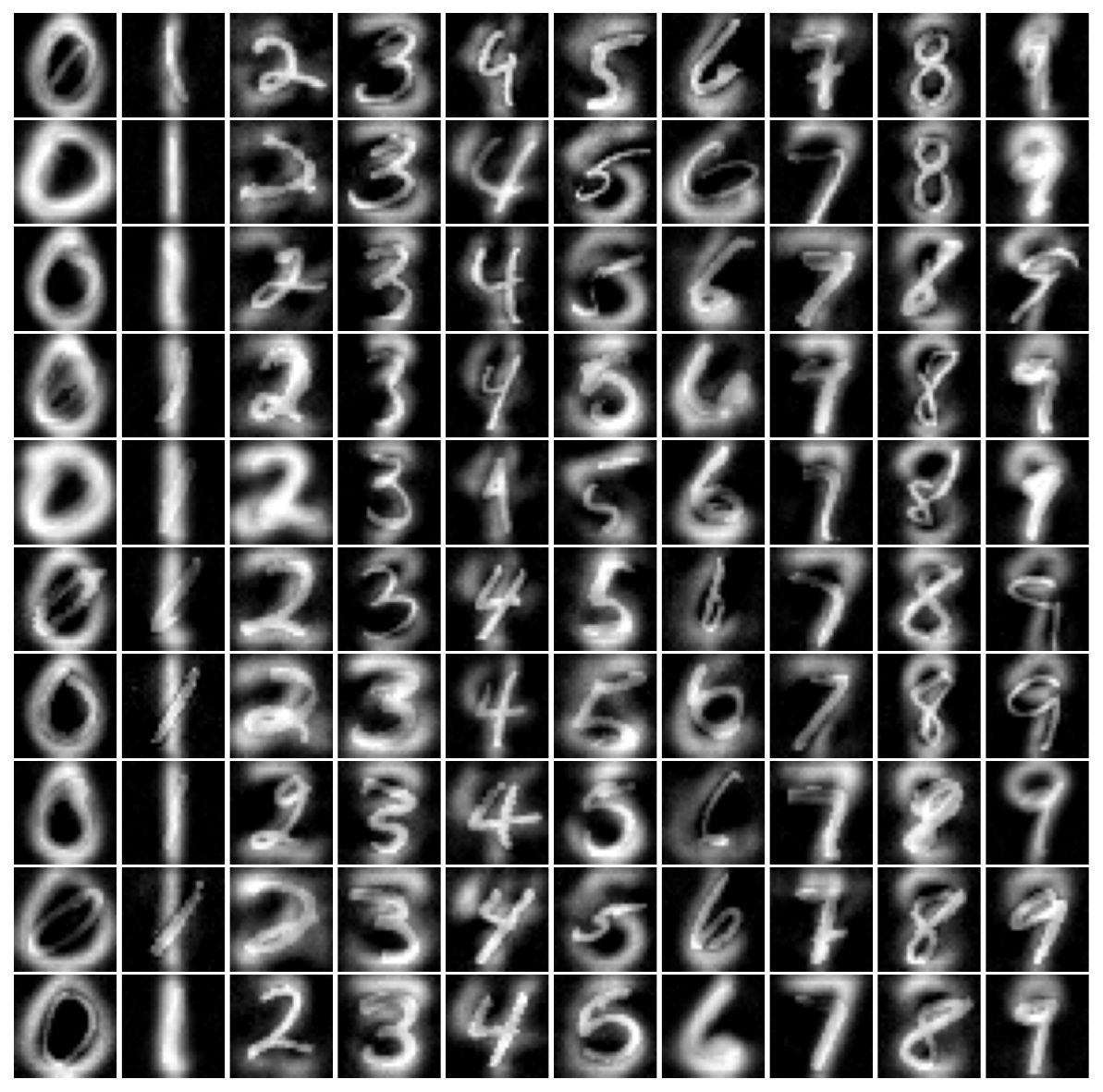}}
    \hspace*{\fill}
    \hspace*{\fill}
    \caption{Examples of images output by the flows for the transfer learning task on Fashion MNIST and USPS with $k=10$ and the best performing hyperparameters, for $\swbdg$ (\textbf{Left}) and $\sotdd$ (\textbf{Right}).}
    \label{fig:imgs_tf_swbdg_best}
\end{figure}

\begin{figure}[t]
    \centering
    \hspace*{\fill}
    \subfloat{\includegraphics[width=0.2\linewidth]{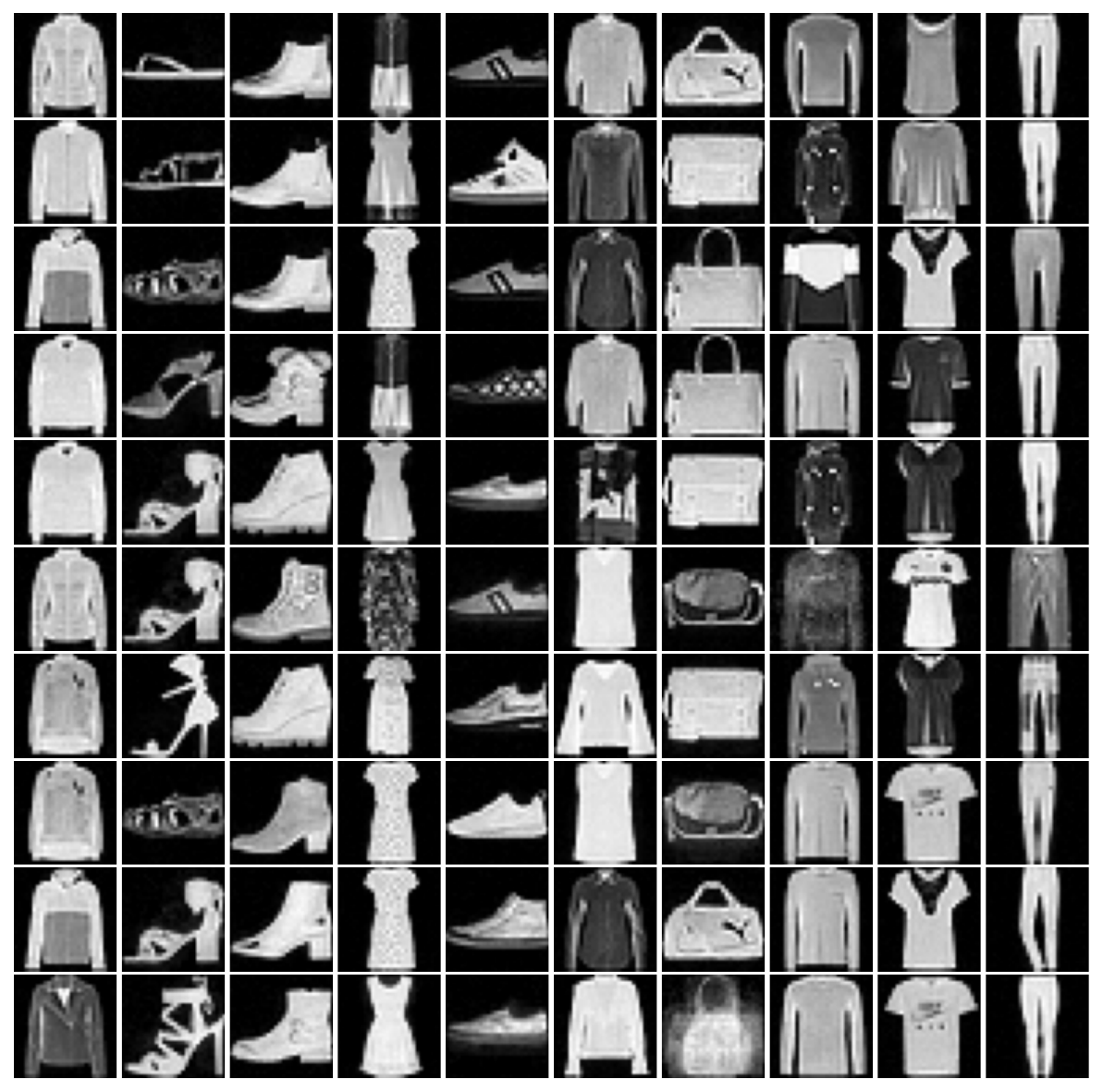}} \hfill
    \subfloat{\includegraphics[width=0.2\linewidth]{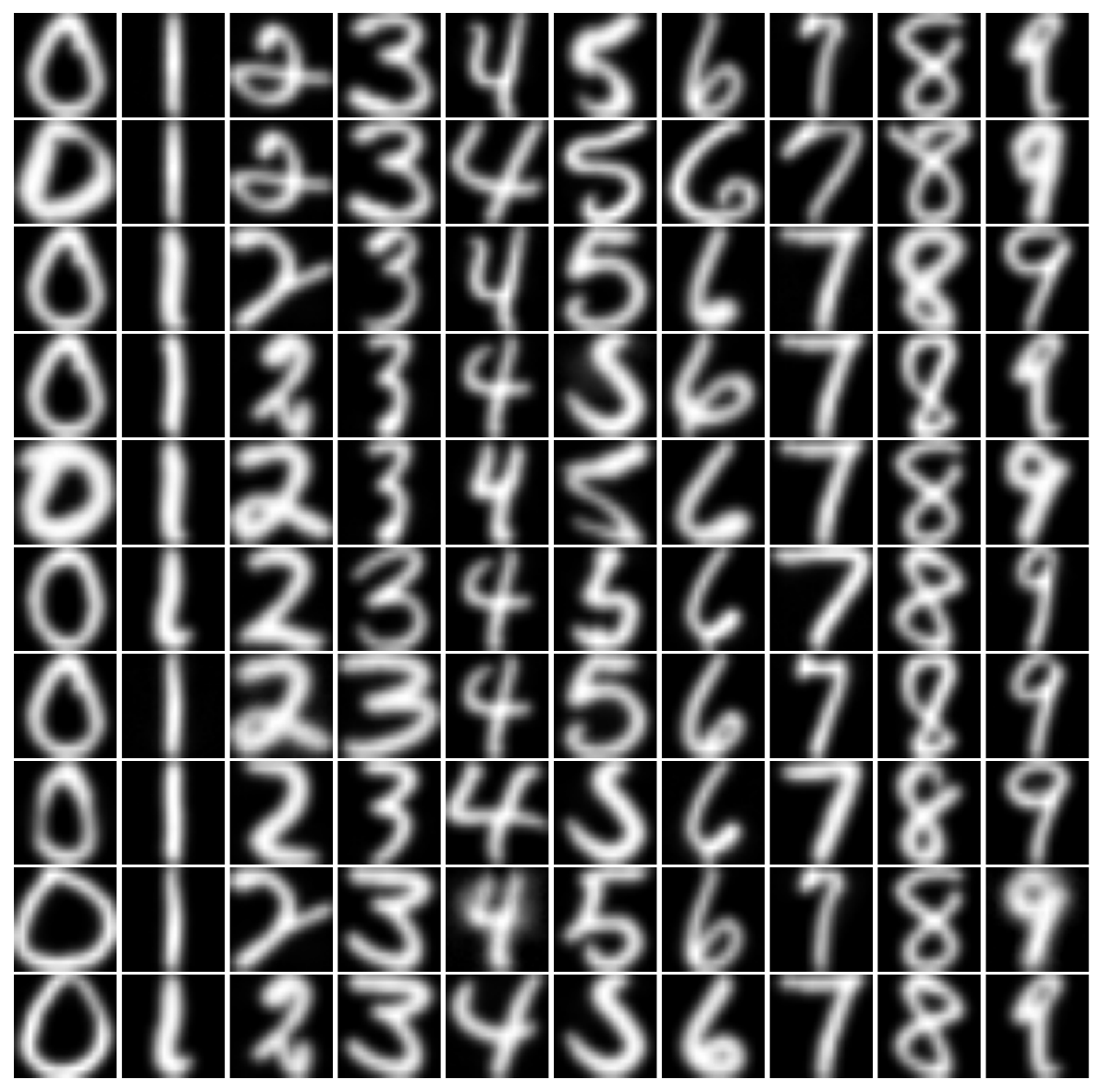}}
    \hfill
    \subfloat{\includegraphics[width=0.2\linewidth]{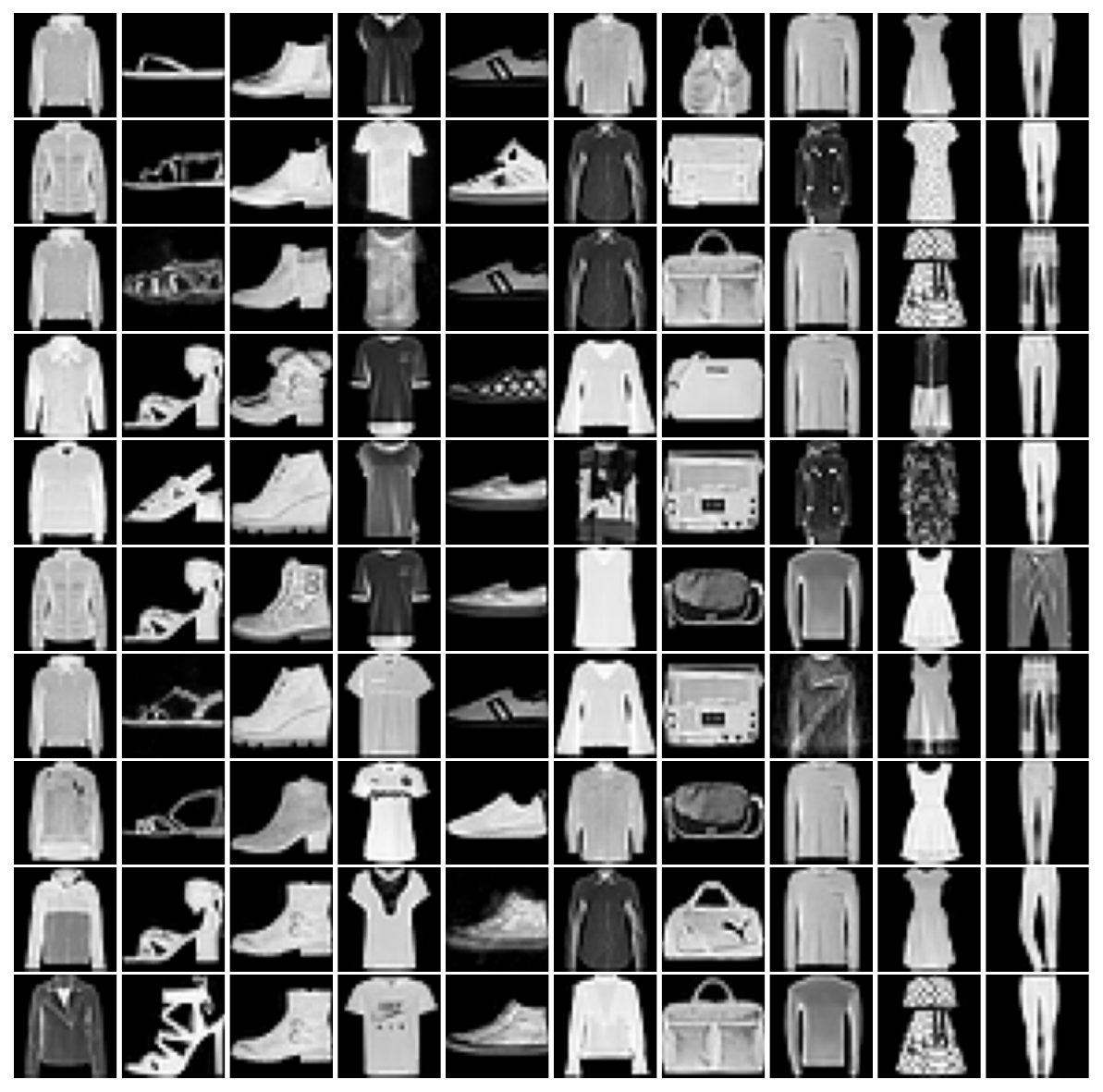}} \hfill
    \subfloat{\includegraphics[width=0.2\linewidth]{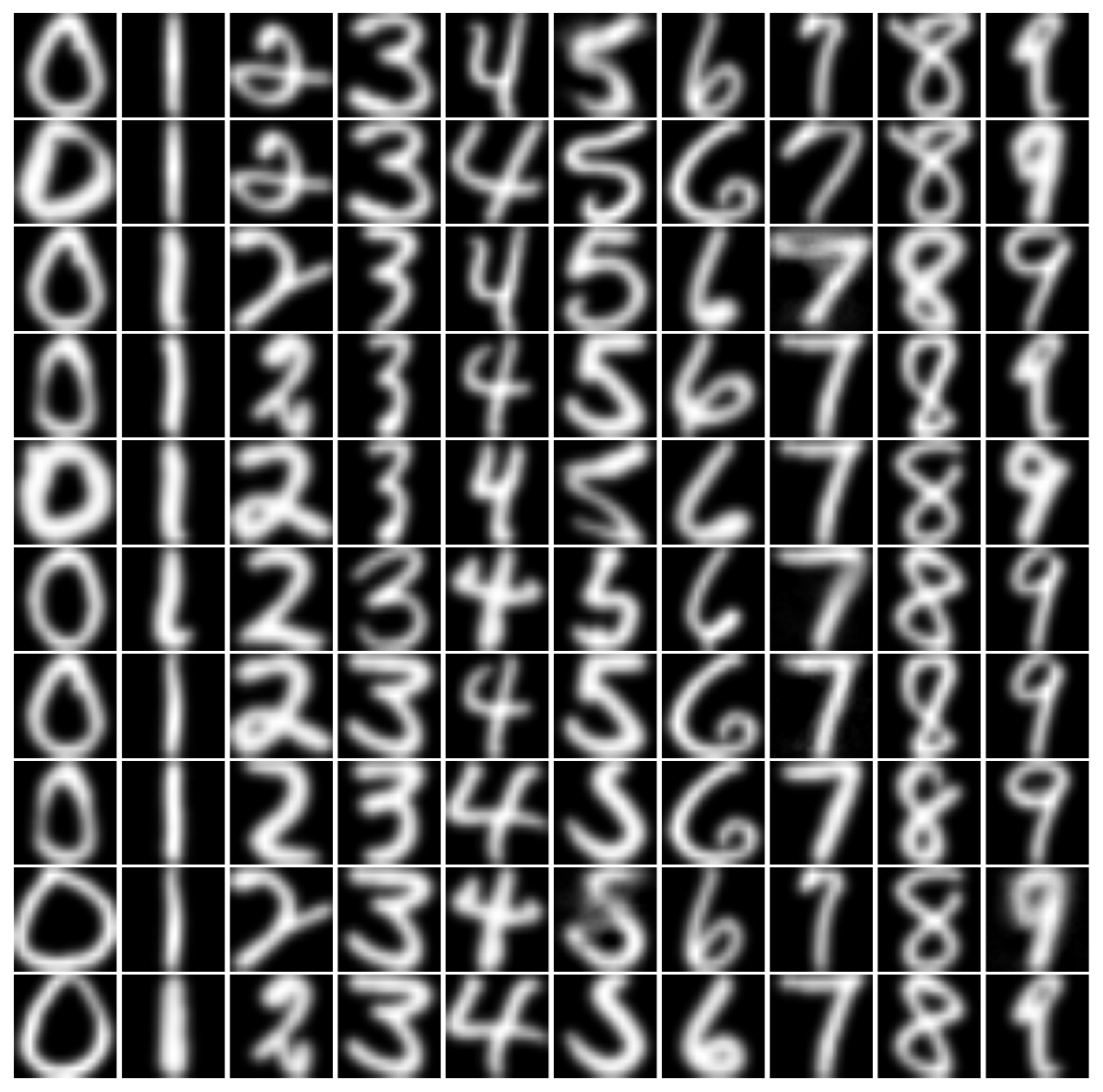}}
    \hspace*{\fill}
    \caption{Examples of images output by the flows for the transfer learning task on Fashion MNIST and USPS with $k=10$, 100K epochs and 10K projections, for $\swbdg$ (\textbf{Left}) and $\sotdd$ (\textbf{Right}).}
    \label{fig:imgs_tf_swbdg}
\end{figure}

\paragraph{Details of the Experiments.} 

\looseness=-1 The datasets of interest on which we learn classifiers are Fashion-MNIST and USPS. Thus, the number of class is always $C=10$, and we use $k\in \{1,5,10,100\}$. For the source dataset $\bP_0$, we always use MNIST with $n=200$ samples by class. In \Cref{tab:results_tf_grid_seach}, we report the accuracy obtained by training a LeNet-5 neural network for 50 epochs, with a AdamW optimizer and a learning rate of $3\cdot 10^{-4}$. We also average the results for 5 trainings of the neural network, and 3 outputs of the flows. The code to set-up the experiment is taken from the github of \citep{bonet2025flowing} available at \url{https://github.com/clbonet/Flowing_Datasets_with_WoW_Gradient_Flows}. We compare the results between the baseline where the neural network is trained directly on the target dataset $\cD^*$ with $k$ samples by class, and minimizing OTDD. These results are taken from \citep[Table 2]{bonet2025flowing}.

For the minimization of $\sotdd$ and $\swbdg$, we choose a step size of $\tau=1$, a momentum of $m=0.9$, and do a grid search over the number of gradient steps $T\in \{100,200,500,1000,10000,20000\}$ and the number of projections $L\in\{500,1000,5000,10000\}$ to approximate the integrals. We report in \Cref{tab:results_tf_grid_seach} the best results over this grid search, and on \Cref{tab:hyperparams_tf_swbdg} the values of the hyperparameters giving the best results. We also show on \Cref{fig:imgs_tf_swbdg_best} examples of images in each class obtained after minimizing respectively $\swbg$ and $\sotdd$, for $k=10$. We observe that these images are not always very clean. Thus, to show that minimizing these distances allows to obtain good looking images, we also report on \Cref{fig:imgs_tf_swbdg} results for $k=10$, $L=10K$ and $T=100K$, which however gave worse results than the ones reported in \Cref{tab:results_tf_grid_seach}. We hypothesize that the presence of noise can help the classifier to better generalize and thus to improve the accuracy results.

\section{EXPERIMENTS ON GAUSSIAN MIXTURES} \label{appendix:xps_gmms}

We show how the Gaussian Busemann Sliced-Wasserstein distances compare with $\W_{\bw}$, $\W_{\sw}$ and $\mathrm{DSMW}$ on the tasks of clustering detection and gradient flows.

\subsection{Clustering Detection}

\begin{figure}[t]
    \centering
    \includegraphics[width=\linewidth]{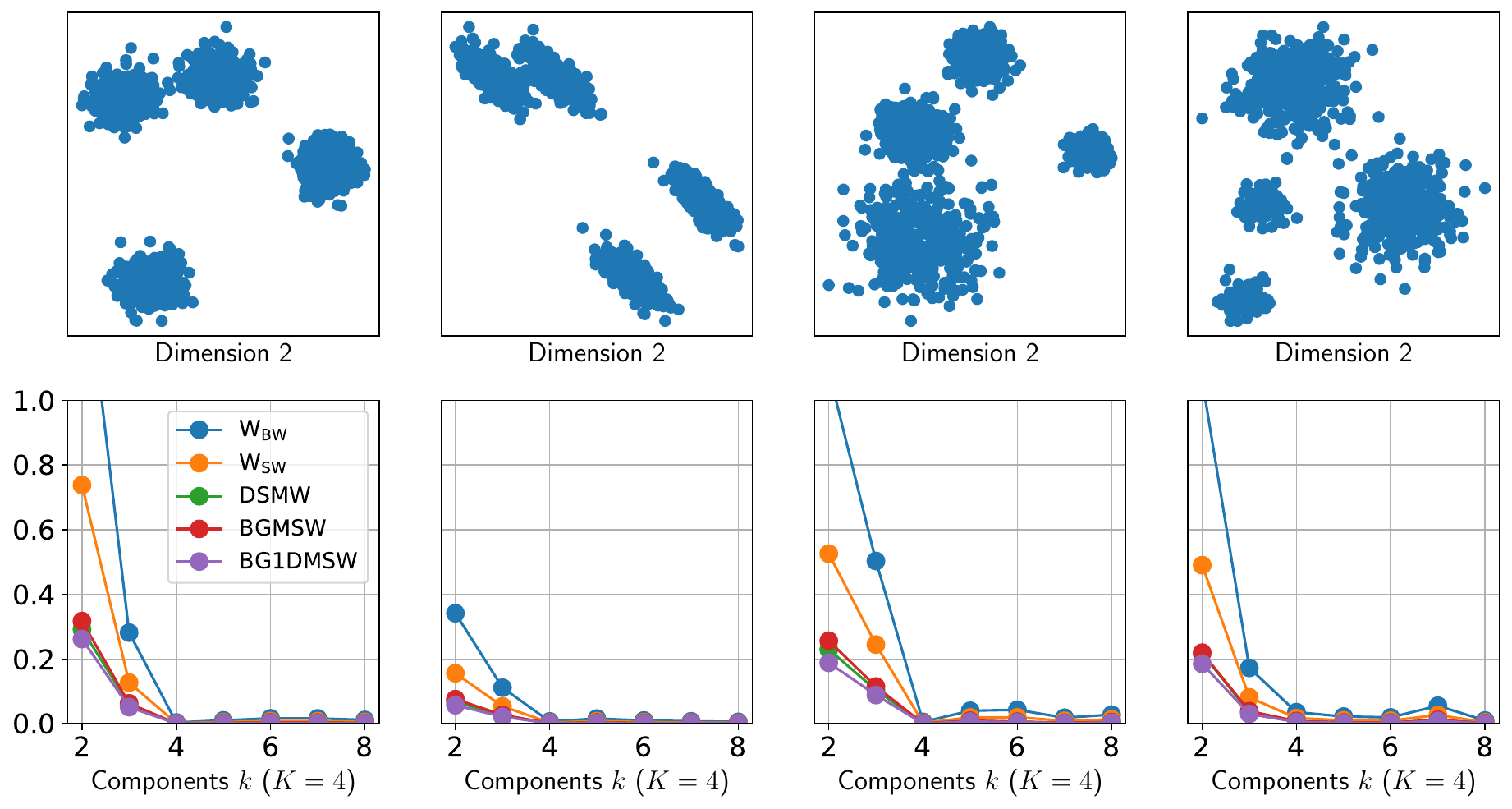}
    \caption{Detection of the number of clusters in dimension $d=2$.}
    \label{fig:cluster_detection_d2}
\end{figure}

\begin{figure}[t]
    \centering
    \includegraphics[width=\linewidth]{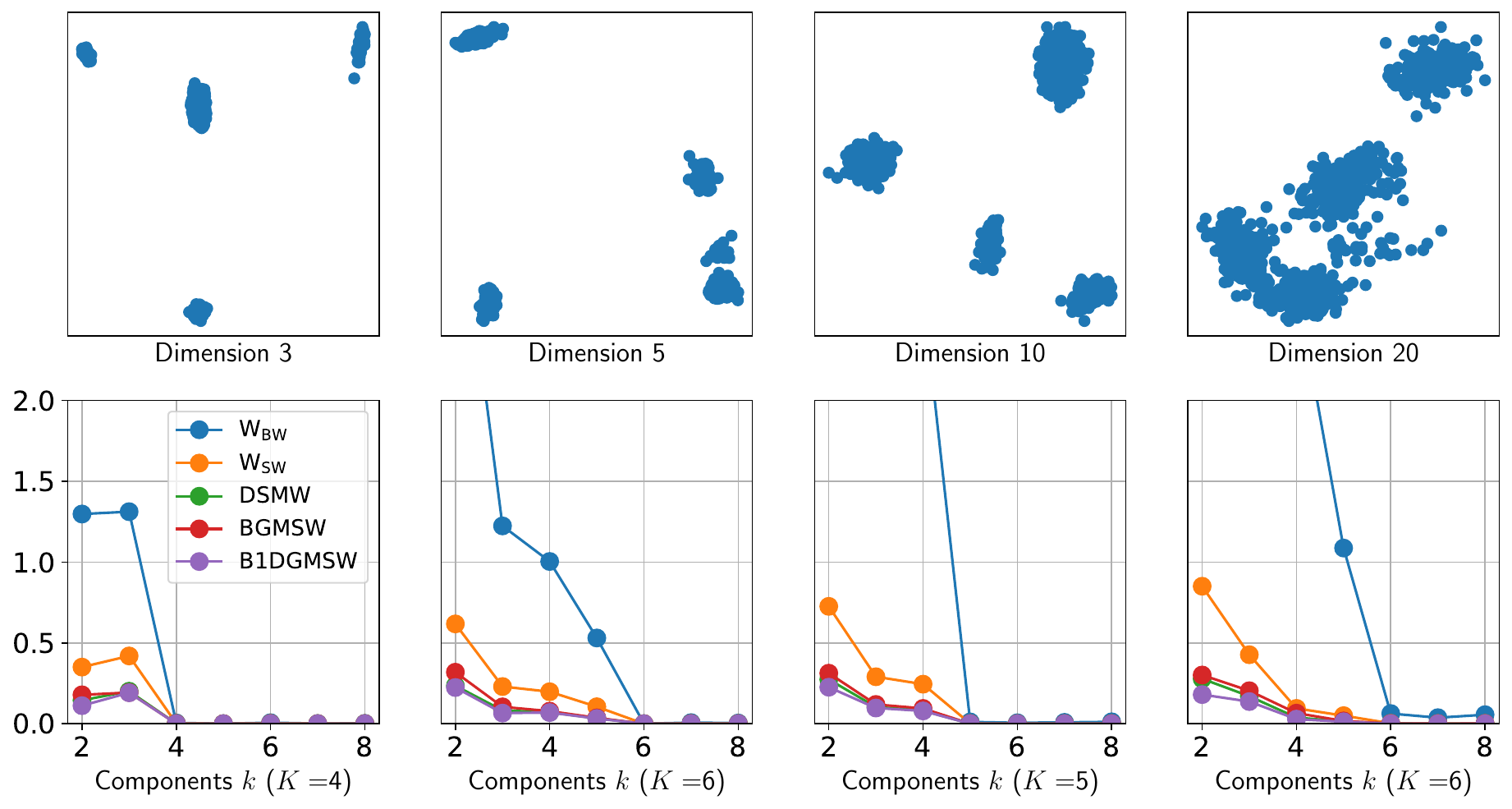}
    \caption{Detection of the number of clusters in dimension $d>2$.}
    \label{fig:cluster_detection_higher_dim}
\end{figure}

We perform the same experiment as in \citep[Section 4.3]{piening2025slicing}. We want to fit a Gaussian mixture on a dataset. A common technique is to use an EM. However, the number of clusters in the mixture needs to be specified and is not necessarily known beforehand. For $K$ components, denote $\bP_K\in\cPBW$ the corresponding mixture. \cite{piening2025slicing} propose to compute the distances between $\bP_k$ and $\bP_{k+1}$ for $k\ge 1$, and to find the smallest $k\in \mathbb{N}^*$ such that adding more components does not change the fitted model. To do so, they increase $k$ until the distances vanish.

We report on \Cref{fig:cluster_detection_d2} the results using the same setting as \citep{piening2025slicing}, \emph{i.e.} we have 1500 samples of a mixture in dimension $d=2$ with 4 clusters. On \Cref{fig:cluster_detection_higher_dim}, we add an experiment in dimensions $d\in\{3,5,10,20\}$ with $K\in\{4,5,6\}$ components, with means distributed uniformly on a sphere with radius $r=100$ and random covariance matrices. In every cases, we observe that the results between DSMW, $\smwbg$ and $\smwbdg$ are very close.

For DMSW, we use the code of \citep{piening2025slicing}, available at \url{https://github.com/MoePien/sliced_OT_for_GMMs} and for $\W_{\bw}$, we use the implementation of \texttt{POT} \citep{flamary2021pot}.

\subsection{Flows}

\begin{figure}[t]
    \centering
    \includegraphics[width=\linewidth]{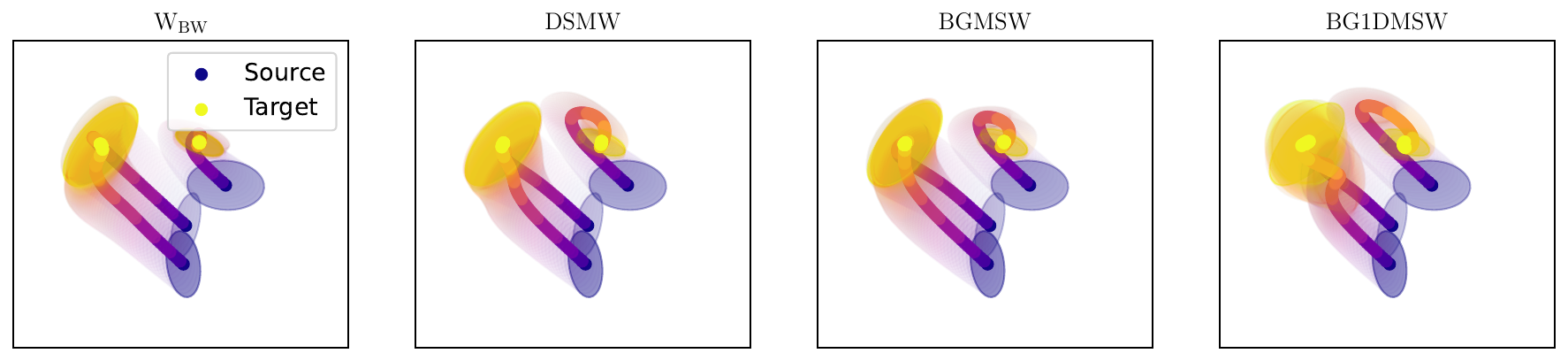}
    \caption{Flows of Gaussian Mixtures minimizing $\W_{\bw}$, $\mathrm{DMSW}$, $\smwbg$ or $\smwbdg$.}
    \label{fig:flow_gmm}
\end{figure}

As a proof of concept, we show on \Cref{fig:flow_gmm} the trajectories of flows minimizing $\W_{\bw}$, $\mathrm{DMSW}$, $\smwbg$ and $\smwbdg$ over $\cPBW$. We use the code from \texttt{POT} \citep{flamary2021pot} of an example minimizing $\W_{\bw}$ and available at \url{https://pythonot.github.io/auto_examples/gaussian_gmm/plot_GMM_flow.html#sphx-glr-auto-examples-gaussian-gmm-plot-gmm-flow-py}.

In this experiment, we start from a mixture with 3 Gaussian, with weights, means, and covariances randomly sampled. The target is a mixture with 2 Gaussian with also weights, means and covariances randomly sampled. Then, we optimize over the weights, means and covariances using the Adam optimizer and with projection steps to stay on the space of Gaussian mixtures, \emph{i.e.} we use a softmax to project the weights and clip the eigenvalues of the matrices to project the covariances in the space of positive definite matrices. The gradients are obtained using backpropagation. Note that this way of minimizing does not correspond to a gradient descent in $\cPBW$, which would not allow to change the weights, see \emph{e.g.} \citep[Appendix D.7]{bonet2025flowing} for a discussion on how to handle different number of components in the mixture.

%% file: main.bbl
\begin{thebibliography}{94}
\providecommand{\natexlab}[1]{#1}
\providecommand{\url}[1]{\texttt{#1}}
\expandafter\ifx\csname urlstyle\endcsname\relax
  \providecommand{\doi}[1]{doi: #1}\else
  \providecommand{\doi}{doi: \begingroup \urlstyle{rm}\Url}\fi

\bibitem[Alfonsi and Jourdain(2025)]{alfonsi2025wasserstein}
Aur{\'e}lien Alfonsi and Benjamin Jourdain.
\newblock {Wasserstein projections in the convex order: regularity and characterization in the quadratic Gaussian case}.
\newblock \emph{arXiv preprint arXiv:2506.23981}, 2025.

\bibitem[Altschuler et~al.(2021)Altschuler, Chewi, Gerber, and Stromme]{altschuler2021averaging}
Jason Altschuler, Sinho Chewi, Patrik~R Gerber, and Austin Stromme.
\newblock {Averaging on the Bures-Wasserstein manifold: dimension-free convergence of gradient descent}.
\newblock \emph{Advances in Neural Information Processing Systems}, 34:\penalty0 22132--22145, 2021.

\bibitem[Alvarez-Melis and Fusi(2020)]{alvarez2020geometric}
David Alvarez-Melis and Nicolo Fusi.
\newblock {Geometric Dataset Distances via Optimal Transport}.
\newblock \emph{Advances in Neural Information Processing Systems}, 33:\penalty0 21428--21439, 2020.

\bibitem[Alvarez-Melis and Fusi(2021)]{alvarez2021dataset}
David Alvarez-Melis and Nicol{\`o} Fusi.
\newblock {Dataset Dynamics via Gradient Glows in Probability Space}.
\newblock In \emph{International conference on machine learning}, pages 219--230. PMLR, 2021.

\bibitem[Ambrosio et~al.(2008)Ambrosio, Gigli, and Savar{\'e}]{ambrosio2008gradient}
Luigi Ambrosio, Nicola Gigli, and Giuseppe Savar{\'e}.
\newblock \emph{{Gradient Flows: in Metric Spaces and in the Space of Probability Measures}}.
\newblock Springer Science \& Business Media, 2008.

\bibitem[Baouan et~al.(2025)Baouan, Rosenbaum, and Pulido]{baouan2025optimal}
Ali Baouan, Mathieu Rosenbaum, and Sergio Pulido.
\newblock {An optimal transport based embedding to quantify the distance between playing styles in collective sports}.
\newblock \emph{Journal of Quantitative Analysis in Sports}, 2025.

\bibitem[Beiglb{\"o}ck et~al.(2025)Beiglb{\"o}ck, Pammer, and Schrott]{beiglbock2025brenier}
Mathias Beiglb{\"o}ck, Gudmund Pammer, and Stefan Schrott.
\newblock {A Brenier Theorem on $(\mathcal{P}_2(\mathcal{P}_2(\mathbb{R}^d)), W_2)$ and Applications to Adapted Transport}.
\newblock \emph{arXiv preprint arXiv:2509.03506}, 2025.

\bibitem[Bellazzi et~al.(2021)Bellazzi, Codegoni, Gualandi, Nicora, and Vercesi]{bellazzi2021gene}
Riccardo Bellazzi, Andrea Codegoni, Stefano Gualandi, Giovanna Nicora, and Eleonora Vercesi.
\newblock {The Gene Mover's Distance: Single-cell similarity via Optimal Transport}.
\newblock \emph{arXiv preprint arXiv:2102.01218}, 2021.

\bibitem[Berg et~al.(2024)Berg, Michele, Pham, Chapel, and Courty]{berg2024horospherical}
Paul Berg, Bjoern Michele, Minh-Tan Pham, Laetitia Chapel, and Nicolas Courty.
\newblock {Horospherical Learning with Smart Prototypes}.
\newblock In \emph{British Machine Vision Conference (BMVC)}, 2024.

\bibitem[Berg et~al.(2025)Berg, Buecher, Michele, Pham, Chapel, and Courty]{berg2025multi}
Paul Berg, L{\'e}o Buecher, Bj{\"o}rn Michele, Minh-Tan Pham, Laetitia Chapel, and Nicolas Courty.
\newblock {Multi-Prototype Hyperbolic Learning Guided by Class Hierarchy}.
\newblock \emph{International Journal of Computer Vision}, pages 1--16, 2025.

\bibitem[Bertrand and Kloeckner(2012)]{bertrand2012geometric}
J{\'e}r{\^o}me Bertrand and Beno{\^\i}t Kloeckner.
\newblock {A geometric study of Wasserstein spaces: Hadamard spaces}.
\newblock \emph{Journal of Topology and Analysis}, 4\penalty0 (04):\penalty0 515--542, 2012.

\bibitem[Bertrand and Kloeckner(2016)]{bertrand2016geometric}
Jérôme Bertrand and Beno{\^\i}t Kloeckner.
\newblock {A geometric study of Wasserstein spaces: isometric rigidity in negative curvature}.
\newblock \emph{International Mathematics Research Notices}, 2016\penalty0 (5):\penalty0 1368--1386, 2016.

\bibitem[Bhatia(2013)]{bhatia2013matrix}
Rajendra Bhatia.
\newblock \emph{{Matrix Analysis}}, volume 169.
\newblock Springer Science \& Business Media, 2013.

\bibitem[Bhatia et~al.(2019)Bhatia, Jain, and Lim]{bhatia2019bures}
Rajendra Bhatia, Tanvi Jain, and Yongdo Lim.
\newblock {On the Bures--Wasserstein distance between positive definite matrices}.
\newblock \emph{Expositiones Mathematicae}, 37\penalty0 (2):\penalty0 165--191, 2019.

\bibitem[Bonet et~al.(2025{\natexlab{a}})Bonet, Vauthier, and Korba]{bonet2025flowing}
Cl{\'e}ment Bonet, Christophe Vauthier, and Anna Korba.
\newblock {Flowing Datasets with Wasserstein over Wasserstein Gradient Flows}.
\newblock In \emph{Forty-second International Conference on Machine Learning}, 2025{\natexlab{a}}.

\bibitem[Bonet et~al.(2023{\natexlab{a}})Bonet, Chapel, Drumetz, and Courty]{bonet2023hyperbolic}
Clément Bonet, Laetitia Chapel, Lucas Drumetz, and Nicolas Courty.
\newblock {Hyperbolic Sliced-Wasserstein via Geodesic and Horospherical Projections}.
\newblock In \emph{Proceedings of 2nd Annual Workshop on Topology, Algebra, and Geometry in Machine Learning (TAG-ML)}, pages 334--370. PMLR, 2023{\natexlab{a}}.

\bibitem[Bonet et~al.(2023{\natexlab{b}})Bonet, Malézieux, Rakotomamonjy, Drumetz, Moreau, Kowalski, and Courty]{bonet2023sliced}
Clément Bonet, Benoît Malézieux, Alain Rakotomamonjy, Lucas Drumetz, Thomas Moreau, Matthieu Kowalski, and Nicolas Courty.
\newblock {Sliced-Wasserstein on Symmetric Positive Definite Matrices for M/EEG Signals}.
\newblock In \emph{International Conference on Machine Learning}, pages 2777--2805. PMLR, 2023{\natexlab{b}}.

\bibitem[Bonet et~al.(2025{\natexlab{b}})Bonet, Drumetz, and Courty]{bonet2025sliced}
Clément Bonet, Lucas Drumetz, and Nicolas Courty.
\newblock {Sliced-Wasserstein Distances and Flows on Cartan-Hadamard Manifolds}.
\newblock \emph{Journal of Machine Learning Research}, 26\penalty0 (32):\penalty0 1--76, 2025{\natexlab{b}}.

\bibitem[Bonneel et~al.(2015)Bonneel, Rabin, Peyr{\'e}, and Pfister]{bonneel2015sliced}
Nicolas Bonneel, Julien Rabin, Gabriel Peyr{\'e}, and Hanspeter Pfister.
\newblock {Sliced and Radon Wasserstein Barycenters of Measures}.
\newblock \emph{Journal of Mathematical Imaging and Vision}, 51\penalty0 (1):\penalty0 22--45, 2015.

\bibitem[Bouveyron and Corneli(2026)]{bouveyron2025scaling}
Charles Bouveyron and Marco Corneli.
\newblock {Scaling Optimal Transport to High-Dimensional Gaussian Distributions with Application to Domain Adaptation}.
\newblock \emph{Statistics and Computing}, 36\penalty0 (2), 2026.
\newblock ISSN 0960-3174.

\bibitem[Brenier(1991)]{brenier1991polar}
Yann Brenier.
\newblock {Polar Factorization and Monotone Rearrangement of Vector-Valued Functions}.
\newblock \emph{Communications on pure and applied mathematics}, 44\penalty0 (4):\penalty0 375--417, 1991.

\bibitem[Bridson and Haefliger(2013)]{bridson2013metric}
Martin~R Bridson and Andr{\'e} Haefliger.
\newblock \emph{{Metric Spaces of Non-Positive Curvature}}, volume 319.
\newblock Springer Science \& Business Media, 2013.

\bibitem[Bronstein et~al.(2017)Bronstein, Bruna, LeCun, Szlam, and Vandergheynst]{bronstein2017geometric}
Michael~M Bronstein, Joan Bruna, Yann LeCun, Arthur Szlam, and Pierre Vandergheynst.
\newblock {Geometric deep learning: going beyond Euclidean data}.
\newblock \emph{IEEE Signal Processing Magazine}, 34\penalty0 (4):\penalty0 18--42, 2017.

\bibitem[Busemann(1955)]{busemann2005geometry}
Herbert Busemann.
\newblock \emph{{The Geometry of Geodesics}}.
\newblock Academic Press, New York, 1955.

\bibitem[Cazelles et~al.(2018)Cazelles, Seguy, Bigot, Cuturi, and Papadakis]{cazelles2018geodesic}
Elsa Cazelles, Vivien Seguy, J{\'e}r{\'e}mie Bigot, Marco Cuturi, and Nicolas Papadakis.
\newblock {Geodesic PCA versus Log-PCA of histograms in the Wasserstein space}.
\newblock \emph{SIAM Journal on Scientific Computing}, 40\penalty0 (2):\penalty0 B429--B456, 2018.

\bibitem[Chami et~al.(2021)Chami, Gu, Nguyen, and R{\'e}]{chami2021horopca}
Ines Chami, Albert Gu, Dat~P Nguyen, and Christopher R{\'e}.
\newblock {HoroPCA: Hyperbolic Dimensionality Reduction via Horospherical Projections}.
\newblock In \emph{International Conference on Machine Learning}, pages 1419--1429. PMLR, 2021.

\bibitem[Chen et~al.(2022)Chen, Yang, and Li]{chen2022augmented}
Xiongjie Chen, Yongxin Yang, and Yunpeng Li.
\newblock {Augmented Sliced Wasserstein Distances}.
\newblock In \emph{International Conference on Learning Representations}, 2022.

\bibitem[Chen et~al.(2018)Chen, Georgiou, and Tannenbaum]{chen2018optimal}
Yongxin Chen, Tryphon~T Georgiou, and Allen Tannenbaum.
\newblock {Optimal transport for Gaussian mixture models}.
\newblock \emph{IEEE Access}, 7:\penalty0 6269--6278, 2018.

\bibitem[Criscitiello and Kim(2025)]{criscitiello2025horospherically}
Christopher Criscitiello and Jungbin Kim.
\newblock {Horospherically Convex Optimization on Hadamard Manifolds Part I: Analysis and Algorithms}.
\newblock \emph{arXiv preprint arXiv:2505.16970}, 2025.

\bibitem[Cuturi(2013)]{cuturi2013sinkhorn}
Marco Cuturi.
\newblock {Sinkhorn Distances: Lightspeed Computation of Optimal Transport}.
\newblock \emph{Advances in neural information processing systems}, 26, 2013.

\bibitem[Delon and Desolneux(2020)]{delon2020wasserstein}
Julie Delon and Agnes Desolneux.
\newblock {A Wasserstein-Type Distance in the space of Gaussian Mixture Models}.
\newblock \emph{SIAM Journal on Imaging Sciences}, 13\penalty0 (2):\penalty0 936--970, 2020.

\bibitem[Diao et~al.(2023)Diao, Balasubramanian, Chewi, and Salim]{diao2023forward}
Michael~Ziyang Diao, Krishna Balasubramanian, Sinho Chewi, and Adil Salim.
\newblock {Forward-Backward Gaussian Variational Inference via JKO in the Bures-Wasserstein Space}.
\newblock In \emph{International Conference on Machine Learning}, pages 7960--7991. PMLR, 2023.

\bibitem[Doorenbos et~al.(2024)Doorenbos, Neila, Sznitman, and Mettes]{doorenbos2024hyperbolic}
Lars Doorenbos, Pablo~M{\'a}rquez Neila, Raphael Sznitman, and Pascal Mettes.
\newblock {Hyperbolic Random Forests}.
\newblock \emph{Transactions on Machine Learning Research}, 2024.
\newblock ISSN 2835-8856.

\bibitem[Durrant and Leontidis(2023)]{durrant2023hmsn}
Aiden Durrant and Georgios Leontidis.
\newblock {HMSN: Hyperbolic Self-Supervised Learning by Clustering with Ideal Prototypes}.
\newblock \emph{arXiv preprint arXiv:2305.10926}, 2023.

\bibitem[Dvurechensky et~al.(2018)Dvurechensky, Gasnikov, and Kroshnin]{dvurechensky2018computational}
Pavel Dvurechensky, Alexander Gasnikov, and Alexey Kroshnin.
\newblock {Computational Optimal Transport: Complexity by Accelerated Gradient Descent is Better than by Sinkhorn’s Algorithm}.
\newblock In \emph{International conference on machine learning}, pages 1367--1376. PMLR, 2018.

\bibitem[Emami and Pass(2025)]{emami2025optimal}
Pedram Emami and Brendan Pass.
\newblock {Optimal transport with optimal transport cost: the Monge--Kantorovich problem on Wasserstein spaces}.
\newblock \emph{Calculus of Variations and Partial Differential Equations}, 64\penalty0 (2):\penalty0 43, 2025.

\bibitem[Fan et~al.(2023)Fan, Yang, and Vemuri]{fan2023horospherical}
Xiran Fan, Chun-Hao Yang, and Baba Vemuri.
\newblock {Horocycle Decision Boundaries for Large Margin Classification in Hyperbolic Space}.
\newblock \emph{Advances in neural information processing systems}, 36:\penalty0 11194--11204, 2023.

\bibitem[Fatras et~al.(2020)Fatras, Zine, Flamary, Gribonval, and Courty]{fatras2020learning}
Kilian Fatras, Younes Zine, R\'emi Flamary, Remi Gribonval, and Nicolas Courty.
\newblock {Learning with minibatch Wasserstein : asymptotic and gradient properties}.
\newblock In Silvia Chiappa and Roberto Calandra, editors, \emph{Proceedings of the Twenty Third International Conference on Artificial Intelligence and Statistics}, volume 108 of \emph{Proceedings of Machine Learning Research}, pages 2131--2141. PMLR, 26--28 Aug 2020.

\bibitem[Fatras et~al.(2021)Fatras, Zine, Majewski, Flamary, Gribonval, and Courty]{fatras2021minibatch}
Kilian Fatras, Younes Zine, Szymon Majewski, R{\'e}mi Flamary, R{\'e}mi Gribonval, and Nicolas Courty.
\newblock {Minibatch optimal transport distances; analysis and applications}.
\newblock \emph{arXiv preprint arXiv:2101.01792}, 2021.

\bibitem[Flamary et~al.(2021)Flamary, Courty, Gramfort, Alaya, Boisbunon, Chambon, Chapel, Corenflos, Fatras, Fournier, et~al.]{flamary2021pot}
R{\'e}mi Flamary, Nicolas Courty, Alexandre Gramfort, Mokhtar~Z Alaya, Aur{\'e}lie Boisbunon, Stanislas Chambon, Laetitia Chapel, Adrien Corenflos, Kilian Fatras, Nemo Fournier, et~al.
\newblock {POT: Python Optimal Transport}.
\newblock \emph{Journal of Machine Learning Research}, 22\penalty0 (78):\penalty0 1--8, 2021.

\bibitem[Flamary et~al.(2024)Flamary, Vincent-Cuaz, Courty, Gramfort, Kachaiev, Quang~Tran, David, Bonet, Cassereau, Gnassounou, Tanguy, Delon, Collas, Mazelet, Chapel, Kerdoncuff, Yu, Feickert, Krzakala, Liu, and Fernandes~Montesuma]{flamary2024pot}
R{\'e}mi Flamary, C{\'e}dric Vincent-Cuaz, Nicolas Courty, Alexandre Gramfort, Oleksii Kachaiev, Huy Quang~Tran, Laurène David, Cl{\'e}ment Bonet, Nathan Cassereau, Th{\'e}o Gnassounou, Eloi Tanguy, Julie Delon, Antoine Collas, Sonia Mazelet, Laetitia Chapel, Tanguy Kerdoncuff, Xizheng Yu, Matthew Feickert, Paul Krzakala, Tianlin Liu, and Eduardo Fernandes~Montesuma.
\newblock {POT Python Optimal Transport (version 0.9.5)}, 2024.
\newblock URL \url{https://github.com/PythonOT/POT}.

\bibitem[Fletcher et~al.(2009)Fletcher, Moeller, Phillips, and Venkatasubramanian]{fletcher2009computing}
P~Thomas Fletcher, John Moeller, Jeff~M Phillips, and Suresh Venkatasubramanian.
\newblock {Computing Hulls and Centerpoints in Positive Definite Space}.
\newblock \emph{arXiv preprint arXiv:0912.1580}, 2009.

\bibitem[Fletcher et~al.(2011)Fletcher, Moeller, Phillips, and Venkatasubramanian]{fletcher2011horoball}
P~Thomas Fletcher, John Moeller, Jeff~M Phillips, and Suresh Venkatasubramanian.
\newblock {Horoball Hulls and Extents in Positive Definite Space}.
\newblock In \emph{Workshop on Algorithms and Data Structures}, pages 386--398. Springer, 2011.

\bibitem[Fujii(2010)]{fujii2010furuta}
Masatoshi Fujii.
\newblock {Furuta Inequality and its Related Topics}.
\newblock \emph{Annals of Functional Analysis}, 1\penalty0 (2):\penalty0 24--45, 2010.

\bibitem[Gallou{\"e}t et~al.(2024)Gallou{\"e}t, Natale, and Todeschi]{natale2022geodesic}
Thomas Gallou{\"e}t, Andrea Natale, and Gabriele Todeschi.
\newblock {From geodesic extrapolation to a variational BDF2 scheme for Wasserstein gradient flows}.
\newblock \emph{Mathematics of Computation}, 2024.

\bibitem[Gallou{\"e}t et~al.(2025)Gallou{\"e}t, Natale, and Todeschi]{gallouet2025metric}
Thomas~O Gallou{\"e}t, Andrea Natale, and Gabriele Todeschi.
\newblock {Metric extrapolation in the Wasserstein space}.
\newblock \emph{Calculus of Variations and Partial Differential Equations}, 64\penalty0 (5):\penalty0 147, 2025.

\bibitem[Gelbrich(1990)]{gelbrich1990formula}
Matthias Gelbrich.
\newblock {On a Formula for the L2 Wasserstein Metric between Measures on Euclidean and Hilbert Spaces}.
\newblock \emph{Mathematische Nachrichten}, 147\penalty0 (1):\penalty0 185--203, 1990.

\bibitem[Geuter et~al.(2025)Geuter, Bonet, Korba, and Alvarez-Melis]{geuter2025ddeqs}
Jonathan Geuter, Clément Bonet, Anna Korba, and David Alvarez-Melis.
\newblock {DDEQs: Distributional Deep Equilibrium Models through Wasserstein Gradient Flows}.
\newblock In \emph{The 28th International Conference on Artificial Intelligence and Statistics}, 2025.

\bibitem[Ghadimi~Atigh et~al.(2021)Ghadimi~Atigh, Keller-Ressel, and Mettes]{ghadimi2021hyperbolic}
Mina Ghadimi~Atigh, Martin Keller-Ressel, and Pascal Mettes.
\newblock {Hyperbolic Busemann Learning with Ideal Prototypes}.
\newblock \emph{Advances in Neural Information Processing Systems}, 34:\penalty0 103--115, 2021.

\bibitem[Givens and Shortt(1984)]{givens1984class}
Clark~R Givens and Rae~Michael Shortt.
\newblock {A Class of Wasserstein Metrics for Probability Distributions}.
\newblock \emph{Michigan Mathematical Journal}, 31\penalty0 (2):\penalty0 231--240, 1984.

\bibitem[Glaser et~al.(2021)Glaser, Arbel, and Gretton]{glaser2021kale}
Pierre Glaser, Michael Arbel, and Arthur Gretton.
\newblock {KALE flow: A Relaxed KL Gradient Flow for Probabilities with Disjoint Support}.
\newblock \emph{Advances in Neural Information Processing Systems}, 34:\penalty0 8018--8031, 2021.

\bibitem[Goodwin et~al.(2024)Goodwin, Lewis, L{\'o}pez-Acedo, and Nicolae]{goodwin2024subgradient}
Ariel Goodwin, Adrian~S Lewis, Genaro L{\'o}pez-Acedo, and Adriana Nicolae.
\newblock {A subgradient splitting algorithm for optimization on nonpositively curved metric spaces}.
\newblock \emph{arXiv preprint arXiv:2412.06730}, 2024.

\bibitem[Han(2023)]{han2023sliced}
Ruiyu Han.
\newblock {Sliced Wasserstein distance between probability measures on Hilbert spaces}.
\newblock \emph{arXiv preprint arXiv:2307.05802}, 2023.

\bibitem[Haviv et~al.(2025{\natexlab{a}})Haviv, Pooladian, Pe'er, and Amos]{haviv2025wasserstein}
Doron Haviv, Aram-Alexandre Pooladian, Dana Pe'er, and Brandon Amos.
\newblock {Wasserstein Flow Matching: Generative Modeling Over Families of Distributions}.
\newblock In \emph{Forty-second International Conference on Machine Learning}, 2025{\natexlab{a}}.

\bibitem[Haviv et~al.(2025{\natexlab{b}})Haviv, Rem{\v{s}}{\'\i}k, Gatie, Snopkowski, Takizawa, Pereira, Bashkin, Jovanovich, Nawy, Chaligne, et~al.]{haviv2025covariance}
Doron Haviv, J{\'a}n Rem{\v{s}}{\'\i}k, Mohamed Gatie, Catherine Snopkowski, Meril Takizawa, Nathan Pereira, John Bashkin, Stevan Jovanovich, Tal Nawy, Ronan Chaligne, et~al.
\newblock {The covariance environment defines cellular niches for spatial inference}.
\newblock \emph{Nature Biotechnology}, 43\penalty0 (2):\penalty0 269--280, 2025{\natexlab{b}}.

\bibitem[Hua et~al.(2023)Hua, Nguyen, Le, Blanchet, and Nguyen]{hua2023dynamic}
Xinru Hua, Truyen Nguyen, Tam Le, Jose Blanchet, and Viet~Anh Nguyen.
\newblock {Dynamic Flows on Curved Space Generated by Labeled Data}.
\newblock In Edith Elkind, editor, \emph{Proceedings of the Thirty-Second International Joint Conference on Artificial Intelligence, {IJCAI-23}}, pages 3803--3811. International Joint Conferences on Artificial Intelligence Organization, 8 2023.
\newblock Main Track.

\bibitem[Hull(1994)]{hull1994database}
Jonathan~J. Hull.
\newblock {A Database for Handwritten Text Recognition Research}.
\newblock \emph{IEEE Transactions on pattern analysis and machine intelligence}, 16\penalty0 (5):\penalty0 550--554, 1994.

\bibitem[Kloeckner(2010)]{kloeckner2010geometric}
Benoit Kloeckner.
\newblock {A geometric study of Wasserstein spaces: Euclidean spaces}.
\newblock \emph{Annali della Scuola Normale Superiore di Pisa-Classe di Scienze}, 9\penalty0 (2):\penalty0 297--323, 2010.

\bibitem[Kolouri et~al.(2019)Kolouri, Nadjahi, Simsekli, Badeau, and Rohde]{kolouri2019generalized}
Soheil Kolouri, Kimia Nadjahi, Umut Simsekli, Roland Badeau, and Gustavo Rohde.
\newblock {Generalized Sliced Wasserstein Distances}.
\newblock \emph{Advances in neural information processing systems}, 32, 2019.

\bibitem[Krizhevsky et~al.(2009)Krizhevsky, Hinton, et~al.]{krizhevsky2009learning}
Alex Krizhevsky, Geoffrey Hinton, et~al.
\newblock {Learning Multiple Layers of Features from Tiny Images}.
\newblock 2009.

\bibitem[Kusner et~al.(2015)Kusner, Sun, Kolkin, and Weinberger]{kusner2015word}
Matt Kusner, Yu~Sun, Nicholas Kolkin, and Kilian Weinberger.
\newblock {From Word Embeddings to Document Distances}.
\newblock In \emph{International conference on machine learning}, pages 957--966. PMLR, 2015.

\bibitem[Lambert et~al.(2022)Lambert, Chewi, Bach, Bonnabel, and Rigollet]{lambert2022variational}
Marc Lambert, Sinho Chewi, Francis Bach, Silv{\`e}re Bonnabel, and Philippe Rigollet.
\newblock {Variational inference via Wasserstein gradient flows}.
\newblock \emph{Advances in Neural Information Processing Systems}, 35:\penalty0 14434--14447, 2022.

\bibitem[LeCun and Cortes(2010)]{lecun2010mnist}
Yann LeCun and Corinna Cortes.
\newblock {{MNIST} handwritten digit database}.
\newblock 2010.

\bibitem[Liu et~al.(2025)Liu, Bai, Lu, Soltoggio, and Kolouri]{liu2025wasserstein}
Xinran Liu, Yikun Bai, Yuzhe Lu, Andrea Soltoggio, and Soheil Kolouri.
\newblock {Wasserstein Task Embedding for Mmeasuring Task Similarities}.
\newblock \emph{Neural Networks}, 181:\penalty0 106796, 2025.

\bibitem[McCann(1997)]{mccann1997convexity}
Robert~J McCann.
\newblock {A Convexity Principle for Interacting Gases}.
\newblock \emph{Advances in mathematics}, 128\penalty0 (1):\penalty0 153--179, 1997.

\bibitem[M{\"u}ller(2001)]{muller2001stochastic}
Alfred M{\"u}ller.
\newblock {Stochastic Ordering of Multivariate Normal Distributions}.
\newblock \emph{Annals of the Institute of Statistical Mathematics}, 53\penalty0 (3):\penalty0 567--575, 2001.

\bibitem[Nguyen and Ho(2022)]{nguyen2022revisiting}
Khai Nguyen and Nhat Ho.
\newblock {Revisiting Sliced Wasserstein on Images: From Vectorization to Convolution}.
\newblock \emph{Advances in Neural Information Processing Systems}, 35:\penalty0 17788--17801, 2022.

\bibitem[Nguyen and Ho(2024)]{nguyen2024hierarchical}
Khai Nguyen and Nhat Ho.
\newblock {Hierarchical Hybrid Sliced Wasserstein: A Scalable Metric for Heterogeneous Joint Distributions}.
\newblock \emph{Advances in Neural Information Processing Systems}, 37:\penalty0 108140--108166, 2024.

\bibitem[Nguyen and Mueller(2026)]{nguyen2024summarizing}
Khai Nguyen and Peter Mueller.
\newblock {Summarizing Nonparametric Bayesian Mixture Posteriors – Sliced Optimal Transport Metrics for Gaussian Mixtures}.
\newblock \emph{Journal of Computational and Graphical Statistics}, pages 1--22, 2026.

\bibitem[Nguyen et~al.(2025{\natexlab{a}})Nguyen, Nguyen, Pham, and Ho]{nguyen2025lightspeed}
Khai Nguyen, Hai Nguyen, Tuan Pham, and Nhat Ho.
\newblock {Lightspeed Geometric Dataset Distance via Sliced Optimal Transport}.
\newblock In \emph{Forty-second International Conference on Machine Learning}, 2025{\natexlab{a}}.

\bibitem[Nguyen et~al.(2025{\natexlab{b}})Nguyen, Yang, and Histace]{nguyen2025neural}
Xuan~Son Nguyen, Shuo Yang, and Aymeric Histace.
\newblock Neural networks on symmetric spaces of noncompact type.
\newblock In \emph{The Thirteenth International Conference on Learning Representations}, 2025{\natexlab{b}}.

\bibitem[Nguyen(2016)]{nguyen2016borrowing}
XuanLong Nguyen.
\newblock {Borrowing strengh in hierarchical Bayes: Posterior concentration of the Dirichlet base measure}.
\newblock \emph{Bernoulli}, 22\penalty0 (3):\penalty0 1535 -- 1571, 2016.

\bibitem[Otto(2001)]{otto2001geometry}
Felix Otto.
\newblock {The geometry of dissipative evolution equations: the porous medium equation}.
\newblock \emph{Communications in Partial Differential Equations}, 26\penalty0 (1-2):\penalty0 101--174, 2001.

\bibitem[Papillon et~al.(2025)Papillon, Sanborn, Mathe, Cornelis, Bertics, Buracas, J~Lillemark, Shewmake, Dinc, Pennec, et~al.]{papillon2025beyond}
Mathilde Papillon, Sophia Sanborn, Johan Mathe, Louisa Cornelis, Abby Bertics, Domas Buracas, Hansen J~Lillemark, Christian Shewmake, Fatih Dinc, Xavier Pennec, et~al.
\newblock {Beyond Euclid: An Illustrated Guide to Modern Machine Learning with Geometric, Topological, and Algebraic Structures}.
\newblock \emph{Machine Learning: Science and Technology}, 6\penalty0 (3):\penalty0 031002, 2025.

\bibitem[Paty et~al.(2020)Paty, d’Aspremont, and Cuturi]{paty2020regularity}
Fran{\c{c}}ois-Pierre Paty, Alexandre d’Aspremont, and Marco Cuturi.
\newblock {Regularity as Regularization: Smooth and Strongly Convex Brenier Potentials in Optimal Transport}.
\newblock In \emph{International Conference on Artificial Intelligence and Statistics}, pages 1222--1232. PMLR, 2020.

\bibitem[Peyr{\'e} and Cuturi(2019)]{peyre2019computational}
Gabriel Peyr{\'e} and Marco Cuturi.
\newblock {Computational Optimal Transport: With Applications to Data Science}.
\newblock \emph{Foundations and Trends{\textregistered} in Machine Learning}, 11\penalty0 (5-6):\penalty0 355--607, 2019.

\bibitem[Piening and Beinert(2025)]{piening2025slicing}
Moritz Piening and Robert Beinert.
\newblock {Slicing the Gaussian Mixture Wasserstein Distance}.
\newblock \emph{Transactions on Machine Learning Research}, 2025.
\newblock ISSN 2835-8856.

\bibitem[Piening and Beinert(2026)]{piening2025slicingWoW}
Moritz Piening and Robert Beinert.
\newblock {Slicing Wasserstein over Wasserstein via Functional Optimal Transport}.
\newblock In \emph{The Fourteenth International Conference on Learning Representations}, 2026.

\bibitem[Pinzi and Savar{\'e}(2025)]{pinzi2025totally}
Alessandro Pinzi and Giuseppe Savar{\'e}.
\newblock {Totally convex functions, $L^{2}$-Optimal transport for laws of random measures, and solution to the Monge problem}.
\newblock \emph{arXiv preprint arXiv:2509.01768}, 2025.

\bibitem[Rabin et~al.(2012)Rabin, Peyr{\'e}, Delon, and Bernot]{rabin2012wasserstein}
Julien Rabin, Gabriel Peyr{\'e}, Julie Delon, and Marc Bernot.
\newblock {Wasserstein Barycenter and its Application to Texture Mixing}.
\newblock In \emph{Scale Space and Variational Methods in Computer Vision: Third International Conference, SSVM 2011, Ein-Gedi, Israel, May 29--June 2, 2011, Revised Selected Papers 3}, pages 435--446. Springer, 2012.

\bibitem[Santambrogio(2015)]{santambrogio2015optimal}
Filippo Santambrogio.
\newblock \emph{{Optimal Transport for Applied Mathematicians}}, volume~55.
\newblock Springer, 2015.

\bibitem[Scetbon et~al.(2021)Scetbon, Cuturi, and Peyr{\'e}]{scetbon2021low}
Meyer Scetbon, Marco Cuturi, and Gabriel Peyr{\'e}.
\newblock {Low-Rank Sinkhorn Factorization}.
\newblock In \emph{International Conference on Machine Learning}, pages 9344--9354. PMLR, 2021.

\bibitem[Seguy and Cuturi(2015)]{seguy2015principal}
Vivien Seguy and Marco Cuturi.
\newblock {Principal Geodesic Analysis for Probability Measures under the Optimal Transport Metric}.
\newblock \emph{Advances in Neural Information Processing Systems}, 28, 2015.

\bibitem[Shu(2020)]{shu2020hopf}
Yan Shu.
\newblock {From Hopf--Lax Formula to Optimal Weak Transfer Plan}.
\newblock \emph{SIAM Journal on Mathematical Analysis}, 52\penalty0 (3):\penalty0 3052--3072, 2020.

\bibitem[Sonoda et~al.(2022)Sonoda, Ishikawa, and Ikeda]{sonoda2022fully}
Sho Sonoda, Isao Ishikawa, and Masahiro Ikeda.
\newblock {Fully-Connected Network on Noncompact Symmetric Space and Ridgelet Transform based on Helgason-Fourier Analysis}.
\newblock In \emph{International Conference on Machine Learning}, pages 20405--20422. PMLR, 2022.

\bibitem[Takatsu(2011)]{takatsu2011wasserstein}
Asuka Takatsu.
\newblock {Wasserstein geometry of Gaussian measures}.
\newblock \emph{Osaka Journal of Mathematics}, 48\penalty0 (4):\penalty0 1005 -- 1026, 2011.

\bibitem[{Van Assel} et~al.(2024){Van Assel}, Courty, Flamary, Garivier, Massias, Vayer, and Vincent-Cuaz]{torchdr}
Hugues {Van Assel}, Nicolas Courty, Rémi Flamary, Aurélien Garivier, Mathurin Massias, Titouan Vayer, and Cédric Vincent-Cuaz.
\newblock {TorchDR : A PyTorch library for Dimensionality Reduction}.
\newblock \url{https://torchdr.github.io/}, 2024.

\bibitem[Vesseron et~al.(2026)Vesseron, Cazelles, Brigant, and Klein]{vesseron2025wasserstein}
Nina Vesseron, Elsa Cazelles, Alice~Le Brigant, and Klein.
\newblock {On the Wasserstein Geodesic Principal Component Analysis of probability measures}.
\newblock In \emph{The Fourteenth International Conference on Learning Representations}, 2026.

\bibitem[Villani(2009)]{villani2009optimal}
C{\'e}dric Villani.
\newblock \emph{{Optimal Transport: Old and New}}, volume 338.
\newblock Springer, 2009.

\bibitem[Virtanen et~al.(2020)Virtanen, Gommers, Oliphant, Haberland, Reddy, Cournapeau, Burovski, Peterson, Weckesser, Bright, {van der Walt}, Brett, Wilson, Millman, Mayorov, Nelson, Jones, Kern, Larson, Carey, Polat, Feng, Moore, {VanderPlas}, Laxalde, Perktold, Cimrman, Henriksen, Quintero, Harris, Archibald, Ribeiro, Pedregosa, {van Mulbregt}, and {SciPy 1.0 Contributors}]{2020SciPy-NMeth}
Pauli Virtanen, Ralf Gommers, Travis~E. Oliphant, Matt Haberland, Tyler Reddy, David Cournapeau, Evgeni Burovski, Pearu Peterson, Warren Weckesser, Jonathan Bright, St{\'e}fan~J. {van der Walt}, Matthew Brett, Joshua Wilson, K.~Jarrod Millman, Nikolay Mayorov, Andrew R.~J. Nelson, Eric Jones, Robert Kern, Eric Larson, C~J Carey, {\.I}lhan Polat, Yu~Feng, Eric~W. Moore, Jake {VanderPlas}, Denis Laxalde, Josef Perktold, Robert Cimrman, Ian Henriksen, E.~A. Quintero, Charles~R. Harris, Anne~M. Archibald, Ant{\^o}nio~H. Ribeiro, Fabian Pedregosa, Paul {van Mulbregt}, and {SciPy 1.0 Contributors}.
\newblock {{SciPy} 1.0: Fundamental Algorithms for Scientific Computing in Python}.
\newblock \emph{Nature Methods}, 17:\penalty0 261--272, 2020.
\newblock \doi{10.1038/s41592-019-0686-2}.

\bibitem[Wang(2021)]{wang2021laplacian}
Ming-Xi Wang.
\newblock {Laplacian Eigenspaces, Horocycles and Neuron Models on Hyperbolic Spaces}, 2021.

\bibitem[Wang et~al.(2013)Wang, Slep{\v{c}}ev, Basu, Ozolek, and Rohde]{wang2013linear}
Wei Wang, Dejan Slep{\v{c}}ev, Saurav Basu, John~A Ozolek, and Gustavo~K Rohde.
\newblock {A Linear Optimal Transportation Framework for Quantifying and Visualizing Variations in Sets of Images}.
\newblock \emph{International journal of computer vision}, 101\penalty0 (2):\penalty0 254--269, 2013.

\bibitem[Xiao et~al.(2017)Xiao, Rasul, and Vollgraf]{xiao2017fashion}
Han Xiao, Kashif Rasul, and Roland Vollgraf.
\newblock {Fashion-MNIST: a Novel Image Dataset for Benchmarking Machine Learning Algorithms}.
\newblock \emph{arXiv preprint arXiv:1708.07747}, 2017.

\bibitem[Zhu et~al.(2021)Zhu, Li, and Cui]{zhu2021busemann}
Guomin Zhu, Wen-Long Li, and Xiaojun Cui.
\newblock {Busemann functions on the Wasserstein space}.
\newblock \emph{Calculus of Variations and Partial Differential Equations}, 60\penalty0 (3):\penalty0 1--16, 2021.

\end{thebibliography}
